\documentclass[final]{article}

\usepackage{amsmath}
\usepackage{amsfonts}
\usepackage{amsthm}
\usepackage[english]{babel} 

\usepackage{algorithmic}
 \usepackage{amssymb}
\usepackage{graphicx}        
 \usepackage{cite}                            
\usepackage{subfigure}
\usepackage{epstopdf}
\usepackage{epsfig}
\usepackage[table]{xcolor}
\usepackage{subfloat}
\usepackage{float}
\usepackage[thinlines]{easytable}
\usepackage{array}
\usepackage[linesnumbered,ruled,vlined]{algorithm2e}
\allowdisplaybreaks

\newtheorem{remark}{Remark}[section]
\newtheorem{definition}{Definition}[section]

\newtheorem{lemma}{Lemma}[section]
\newtheorem{theorem}{Theorem}[section]
\newtheorem{proposition}{Proposition}[section]

\def\OV{\overline V}

\def\RR{\mathbb R}
\def\EE{\mathcal E}

\newcommand{\BS}{\mathbb S}

\def\theta{\vartheta}

\def\argmin{{\rm arg}\!\min}
\def\be{\begin{equation}}
\def\ee{\end{equation}}
\def\bea{\begin{eqnarray}}
\def\eea{\end{eqnarray}}
\def\Pi{P}

\newtheorem{assu}{Assumption}[section]
\newcommand{\mc}[1]{\mathcal{#1}}

\newcommand{\la}{\langle}
\newcommand{\ra}{\rangle}

\title{Consensus-Based Optimization on the Sphere: Convergence to Global Minimizers and Machine Learning}
\author{Massimo Fornasier \footnote{Department of Mathematics, Technical University of Munich, Boltzmannstra{\ss}e 3, 85748 Garching (Munich), Germany
(massimo.fornasier{@}ma.tum.de).}\qquad
Hui Huang \footnote{Department of Mathematics and Statistics, University of Calgary
	(hui.huang1@ucalgary.ca), 2500 University Drive NW
Calgary, AB, Canada.}\qquad
Lorenzo Pareschi\footnote{Department of Mathematics \& Computer Science, University of Ferrara, Via Machiavelli 30, Ferrara, 44121, Italy (lorenzo.pareschi{@}unife.it).}\qquad
Philippe S\"{u}nnen \footnote{Department of Mathematics, Technical University of Munich, Boltzmannstra{\ss}e 3, 85748 Garching (Munich), Germany
	(philippe.suennen{@}ma.tum.de).}
}

\begin{document}
\maketitle

\begin{abstract}
 We investigate the implementation of a new stochastic Kuramoto-Vicsek-type model for global optimization of nonconvex functions on the sphere. This model belongs to the class of Consensus-Based Optimization. In fact, particles move on the sphere driven by a drift towards an instantaneous consensus point, which is computed as a convex combination of particle locations, weighted by the cost function according to Laplace's principle, and it represents an approximation to a global minimizer. The dynamics is further perturbed by a random vector field to favor exploration, whose variance is a function of the distance of the particles  to the consensus point. In particular, as soon as the  consensus is reached the stochastic component vanishes.
The main results of this paper are about the proof of convergence of the numerical scheme to global minimizers provided conditions of well-preparation of the initial datum. The proof combines previous results of mean-field limit with a novel asymptotic analysis, and classical convergence results of numerical methods for SDE.
We present several numerical experiments, which show that the algorithm proposed in the present paper scales well with the dimension and is extremely versatile. To quantify the performances of the new approach, we show that the algorithm is able to perform essentially as good as {\it ad hoc} state of the art methods  in challenging problems in signal processing and machine learning, namely the phase retrieval problem and the robust subspace detection.  
\end{abstract}


\tableofcontents

\section{Introduction}

\subsection{Derivative-free optimization and metaheuristics}

Machine learning is about parametric nonlinear algorithms, whose parameters are optimized towards several tasks such as feature selection, dimensionality reduction, clustering, classification, regression, and generation. In view of the nonlinearity of the algorithms and the use of often nonconvex data misfits or penalizations/regularizations, the training phase is most commonly a nonconvex optimization. Moreover, 
the efficacy of such methods is often determined by considering a large amount of parameters, which makes the optimization problem high dimensional and therefore quite hard.
Often first order methods, such as gradient descent methods, are preferred both because of speed and scalability and because they are considered generically able to escape the trap of saddle points \cite{recht19}, and in some cases they are able even to compute global minimizers \cite{Chen_2019,liu2019bad,bah2019}.
Nevertheless,   for some models, such as training of certain feed-forward deep neural networks, the gradient tends to explode or vanish, \cite{bengio1994learning}. 
For many other problems the derivative of the objective function can be extremely computational expensive  to compute or the objective function may not be even differentiable at all. 
Finally, gradient descent methods do not offer in general guarantees of global convergence and, in view of high dimensionality and nonconvexity, a large amount of local minimizers are expected to possibly trap the dynamics (see Section \ref{sec:robsub} for concrete examples).\\
Long before the current uses in machine learning, nonconvex optimizations have been considered in optimal design of any sort of processes and several solutions  have been proposed to tackle these problems. In this paper we are concerned with those which fall into the class of  \textit{metaheuristics} 
\cite{Aarts:1989:SAB:61990,Back:1997:HEC:548530,Blum:2003:MCO:937503.937505,Gendreau:2010:HM:1941310}, which provide empirically robust solutions to hard optimization problems with fast algorithms. Metaheuristics are  methods that orchestrate an interaction between local improvement procedures and global/high level strategies, and combine random and deterministic decisions,
to create a process capable of escaping from local optima and performing a robust search of a solution space.
Starting with the groundbreaking work of Rastrigin on Random Search in 1963 \cite{rastrigin63}, 
numerous mechanisms for multi-agent global optimization have been  considered, among the most prominent instances we recall the Simplex Heuristics \cite{NeldMead65}, Evolutionary Programming \cite{Fogel:2006:ECT:1202305},  the  Metropolis-Hastings sampling algorithm \cite{hastings70}, Genetic Algorithms \cite{Holland:1992:ANA:531075}, Particle Swarm Optimization (PSO) \cite{kennedy2010particle,poli2007particle}, Ant Colony Optimization (ACO) \cite{dorigo2005ant}, Simulated Annealing (SA), \cite{holley1988simulated,kirkpatrick1983optimization}. 
Despite the tremendous empirical success of these techniques, it is still quite difficult to provide guarantees of robust convergence to global minimizers, because of the random component of metaheuristics, which would require to discern the stochastic dependencies. Such analysis is often a very hard task, especially for those methods that combine instantaneous decisions with memory mechanisms. 

\subsection{Consensus-based optimization}

Recent work  by Pinnau, Carrillo et al. \cite{pinnau2017consensus,carrillo2018analytical}  on Consensus-based Optimization (CBO) focuses on instantaneous stochastic and deterministic decisions in order to establish a consensus among particles on the location of the global minimizers within a domain.
In view of the instantaneous nature of the dynamics, the evolution can be interpreted as a system of first order stochastic differential equations (SDEs), whose large particle limit is approximated by a deterministic partial differential equation of mean-field type. The large time behavior of such a deterministic PDE can be analyzed by classical techniques of large deviation bounds and the global convergence of the mean-field model can be mathematically proven in a rigorous way for a large class of optimization problems, see \cite{fornasier2021consensusbased}. Certainly CBO is a significantly simpler mechanism with respect to more sophisticated metaheuristics, which can include different features including memory of past exploration. Nevertheless, it is general enough to explain other metaheuristics methods such as particle swarm optimization \cite{cipriani2021zero,grassipareschi21} and powerful and robust enough to tackle many interesting nonconvex optimizations of practical relevance in machine learning. In particular CBO and variants have been recently tested as optimization methods for the training of artificial neural networks, showing competitive results over stochastic gradient descent, also in terms of generalization error, see \cite[Section 5.5 and Figure 6 and Figure 7]{benfenati2021binary} and \cite[Section 4.3 and Figure 6]{carrillo2019consensus}. From the theoretical side, one may refer, for instance, to the recent paper \cite{Cao_Gu_2020} for theoretical estimates of the generalization error  in training deep neural networks. If one really inspects carefully the results and the proofs, one realizes that \cite[Theorem 3.2]{Cao_Gu_2020} is all about the {\it global} optimization by gradient descent of the empirical risk. Since CBO methods are precisely designed to achieve global optimization, they also allow for same guarantees of generalization errors as gradient descent, as one could simply substitute \cite[Theorem 3.2]{Cao_Gu_2020} with any global convergence result of CBO and obtain the same bounds. We mention also that CBO with adaptive momentum has been recently proposed in \cite{chen2020consensusbased}, providing better generalization results than the state-of-the-art method Adam in solving a deep learning task for partial differential equations with low-regularity solutions.
Some theoretical gaps remain open in the analysis of CBO though, in particular the lack of a rigorous derivation of the mean-field limit  as in \cite{carrillo2018analytical}, which as been very recently established just 
in a non-quantitative form in \cite{huang2021meanfield}.

\subsection{Consensus-based optimization on the sphere}

Motivated by lack of a quantitative mean-field limit and by the several potential applications in machine learning, in the companion paper \cite{fhps20-1} we introduced  {for the first time in the literature a new CBO approach to solve the following constrained optimization problem
\begin{equation}\label{typrob}
v^\ast \in \argmin\limits_{v\in \Gamma}\EE(v)\,,
\end{equation}
where $\EE:\mathbb R^{d} \to \mathbb R$ is a given continuous cost function, which we wish to minimize over a compact hypersurface $\Gamma$. In this paper we consider the particular case of
$\Gamma=\mathbb S^{d-1}$ being the hypersphere, for which we formulate} a system of $N$ interacting particles  $((V_t^i)_{t\geq 0 })_{i=1,\dots,N}$ satisfying  the following stochastic Kuramoto-Vicsek-type dynamics expressed in It\^{o}'s form
\begin{align} \label{stochastic Kuramoto-Vicsek}
dV_t^i &= \lambda P(V_t^i)v_{\alpha,\EE}(\rho_t^N)dt + \sigma |V_t^i - v_{\alpha,\EE}(\rho_t^N)| P(V_t^i)dB_t^i-\frac{\sigma^2}{2}(V_t^i-v_{\alpha,\EE}(\rho_t^N))^2\frac{(d-1)V_t^i}{|V_t^i|^2}dt\,,
\end{align}
where $\lambda>0$ is a suitable drift parameter, $\sigma>0$ is a diffusion parameter,
\begin{equation}
\rho_t^N=\frac{1}{N}\sum_{i=1}^{N}\delta_{V_t^i}
\end{equation}
is the empirical measure of the particles ($\delta_v$ is the Dirac measure at $v\in\RR^d$), and
\begin{equation}\label{ValphaE}
v_{\alpha,\EE}(\rho_t^N)=\sum_{j=1}^{N}\frac{V_t^j\omega_\alpha^\EE(V_t^j)}{\sum_{i=1}^{N}\omega_\alpha^\EE(V_t^i)}=\frac{\int_{\mathbb R^{d}}v\omega_\alpha^\EE(v)d\rho_t^N}{\int_{\mathbb R^{d}}\omega_\alpha^\EE(v)d\rho_t^N}\,, \qquad  \omega_\alpha^\EE(v):=e^{-\alpha\EE(v)}\,.
\end{equation}
{Here and below we denote with $\int f(v) d \mu(v)$ or equivalently $\int f(v) \mu(dv)$ the integration of an arbitrary function with respect to a measure $\mu$.}
This stochastic system is considered complemented with independent and identically distributed (i.i.d.) initial data $V_0^i\in\BS^{d-1}$ with $i=1,\cdots,N$, and the common law is denoted by $\rho_0\in \mc{P}(\BS^{d-1})$. The trajectories $((B_t^i)_{t\geq0})_{i=1,\dots N}$ denote $N$ independent standard Brownian motions in $\RR^d$.
In \eqref{stochastic Kuramoto-Vicsek} the projection operator $P(\cdot)$ is defined by
\begin{equation}\label{defP}
P(v)=I-\frac{vv^{T}}{|v|^2}\,.
\end{equation}
It is easy to check that
\begin{equation}\label{orthproj}
P(v)v=0
\mbox{ and } v\cdot P(v)y=0 \mbox{ for all }y\in \RR^d\,.
 \end{equation}
The choice of the weight function $\omega_\alpha^\EE$ in \eqref{ValphaE} comes from  the  well-known Laplace principle \cite{miller2006applied,Dembo2010,pinnau2017consensus}, a classical asymptotic method for integrals, which states that for any probability measure $\rho\in\mc{P}_{\rm{ac}}(\RR^d)$ (absolutely continuous), it holds
\begin{equation}\label{Laplace}
\lim\limits_{\alpha\to\infty}\left(-\frac{1}{\alpha}\log\left(\int_{\RR^d}e^{-\alpha\EE(v)}d\rho(v)\right)\right)=\inf\limits_{v \in \rm{supp }\rho} \EE(v)\,.
\end{equation}

Let us discuss the mechanism of the dynamics. The right-hand-side of the equation \eqref{stochastic Kuramoto-Vicsek} is made of three terms. {The first deterministic term $ \lambda P(V_t^i)v_{\alpha,\EE}(\rho_t^N)dt =-\lambda P(V_t^i)(V_t^i-v_{\alpha,\EE}(\rho_t^N)) dt$, because of \eqref{orthproj} $P(V_t^i)V_t^i=0$ , imposes a drift to the dynamics towards $v_{\alpha,\EE}$, which is the current consensus point at time $t$ as an approximation to the global minimizer, and the term disappears when $V_t^i=v_{\alpha,\mathcal{E}}$. In fact, the consensus point $v_{\alpha,\EE}$ is explicitly computed as in \eqref{ValphaE} and it may lay in general outside $\mathbb S^{d-1}$. This choice of an embedded weighted barycenter is  very simple, compatible with fast computations, and, for a compact manifold as $\mathbb S^{d-1}$, it is a good proxy for a minimizer $v^*$. One could alternatively consider the computation of a weighted barycenter on the manifold
$$
v_{\alpha,\mathcal E}^{\mathbb S^{d-1}}(\rho_t^N) = \arg \min_{v \in \mathbb S^{d-1}} \int_{\mathbb S^{d-1}} d_{\mathbb S^{d-1}}(v,w)^p e^{-\alpha \mathcal E(w)} d\rho_t^N(w),
$$
where $d_{\mathbb S^{d-1}}$ is the (Riemannian) distance on $\mathbb S^{d-1}$, $p>0$, and $\rho_t^N$ is again the particle distribution. However, the computation of $v_{\alpha,\mathcal E}^{\mathbb S^{d-1}}$ is in general not explicit and one may have to solve at each time $t$ a nontrivial optimization problem over the sphere in order to compute $v_{\alpha,\mathcal E}^{\mathbb S^{d-1}}$, the so-called Weber problem. These are all good reasons for choosing the simpler embedded alternative \eqref{ValphaE}.}

The second stochastic term $\sigma |V_t^i - v_{\alpha,\EE}(\rho_t^N)| P(V_t^i)dB_t^i$ introduces a random decision to favor the exploration, whose variance is a function of the distance of particles to the consensus points. In particular, as soon as the consensus is reached, then the stochastic component vanishes. The last term $-\frac{\sigma^2}{2}(V_t^i-v_{\alpha,\EE}(\rho_t^N))^2\frac{(d-1)V_t^i}{|V_t^i|^2}dt$, combined with $P(\cdot)$, it is needed to ensure that the dynamics stays on the sphere despite the Brownian motion component.  
{Namely, this third term stems from It\^{o}'s formula to ensure $d |V_t^i|^2=0$, see \cite[Theorem 2.1]{fhps20-1}.}
We further notice that the dynamics does not make use of any derivative of $\EE$, but only of its pointwise evaluations, which appear integrated in \eqref{ValphaE}. 
Hence, the equation can be in principle numerically implemented at discrete times also for cost functions $\EE$ which are just continuous and with no further smoothness {and the resulting numerical scheme is fully derivative-free}. We require more regularity of $\EE$ exclusively to ensure formal well-posedness of the evolution and for the analysis of its large time behavior, but it is not necessary for its numerical implementation. {A possible discrete-time approximation and resulting numerical scheme, which we consider in this paper is given by the projected Euler-Maruyama method as follows: generate i.i.d. $V_0^i$, $i=1,\ldots,N$ sample vectors according to $\rho_0 \in \mathcal P(\mathbb S^{d-1})$ and iterate for $n=0,1,\dots$
\begin{eqnarray}
\tilde V^i_{n+1} &\gets& V^i_n +\Delta t \lambda P(V_n^i)V_n^{\alpha, \EE} + \sigma |V_n^i - V_n^{\alpha, \EE}| P(V_n^i)\Delta B_n^i \nonumber\\ 
&& \phantom{XXXXXXXXXXX} -\Delta t\frac{\sigma^2}{2}(V_n^i-V_n^{\alpha, \EE})^2(d-1)V_n^i, \label{Intro KViso num}\\
V^i_{n+1} &\gets& \tilde V^i_{n+1}/|\tilde V^i_{n+1}|, \quad i=1,\ldots,N, \nonumber
\end{eqnarray}
for 
\be
V_n^{\alpha, \EE} = \frac1{N_\alpha} \sum_{j=1}^N w_\alpha^{\EE}(V^j_n)V^j_n,\qquad N_\alpha = \sum_{j=1}^N w_\alpha^{\EE}(V^j_n),
\label{eq:valpha}
\ee
where $w_\alpha^{\EE}(V^j_n)=\exp(-\alpha\EE(V^j_n))$, and 
 $\Delta B_n^i$ are independent normal random vectors $N(0,\Delta t)$. Let us stress however that this is by no means the only possible discretization and we refer to, e.g., \cite{Platen}, for picking a favorite alternative scheme.}
 
 \subsection{Main result and sketch of its proof}

The main result of the present paper establishes the convergence of the discrete- and continuous-time dynamics to global minimizers of $\EE$ under mild smoothness conditions and local coercivity of the function around global minimizers. The analysis goes in two steps:\\
First of all, one needs to establish the large particle limit of the stochastic dynamics \eqref{stochastic Kuramoto-Vicsek}. This first step was already obtained in \cite{fhps20-1}, whose main results  are about the  well-posedness of \eqref{stochastic Kuramoto-Vicsek} and its rigorous mean-field limit - which is an open issue for unconstrained CBO \cite{carrillo2018analytical} -  to the following nonlocal, nonlinear Fokker-Planck equation
\begin{equation}\label{PDE}
\partial_t \rho_t= \lambda \nabla_{\BS^{d-1}} \cdot ((\langle v_{\alpha, \EE }(\rho_t), v \rangle v - v_{\alpha,\EE}(\rho_t) )\rho_t)+\frac{\sigma^2}{2}\Delta_{\BS^{d-1}} (|v-v_{\alpha,\EE}(\rho_t) |^2\rho_t),\quad t>0,~v\in\BS^{d-1}\,,
\end{equation}
with the initial datum $\rho_0\in\mc{P}(\BS^{d-1})$. Here $\rho_t=\rho(t,v)\in \mc{P}(\BS^{d-1})$ is a Borel probabilty measure on $\BS^{d-1}$ and
\[
v_{\alpha,\EE}(\rho_t)  = \frac{\int_{\BS^{d-1}} v \omega_\alpha^\EE(v)\,d\rho_t}{\int_{\BS^{d-1}} \omega_\alpha^\EE(v)\,d\rho_t}.
\]
The operators $\nabla_{\BS^{d-1}} \cdot$ and $\Delta_{\BS^{d-1}} $ denote the divergence and Laplace-Beltrami operator on the sphere $\BS^{d-1}$ respectively. The mean-field limit is achieved through the coupling method \cite{sznitman1991topics,fetecau2019propagation,huang2017error} by introducing the mean-filed dynamics 
satisfying
\begin{align} \label{monoparticle}
d\OV_t^i &= \lambda P(\OV_t^i)v_{\alpha,\EE}(\rho_t)dt + \sigma |\OV_t^i - v_{\alpha,\EE}(\rho_t)| P(\OV_t^i)dB_t^i-\frac{\sigma^2}{2}(\OV_t^i-v_{\alpha,\EE}(\rho_t))^2\frac{(d-1)\OV_t^i}{|\OV_t^i|^2}dt\,,
\end{align}
where $((\OV_t^i)_{t\geq 0 })_{i=1,\dots,N}$ are i.i.d. with common law $(\rho_t)_{t\geq 0}$ satisfying \eqref{PDE}. It  yields the following quantitative form of mean-field limit
\begin{equation}\label{rateN}
\sup_{t \in [0, T]} \sup_{i=1,\dots,N}\mathbb E \left[ |V_t^i-\OV_t^i|^2\right] \leq C N^{-1}, \quad N \to \infty,
\end{equation}
for any $T>0$ time horizon, see \cite[Theorem 3.1 and Remark 3.1]{fhps20-1}. The rate of convergence \eqref{rateN} is not affected by the curse of dimension and, for $\Gamma=\mathbb S^{d-1}$ the constant $C$ depends at most linearly in $d$ and, as a worst case analysis, exponentially in $\alpha$ and in $T$, see \cite[Remark 3.2 and Lemma 3.1]{fhps20-1} respectively.
Besides the well-posedness of \eqref{PDE} in the space of probability measures established in \cite[Section 2.3]{fhps20-1},  for more regular datum $\rho_0$, we prove additionally in the present paper existence and uniqueness of distributional solutions $\rho \in L^2([0,T],H^1(\mathbb S^{d-1}))$ at any finite time $T>0$, see Theorem \ref{thmregularity}. This auxiliary regularity results is needed 
in our convergence analysis.\\
The second step to establish global convergence, which is also carried out in the present paper, is about proving the large time asymptotics of the PDE solution $\rho_t(v)=\rho(t,v)$. In Theorem \ref{thm:mainresult} we show that, for any $\epsilon>0$ there exists suitable parameters $\alpha,\lambda,\sigma$ and well-prepared initial densities $\rho_0$ such that for $T^*>0$ large enough the expected value of the distribution
$E(\rho_{T^*})=\int v d \rho_{T^*}(v)$ is {near a global minimizers $v^*$ of $\EE$, i.e.,
\begin{equation}\label{largetime}
|E(\rho_{T^*}) - v^*| \leq C \epsilon.
\end{equation}}
   The convergence to $E(\rho_{T^*})$ is exponential in time and the rate depends on the parameters $\epsilon, \alpha,\lambda,\sigma$. {We summarize the main result as follows.
\begin{theorem}\label{mainresult00}
Assume $\EE \in C^2(\mathbb S^{d-1})$ and that for any $v \in \mathbb S^{d-1}$ 
		there exists  a minimizer $v^*\in \mathbb S^{d-1}$ of $\EE$ (which may depend on $v$) such that  it holds 
		\begin{equation}\label{coerc}
		|v-v^\ast| \leq  C_0|\EE(v)-\underline \EE|^\beta\,,
		\end{equation}
		where  $\beta, C_0$ are some positive constants and $\underline \EE:=\inf_{v \in \mathbb S^{d-1}}\EE(v)$. We also denote 
$\overline \EE:=\sup_{v \in \mathbb S^{d-1}}\EE(v)$, $C_{\alpha,\EE}=e^{\alpha (\overline \EE-\underline \EE)}$, and $C_{\sigma,d}=\frac{(d-1)\sigma^2}{2}$.
Additionally for any $\epsilon>0$ assume that the initial datum $\rho_0$ and parameters $\lambda, \sigma$ are {\it well-prepared} in the sense of Definition \ref{def:wellprep} for a time horizon $T^*>0$  and a parameter $\alpha^*>0$ large enough, depending on $C_0$ and $\beta$. Then the iterations $\{V_{n}^i:=V_{\Delta t, n}^i: n=0,\dots,n_{T^*}; i=1\dots N\}$ generated by a discrete-time approximation of \eqref{stochastic Kuramoto-Vicsek}  fulfill the following error estimate 
\begin{eqnarray}\label{mainresultXX}
 \mathbb E \left   [\left|\frac{1}{N} \sum_{i=1}^N V_{ n_{T^*}}^i - v^* \right|^2 \right ] 
&\leq&  \underbrace{C_1 (\Delta t)^{2m}}_{Discr.\, err.}+  \underbrace{C_2 N^{-1}}_{{
Mean-field\, lim.}} +  \underbrace{C_3 \epsilon^2}_{{Laplace\, princ.}} \,,
\end{eqnarray}
where $m$ is the order of approximation of the numerical scheme. (For the projected Euler-Maruyama scheme the order is $m=1/2$.) The constant $C_1$ depends linearly on the dimension $d$ and the number of particles $N$, and possibly exponentially on $T^*$ and the parameters $\lambda$ and $\sigma$;  the constant $C_2$ depends linearly on the dimension $d$,  polynomially in $C_{\alpha^*,\EE}$, and exponentially in $T^*$; the constant $C_3$ depends on $C_0$ and $\beta$. The convergence is exponential with rate
\begin{equation}\label{rateofconv}
\lambda\theta- 4C_{\alpha^*,\EE}C_{\sigma,d}>0,
\end{equation}
for a suitable $0<\theta<1$.
\end{theorem}
The detailed proof of this result is reported in Section \ref{sec:mainproof}. We provide here a sketch of it.
\begin{proof} {\bf (Sketch)}. 
We recall the definitions of  expectation and variance of $\rho_t$ as
\begin{equation*}
E(\rho_t):=\int_{\BS^{d-1}} v d \rho_t (v) \quad V(\rho_t):=\frac{1}{2}\int_{\BS^{d-1}} |v-E(\rho_t)|^2 d\rho_t (v).
\end{equation*}
By combining the coercivity condition \eqref{coerc} and the Laplace principle \eqref{Laplace} we show that
$$
\left|\frac{E(\rho_t)}{|E(\rho_t)|}- v^\ast\right|\leq C(C_0,\|\nabla \EE\|_\infty,\beta)\left((C_{\alpha,\EE})^{\beta}V(\rho_t)^{\frac{\beta}{2}}+\varepsilon^\beta\right).
$$
Hence, in order to prove the large time convergence to a global minimizer \eqref{largetime}, we may want to show that the variance is monotonically decreasing to zero with an exponential rate. An explicit computation leveraging the PDE \eqref{PDE} yields
\begin{align*}
\frac{d}{dt}V(\rho_t)&=-\lambda V(\rho_t)\la E(\rho_t),v_{\alpha,\EE}\ra -\frac{\lambda}{2}\frac{v_{\alpha, \EE}^2+1}{2}2V(\rho_t)+\frac{\lambda}{4}\int_{\BS^{d-1}} (E(\rho_t)-v)^2 (v-v_{\alpha, \EE})^2 d\rho_t\\
&\quad+C_{\sigma,d}\int_{\BS^{d-1}}(v-v_{\alpha, \EE})^{2}\la E(\rho_t),v \ra d\rho_t\\
&\leq -\lambda V(\rho_t)\left(\la E(\rho_t),v_{\alpha,\EE}\ra +\frac{v_{\alpha, \EE}^2+1}{2}\right)\\ &\quad+\frac{\lambda}{4}\int_{\BS^{d-1}} (E(\rho_t)-v)^2 (v-v_{\alpha, \EE})^2 d\rho_t+4C_{\alpha,\EE}C_{\sigma,d}V(\rho_t)\,.
\end{align*}
The idea is to balance all the terms on the right-hand side by using the parameters $\lambda,\sigma$ in such a way of obtaining a negative sign. Under assumptions of well-preparation, $V(\rho_t)$ is actually small enough for ensuring
$|E(\rho_t)|\approx \la E,v_{\alpha,\EE}\ra \approx |v_{\alpha, \EE}| \approx 1$, and, thanks to Theorem \ref{thmregularity}, for any $\delta>0$ arbitrarily small 
$$
\frac{\lambda}{4}\int_{\BS^{d-1}} (E(\rho_t)-v)^2 (v-v_{\alpha, \EE})^2 d\rho_t \leq \gamma \lambda V(\rho_t) + \delta,
$$
for a suitable $0<\gamma<2$, $\gamma$ depending on $\delta$. Hence, 
$$V(\rho_t) \leq V(\rho_0)e^{-(\lambda \theta - 4C_{\alpha,\EE}C_{\sigma,d})t} + \delta$$
for $\theta\approx 2-\gamma$, and one concludes the convergence  in finite time $T^*$ as in \eqref{largetime}, under the given error threshold $\epsilon>0$.
By combining now classical results of convergence of numerical approximations\footnote{In this paper we consider numerical approximations by Euler-Maruyama scheme, which converges strongly with order $m=1/2$ \cite{Platen}, see Algorithm \ref{algo:sKV-CBO} in Section 2.2.}  $({V}_{\Delta t,n}^i)_{i=1,\dots,N}$ \cite{Platen}
with the the mean-field approximation \eqref{rateN} \cite{fhps20-1} and the large time behavior \eqref{largetime}, which is proven in detail in Theorem \ref{thm:mainresult} below, we obtain that the expected large time outcome of the numerical approximation to \eqref{stochastic Kuramoto-Vicsek} is about a global minimizer of $\EE$
\begin{align}\label{mainresult}
&\mathbb E \left   [\left|\frac{1}{N} \sum_{i=1}^N V_{\Delta t, n_{T^*}}^i - v^* \right|^2 \right ] \nonumber \\
\lesssim& \mathbb E \left   [\left|\frac{1}{N} \sum_{i=1}^N( V_{\Delta t,n_{ T^*}}^i -  V_{T^*}^i) \right|^2 \right ]  +  \mathbb E \left   [\left|\frac{1}{N} \sum_{i=1}^N (V_{T^*}^i -  \OV_{T^*}^i) \right|^2 \right ]  \notag\\
&+  \mathbb E \left   [\left|\frac{1}{N} \sum_{i=1}^N \OV_{T^*}^i - E(\rho_{T^*})  \right|^2 \right ] + |E(\rho_{T^*}) - v^*|^2 \nonumber  \\
\lesssim& (\Delta t)^{2m} + N^{-1} + \epsilon^2, 
\end{align}
where $m$ is the order of strong convergence of the numerical method. 
\end{proof}
}
{

\subsection{Discussion}
Some comments about the result and its proof are in order.
First of all, we stress that Theorem \ref{mainresult00} is the first and so far the unique complete result of convergence of consensus-based optimizations in the literature.
In fact, the results in \cite{carrillo2018analytical,carrillo2019consensus} are exclusively addressing the large time behavior of the mean-field PDE, because, for consensus-based optimization in the Euclidean space, a mean-field approximation of the type \eqref{rateN} has not been established yet for unbounded $\EE$. Similarly, the convergence proof of the purely numerical scheme in \cite{ha2019convergence} establishes convergence in $\mbox{ess} \inf_{\omega}$ over all possible realizations $\omega$ (significantly weaker than \eqref{mainresultXX}). In particular it
does not provide a rate of convergence in terms of number $N$ of particles. Moreover the result is established under the
simplified assumption that the noise is equal for all particles.} \\
Our proof strategy described above made of a numerical approximation, mean-field limit, and asymptotic analysis parallels a similar approach by Montanari et al. \cite{MeiE7665,javanmard2019analysis} for proving the convergence of stochastic gradient descent to global minimizers in the  training of two-layer neural networks.\\
{The initial datum $\rho_0$ has to be interpreted as the uncertainty on the location of a global minimizer.} The condition of Definition \ref{def:wellprep} of well-preparation of $\rho_0$ may have a locality flavour, i.e., they  essentially require that $\rho_0$ has small variance and simultaneously it not centered too far from a global minimizers $v^*$ of $\EE$. However, in the case the function $\EE$ is symmetric, i.e., $\EE(v)= \EE(-v)$ (as it happens in numerous applications, in particular the ones we present in this paper) and {$C_0>0$ is relatively large for $\beta\geq1$}, then the condition is generically/practically satisfied at least for one of the two global minimizers $\pm v^*$. The convergence result is based on proving the monotone decay of the variance $V(\rho_t)=\int |v- E(\rho_t)|^2 d\rho_t(v)$, see Proposition \ref{mainp}, and this cannot be achieved unless the initial condition is well-prepared. 
In fact, for a non-symmetric function $\EE$, a given unique global minimizer $v^*$, and for a datum $\rho_0$ fully concentrated around the opposite vector $-v^*$, i.e.,  on the other side of the sphere, the variance may start small, but it must grow well before getting small again. Hence, it is not possible for arbitrary $\EE$ and initial datum to have monotone decay of the variance, and we conjecture that the result can be further improved to obtain even more generic initial conditions, but one needs to use a different proving technique.\\
Let us now discuss  the interplay between the different approximations and the constants appearing in \eqref{mainresultXX}. While the constants $C_1,C_2$ in \eqref{mainresultXX} depend explicitly only linearly on the dimension $d$, the constants $C_2$ may depend polynomially on $C_{\alpha^*,\EE}=e^{\alpha^* (\overline \EE-\underline \EE)}$, hence, exponentially in $\alpha^*$. This exponential dependence stems from the worst case analysis due to \cite[Lemma 3.1]{fhps20-1}. Moreover, at this level of generality it is difficult to establish how $\alpha^*$ depends on $d$ as such dependence is strongly affected by the particular objective function $\EE$ and $\rho_0$: to clarify the predicament, in the extreme case where $\rho_0 \approx \delta_{v^*}$ it is
$$\left(-\frac{1}{\alpha}\log\left(\int_{\RR^d}e^{-\alpha\EE(v)}d\rho_0(v)\right)\right)\approx \underline \EE,
$$
independently of $\alpha$ (also for $\alpha$ very small!). Also one does not expect a strong dependence of $\alpha^*$ on $d$ for the case where $C_0>0$ is large for $\beta \geq 1$ with a symmetric behavior of $\EE$ around $v^*$.  Instead, in the worst case scenario we may need $\alpha^* \geq d$ and our estimates may simply reflect the fact that the optimization problem at hand is NP-hard or intrinsically affected by the curse of dimensionality. However, this is by no means the typical situation, as in our numerical experiments in Section \ref{sec:robsub} we show that the method scales well with the dimension also for problems with $d\approx 3000$.
Due to the worst case analysis and the use of Gronwall's inequalities in the literature, constants $C_1,C_2$ may depend also exponentially on $T^*$; however $T^*$ is fixed at the beginning and the initial datum and parameters are assumed to be well-prepared so that the algorithm reaches precisely at the time $T^*$ the expected accuracy. Hence, $T^*$ does not need to be very large if we assume that our initial datum $\rho_0$ offers already a reasonable confidence on the location of a global minimizer. In particular the rate of convergence $\lambda\theta- 4C_{\alpha,\EE}C_{\sigma,d}$ is completely determined by the choices of $\lambda$ and $\sigma$.
The choice of $\lambda>0$ large necessarily implies the discretization parameter $\Delta t$ small, as $C_1$ depends by the worst case analysis and the use of Gronwall's inequalities exponentially in $\lambda$.
Recent work \cite{carrillo2018analytical} on unconstrained consensus-based optimization with {\it anisotropic} noise suggests the possibility of having parameters completely independent of the dimension.

\begin{figure}[tb]
	\includegraphics[scale=0.4]{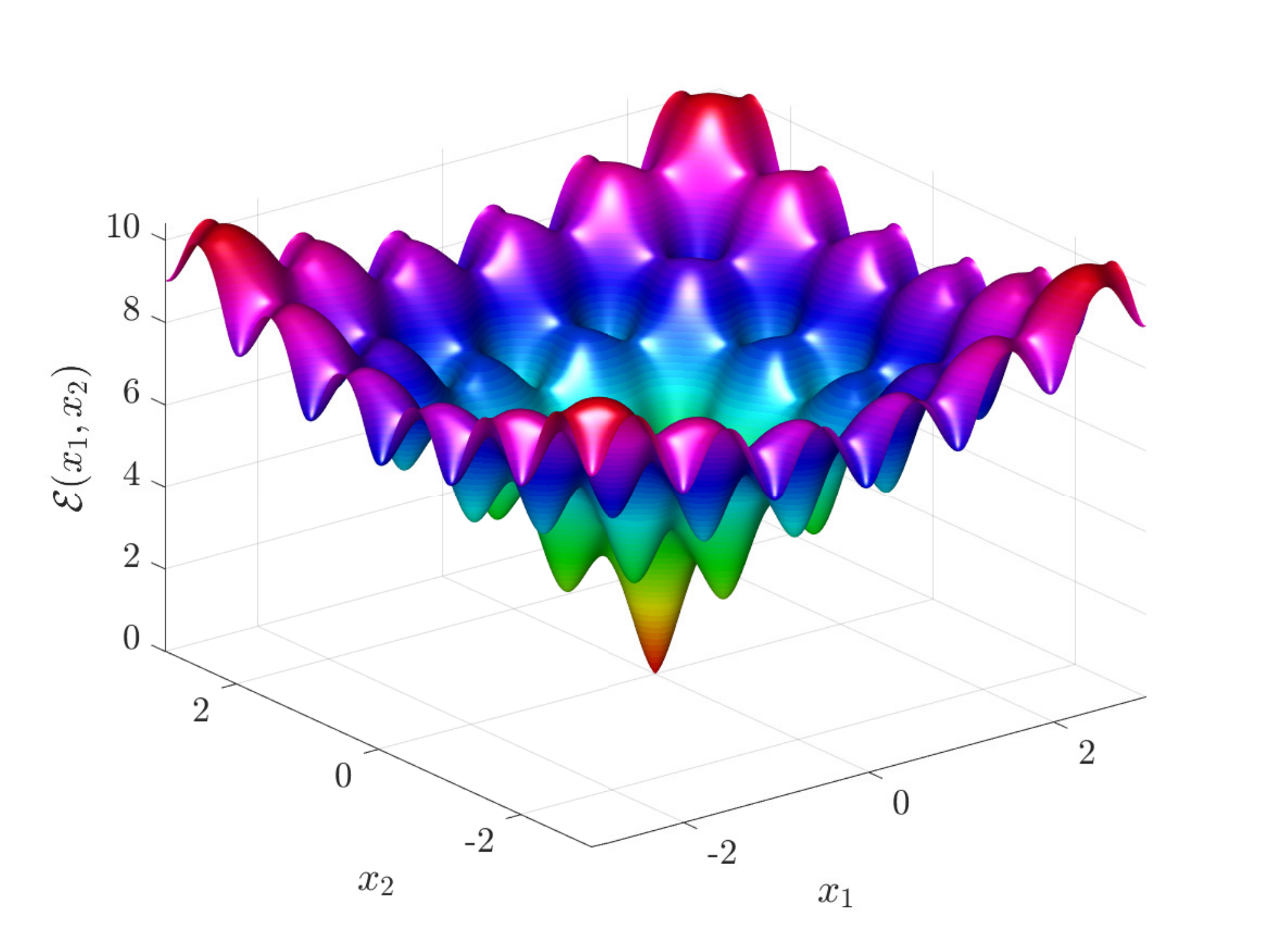}\,
	\includegraphics[scale=0.4]{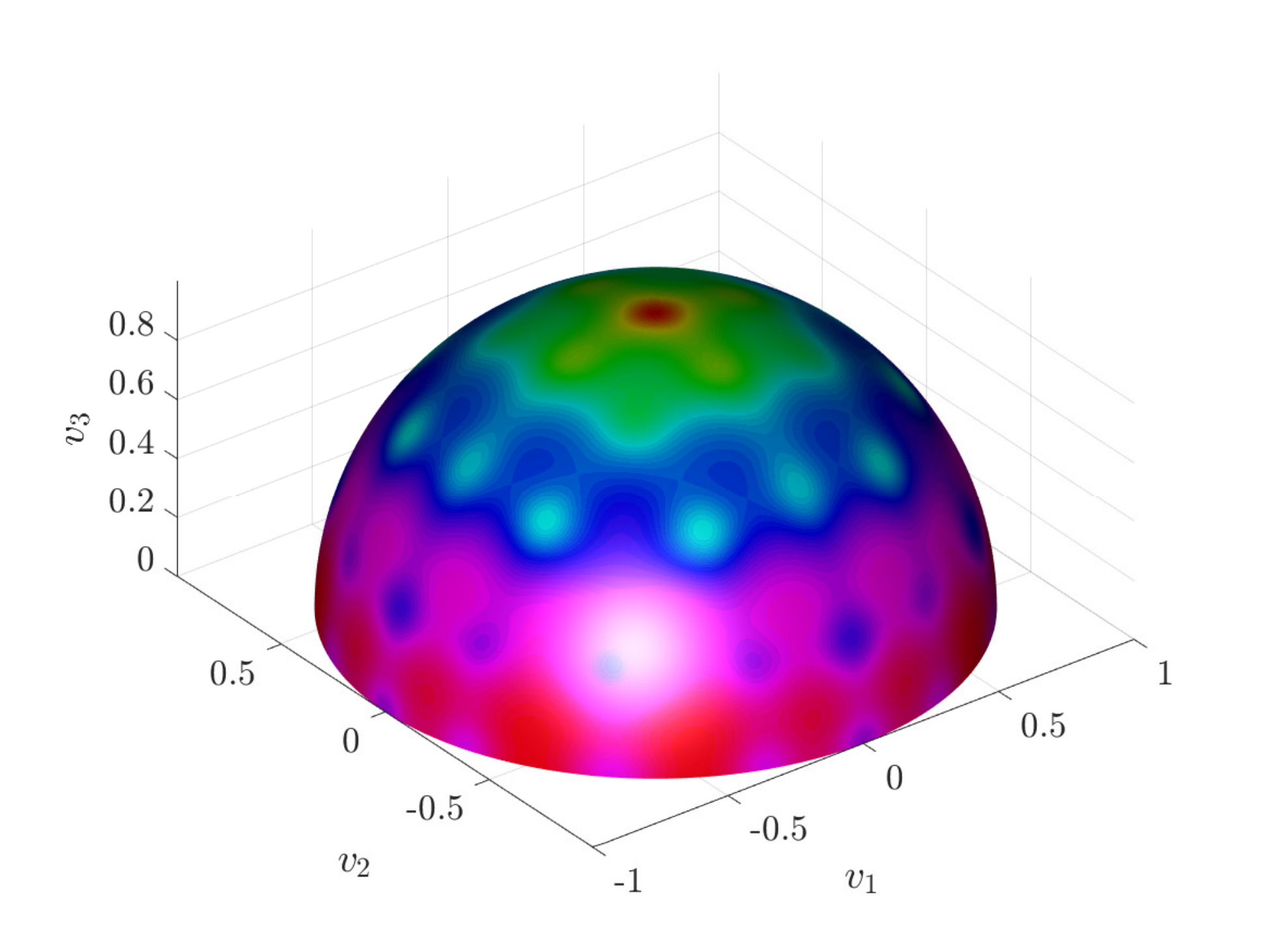}\,
	\caption{The Ackley function for $d=2$ on $[-3,3]^2$ and its representation for $d=3$ in the constrained case over the half sphere $\BS^2$ (right). The global minimum corresponds to the direction $v^*=(0,0,1)^T$.}
	\label{fg:ackley}
\end{figure}

\subsection{Organization of the paper}

The rest of the paper is organized as follows: in Section \ref{numsec} we present and explain right away the numerical implementation, Algorithm \ref{algo:sKV-CBO}, of the  stochastic Kuramoto-Vicsek (sKV) system \eqref{stochastic Kuramoto-Vicsek}. We further propose a few relevant speed-ups, which will be implemented in Algorithm \ref{algo:sKV-CBOfc}. As a warm up, we illustrate the behavior of the algorithms on the synthetic example of the Ackley function over the sphere (see Figure \ref{fg:ackley}) in dimension $d=3$. In the second part of this section, we present applications in signal processing and machine learning, namely the phase retrieval problem and the robust subspace detection and we provide comparisons with state of the art methods. For the robust subspace detection we test the algorithm also in dimension $d\approx 3000$ on the {\it Adult Faces Database} \cite{10kUSAdultFaces} for the  computation of eigenfaces.
These experiments show that the algorithm scales well with the dimension and is extremely versatile (one just needs to modify the definition of the function $\EE$ and the rest  goes with the same code!). The algorithm is able to perform essentially as good as {\it ad hoc} state of the art methods and in some instances it obtains quantitatively better results. 
For the sake of  reproducible research, in the repository {\em https://github.com/PhilippeSu/KV-CBO} we provide the Matlab code, which implements the algorithms on the test cases of this paper. 
In Section \ref{analsec} we provide the analysis of global optimization guarantees, which yield the main error estimate \eqref{mainresultXX}. In Section \ref{sec:aux} we collect proofs of a few auxiliary results.

\section{Numerical Implementation and Tests}\label{numsec}

In this section we report several tests and examples of application of the consensus based optimization (CBO) method based on the stochastic Kuramoto-Vicsek (sKV) system. First, we discuss fast first order discretization methods for the stochastic system, which preserve the dynamics on the multi-dimensional sphere. Implementation aspects and speed-ups are also analyzed. In particular, we derive fast algorithms, which permit to obtain an exponentially diminishing computational cost in time. Next, we test the method and its sensitivity to the choice of the computational parameters with respect to some well-known prototype test functions in high dimensions. Real-life applications are also provided to sustain  the versatility and scalability of the method.

\subsection{Discretization of the sKV system}

We discuss the discretization of the sKV system in It\^{o}'s form
\begin{align} 
dV_t^i &= \lambda P(V_t^i)V_t^{\alpha, \EE} dt + \sigma |V_t^i - V_t^{\alpha, \EE}| P(V_t^i)dB_t^i-\frac{\sigma^2}{2}(V_t^i-V_t^{\alpha, \EE})^2\frac{(d-1)V_t^i}{|V_t^i|^2}dt\,,
\label{eq:SKVI}
\end{align}
with $V{_t^i }\in\BS^{d-1}$, $i=1,\ldots,N$, and 
$$
V_t^{\alpha, \EE}=\sum_{j=1}^{N}\frac{V_t^j\omega_\alpha^\EE(V_t^j)}{\sum_{i=1}^{N}\omega_\alpha^\EE(V_t^i)}=v_{\alpha,\EE}(\rho_t^N). 
$$
First let us remark that for $d=2$ the problem is considerably simpler since the passage to spherical coordinates permits an easy integration of the system by preserving its geometrical nature of motion on $\BS^1$. However, for arbitrary dimensions this is more complicated and we must integrate the stochastic system in the vector form  \eqref{eq:SKVI}. We refer to \cite{Platen} for an introduction to numerical methods for SDEs and to \cite{HLW} for deterministic time discretizations, which preserve some geometrical properties of the solution. 

Let us denote $|V|=\|V\|_2=\langle V, V\rangle^{\frac12}$ the Euclidean norm. A simple geometrical argument allows to prove the following observation:
\begin{lemma}\label{lemnorm}
	Let us consider a one step time discretization of \eqref{eq:SKVI} in the general form
	\be
	V_{n+1}^i=V_{n}^i+\Phi(\Delta t, V^i_n,V^i_{n+1},\xi^i_n)
	\label{eq:gen}
	\ee
	where the function $\Phi(\Delta t,\cdot,\cdot,\xi^i_n):\RR^{2d}\to\RR^{d}$ defines the method, $\Delta t>0$ is the time step, $V_n^i \approx V^i_t|_{t=t^n}$, $t^n=n\Delta t$ and $\xi^i_n$ are independent random variables.
	
	Then 
	\be
	| V^i_{n+1} |^2= |V^i_n|^2 
	\label{eq:npres}
	\ee
	if and only if 
	\be
	\langle \Phi(\Delta t, V^i_n,V^i_{n+1},\xi^i_n), V^i_{n+1}+V^i_n \rangle=0.
	\ee
\end{lemma}

This shows that $\Phi(\Delta t, V^i_n,V^i_{n+1},\xi^i_n)$ must be orthogonal to $V^i_{n+1}+V^i_n$ in order to preserve the norm and, consequently, to obtain one step methods satisfying \eqref{eq:npres} we have to resort to implicit methods.

For example, it is immediate to verify that the Euler-Maruyama method 
\be
V^i_{n+1} = V^i_n +\Delta t\lambda P(V_n^i)V_n^{\alpha, \EE} dt + \sigma |V_n^i - V_n^{\alpha, \EE}| P(V_n^i)\Delta B_n^i-\Delta t\frac{\sigma^2}{2}(V_n^i-V_n^{\alpha, \EE})^2\frac{(d-1)V_n^i}{|V_n^i|^2}\,,
\label{eq:EM}
\ee
where $\Delta B_n^i=B^i_{t^{n+1}}-B^i_{t^{n}}$ are independent normal random variables $N(0,\Delta t)$ with mean zero and variance $\Delta t$, is not invariant with respect to the norm of $V^i_n$.

A method that preserves the norm is obtained by modifying the Euler-Maryuama method as follows
\[
V^i_{n+1} = V^i_n +\Delta t\lambda P(V_{n+\frac12}^i)(V_n^{\alpha, \EE} - V^i_n) dt + \sigma |V_n^i - V_n^{\alpha, \EE}| P(V_{n+\frac12}^i)\Delta B_n^i-\Delta t\frac{\sigma^2}{2}(V_n^i-V_n^{\alpha, \EE})^2\frac{(d-1)V_{n+\frac12}^i}{|V_{n+\frac12}^i|^2}\,,
\label{eq:MP1}
\]
where $V^i_{n+\frac12}={V^i_{n+1}+V^i_n}$ and, for consistency, we have the term $-V^i_n$ in the alignment process since now $P(V^i_{n+\frac12})V^i_n \neq 0$.
By similar arguments, we can construct implicit methods of weak order higher than one which preserve the norm of the solution.  

Implicit methods, however, due to the nonlinearity of the projection operator $P(\cdot)$ require the inversion of a large nonlinear system. This represents a serious drawback for our purposes, where efficiency of the numerical solver is fundamental. 

In order to promote efficiency, we consider instead explicit one-step methods that preserve the geometric properties  by 
adopting a projection method at each time step for the iterations to stay on the sphere \cite{HLW}. This corresponds to solve the stochastic differential problem under the algebraic constraint to preserve the norm.

Since we are on the unit hypersphere, we simply divide the numerical approximation by its Euclidean norm to get a vector of length one. This class of schemes has the general form
\be
\left\lbrace
\begin{aligned}
	\widetilde V^i_{n+1}&=V^i_{n}+\Phi(\Delta t, V^i_n,\widetilde V^i_{n+1},\xi^i_n),\\
	V^i_{n+1}&=\frac{\widetilde V^i_{n+1}}{|\widetilde V^i_{n+1}|}.
	\label{eq:gen2}
\end{aligned}
\right.
\ee
We keep the dependence from $\widetilde V^{n+1}$ on the right hand side to include semi-implicit methods with better stability properties then the Euler-Maruyama scheme. One example is obtained by the following integration scheme
\[
\widetilde V^i_{n+1}=V^i_{n} + \Delta t \lambda P(V_n^i)V_n^{\alpha, \EE} + \sigma |V_n^i - V_n^{\alpha, \EE}| P(V_n^i)\Delta B_n^i-\Delta t \frac{\sigma^2}{2}(V_n^i-V_n^{\alpha, \EE})^2(d-1)\widetilde V_{n+1}^i
\] 
which can be written explicitly as
\be
\widetilde V^i_{n+1}=\frac{1}{1+\Delta t \frac{\sigma^2}{2}(V_n^i-V_n^{\alpha, \EE})^2(d-1)}\left(V^i_{n} + \Delta t \lambda P(V_n^i)V_n^{\alpha, \EE} + \sigma |V_n^i - V_n^{\alpha, \EE}| P(V_n^i)\Delta B_n^i\right).
\label{eq:SI}
\ee
In our experiments, since efficiency of the numerical solver is of paramount importance, we rely on projection methods of the type \eqref{eq:gen2} based on the simple Euler-Maruyama scheme \eqref{eq:EM} or the semi-implicit scheme \eqref{eq:SI}. 
\begin{remark}
	Another popular approach is based on simulating the two fundamental processes characterizing the dynamics by a splitting method on the time interval $[n\Delta t,(n+1)\Delta t]$ 
	\be
	\left\lbrace
	\begin{aligned} 
		d \tilde V_t^i &= \lambda P(\tilde V_t^i)\tilde V_t^{\alpha, \EE} dt\,,\qquad \tilde V_0^i = V_t^i|_{t=n\Delta t},\\
		dV_t^i &= \sigma |V_t^i - V_t^{\alpha, \EE}| P(V_t^i)dB_t^i-\frac{\sigma^2}{2}(V_t^i-V_t^{\alpha, \EE})^2\frac{(d-1)V_t^i}{|V_t^i|^2}dt\,,\qquad V_0^i = \tilde V_t^i|_{t=(n+1)\Delta t},
		\label{eq:EMs}
	\end{aligned}
	\right.
	\ee

	where the first step is a standard alignment dynamics over the hypersphere and the second step corresponds to solve a Brownian motion with variance $\sigma^2(V_t^i-V_t^{\alpha, \EE})^2$ on the unit hypersphere. Typically, the approximated value of $V_t^{\alpha, \EE}$ is kept constant in a splitting time step to avoid computing it twice and increasing the computational cost. This approach would allow to solve the first step using standard structure preserving ODEs approaches \cite{HLW} and to use specific simulation methods for the Brownian motion over the hypersphere in the second step \cite{GPB, MMB}. We will leave to further study the possibility to apply methods in the splitting form \eqref{eq:EMs}.
\end{remark}

\subsection{Implementation aspects and generalizations}

First let us point out that the set of three computational parameters, $\Delta t$, $\sigma$ and $\lambda$, defining the discretization scheme can be reduced since we can rescale the time by setting
\[
\tau=\lambda \Delta t ,\qquad \nu^2 = \frac{\sigma^2}{\lambda},  
\]
to obtain a scheme which depends only on two parameters $\tau$ and $\nu$. In practice, we can simply assume $\lambda=1$ and keep the original notations. Starting from a set of computational parameters and a given objective function ${\EE}(\cdot)$ defined on $\BS^{d-1}$, the simplest KV-CBO method is described in Algorithm \ref{algo:sKV-CBO}.

\begin{algorithm}[h]
	\KwIn{$\Delta t$, $\sigma$, $\alpha$, $d$, $N$, $n_T$ and the function ${\EE}(\cdot)$}
	Generate $V_0^i$, $i=1,\ldots,N$ sample vectors uniformly on $\BS^{d-1}$; 
	
	\For{$n=0$ {\bf to} $n_T$}{
		Generate $\Delta B_n^i$ independent normal random vectors $N(0,\Delta t)$;
		
		Compute $V_n^{\alpha, \EE}$;
		
		$\tilde V^i_{n+1} \gets V^i_n +\Delta t P(V_n^i)V_n^{\alpha, \EE} + \sigma |V_n^i - V_n^{\alpha, \EE}| P(V_n^i)\Delta B_n^i-\displaystyle\Delta t\frac{\sigma^2}{2}(V_n^i-V_n^{\alpha, \EE})^2(d-1)V_n^i$,
		$V^i_{n+1} \gets \tilde V^i_{n+1}/|\tilde V^i_{n+1}|$, $i=1,\ldots,N$;
	}
	\caption{KV-CBO}
	\label{algo:sKV-CBO}
\end{algorithm}

{The approximation order of the projected Euler-Maruyama method as in KV-CBO is $m=1/2$. In fact, the only difference with respect to the classical Euler-Maruyama method is the post-projection onto the sphere $V^i_{n+1} \gets \tilde V^i_{n+1}/|\tilde V^i_{n+1}|$. We show in the proof of Theorem \ref{mainresult00} in Section \ref{sec:mainproof} that this may introduce an error of at most order $\Delta t$, preserving the order of convergence.}

{Let us now discuss briefly about the complexity of the scheme as optimization method in order to place the discussion in the correct frame. It is of utmost importance to stress once again that the method is of $0$-order, i.e., it is derivative-free. Hence, for those problems for which computing derivatives of the objective function is an unfeasible task, either because of complexity or because of non-differentiability, the KV-CBO is necessarily superior in terms of complexity than  first  order methods such as gradient descent or second order methods such as Newton method.
Note, in particular, that the computational cost for a single time step of KV-CBO is ${\mathcal O}(N)$, the minimum cost to evolve a system of $N$ particles since $V_n^{\alpha, \EE}$ is the same for all agents. Let us however mention that the KV-CBO is highly and very easily parallelizable and therefore, on a multi-processor parallel machine, the method can be easily reduced to complexity $\mathcal O(N/\mathcal P)$, where $\mathcal P$ is the number of processors.} The algorithm may be complemented with a suitable stopping criterion, for example checking consensus using the quantity
\be
\frac1{N}\sum_{i=1}^N |V^i_n-V_n^{\alpha,\EE}| \leq \varepsilon, 
\label{eq:cons}
\ee
or checking, as in \cite{carrillo2019consensus}, for $p\geq 0$  that
\be
| V_{n+1}^{\alpha,\EE}-V_{n-p}^{\alpha,\EE}| \leq \varepsilon,
\ee
for a given tolerance $\varepsilon$. In point 5 of Algorithm \ref{algo:sKV-CBO} we used the Euler-Maruyama discretization \eqref{eq:EM}, similarly one could use the semi-implicit method \eqref{eq:SI}. The computational parameters $\Delta t$, $\sigma$ and $\alpha$ can in practice be adaptively modified from step to step to improve the performance of the method.
In the sequel we analyze in more detail some computational aspects and speed ups related to Algorithm \ref{algo:sKV-CBO}. 

\subsubsection*{Sampling over $\BS^{d-1}$}

First let us discuss point $1$ of algorithm \ref{algo:sKV-CBO}, namely how to generate points uniformly over the $d$-dimensional sphere. 
Despite the fact that our theoretical results would suggest to use a more concentrated measure $\rho_0$ to generate the initial points, see Definition \ref{def:wellprep}, the uniform distribution is likely the simplest to be realized and it does certainly not induce initial bias towards any direction. 
Even though many methods have been designed for low dimension $d \leq 3$, very few of them can be extended to large dimensions. Therefore, the one that is often used for a $d$-dimensional sphere is the method of normalized Gaussians first proposed by Muller and later by Marsaglia \cite{Muller, Marsaglia}. The method is extremely simple, and exploits the non-obvious relationship between a uniform distribution on the sphere and the normal distribution. More precisely, to pick a random point on a $d$-dimensional sphere one first generates $d$ standard normal random variables $\xi_1, \xi_2, \ldots, \xi_d \sim N(0,1)$, then the distribution of the vectors of components
\be
v_k=\frac{\xi_k}{\sqrt{\xi_1^2+\ldots+\xi_d^2}}, \qquad k=1,\ldots,d
\ee
coincides with the uniform one over the hypersphere $\BS^{d-1}$.

\subsubsection*{Evaluation of $V_n^{\alpha,\EE}$}

Let us observe that the computation of $V_n^{\alpha, \EE}$, points 2 and 6 of Algorithm \ref{algo:sKV-CBO}, is crucial and that a straightforward evaluation using
\be
V_n^{\alpha, \EE} = \frac1{N_\alpha} \sum_{j=1}^N w_\alpha^{\EE}(V^j_n)V^j_n,\qquad N_\alpha = \sum_{j=1}^N w_\alpha^{\EE}(V^j_n),
\label{eq:valpha}
\ee
where $w_\alpha^{\EE}(V^j_n)=\exp(-\alpha\EE(V^j_n))$, is generally numerically unstable since for large values of $\alpha \gg 1$ the value of $N_\alpha$ is close to zero. On the other hand, the use of large values of $\alpha$ is essential for the performance of the method. A practical way to overcome this issue is based on the following numerical trick
\begin{eqnarray*}
	\frac{w_\alpha^{\EE}(V^j_n)}{N_{\alpha}} &=& \frac{\exp(-\alpha\EE(V^j_n))}{\sum_{j=1}^N \exp(-\alpha\EE(V^j_n))}\cdot \frac{\exp(\alpha\EE(V_n^*))}{\exp(\alpha\EE(V_n^*))}\\
	&=& \frac{\exp(-\alpha(\EE(V_n^j)-\EE(V_n^*)))}{\sum_{j=1}^N \exp (-\alpha(\EE(V_n^j)-\EE(V_n^*)))}
\end{eqnarray*}
where
\be
V_n^{*} := \argmin_{V\in\{V_n^i\}^N_{i=1}} \EE(V)
\label{eq:min}
\ee
is the location of the particle with the minimal function value in the current population. This ensures that for at least one particle $V^j_n=V_n^{*}$ , we
have $\EE(V_n^j)-\EE(V_n^*) = 0$ and therefore, $\exp(-\alpha(\EE(V_n^j)-\EE(V_n^*)))=1$. For the sum this
leads to $\sum_{j=1}^N \exp (-\alpha(\EE(V_n^j)-\EE(V_n^*)))\geq 1$, so that the division does not induce a
numerical problem. In the numerical simulations we will always compute the weights by the above strategy. Note that, the evaluation of \eqref{eq:min} has linear cost, and does not affect the overall cost. The computation of $V_n^{\alpha, \EE}$ may be accelerated by using the random approach presented in \cite{AlPa} (see Algorithm 4.7). Namely, by considering a random subset $J_M$ of size $M < N$ of the indexes $\{1,\ldots,N\}$ and computing
\be
V_n^{\alpha, \EE,J_M} = \frac1{N^{J_M}_\alpha} \sum_{j\in J_M} w_\alpha^{\EE}(V^j_n)V^j_n,\qquad N^{J_M}_\alpha = \sum_{j\in J_M} w_\alpha^{\EE}(V^j_n).  
\label{eq:valphaM}
\ee
Similarly, we will stabilize the above computation by centering it to
\be
V_n^{J_M,*} := \argmin_{V\in\{V_n^j\}_{j\in J_M}} \EE(V).
\label{eq:min2}
\ee
The random subset is typically chosen at each time step in the simulation.  

\begin{remark}
	As a further randomization variant, at each time step, we may partition particles into disjoint subsets $J^k_M$, $k=1,\ldots,S$ of size $M$ such that $SM=N$ and compute the evolution of each batch separately (see \cite{JLJ, carrillo2019consensus} for more details). Since the computational cost of the CBO method is linear, unlike \cite{AlPa, JLJ, HVP} these randomization techniques can accelerate the simulation process (and eventually improve the particles exploration dynamic thanks to additional stochasticity), but do not reduce the overall asymptotic cost ${\mathcal O}(N)$. 
\end{remark}

\subsubsection*{Fast method}

Using a constant number of particles is not the most efficient way to simulate the trend towards equilibrium of a system, typically because we can use some (deterministic) information on the steady state to speed up the method. In the case of CBO methods, asymptotically the variance of the system tends to vanish because of the consensus dynamics, see Proposition \ref{mainp}. So, we may accelerate the simulation by discarding particles in time accordingly to the variance of the system \cite{AlPa}. This also influences the computation of $V_n^{\alpha,\EE}$ by increasing the randomness and reducing the possibilities to get trapped in a local minimum. For a set of $N_n$ particles we define the empirical variance at time $t^n = n\Delta t$ as
\[
\Sigma_n = \frac1{N_n}\sum_{j=1}^{N_{n}} (V_n^j-\bar{V}_n)^2,\qquad \bar{V}_n = \frac1{N_{n}}\sum_{j=1}^{N_n} V^j_n.
\]
When the trend to consensus is monotone, that is $\Sigma_{n+1} \leq \Sigma_n$, we can discard particles uniformly in the next time step $t^{n+1}=(n+1)\Delta t$ accordingly to the ratio $\Sigma_{n+1}/\Sigma_n \leq 1$, without affecting their theoretical distribution. One way to realize this is to define the new number of particles as 
\be
N_{n+1}=\left[\!\!\left[N_n \left(1+\mu\left(\frac{\widehat{\Sigma}_{n+1}-\Sigma_{n}}{\Sigma_{n}}\right)\right)\right]\!\!\right]
\ee
where $[\![\,\cdot\,]\!]$ denotes the integer part, $\mu\in [0,1]$ and
\[
\widehat{\Sigma}_{n+1} = \frac1{N_n}\sum_{j=1}^{N_{n}} (V_{n+1}^j-\widehat{{V}}_{n+1})^2,\qquad \widehat{{V}}_{n+1} = \frac1{N_{n}}\sum_{j=1}^{N_n} V^j_{n+1}.
\]
For $\mu=0$ we have the standard algorithm where no particles are discarded whereas for $\mu=1$ we achieve the maximum speed up. We implement the details of the method, which includes the speed-up techniques just discussed, in Algorithm \ref{algo:sKV-CBOfc}. As before we fix $\lambda=1$. 

\begin{algorithm}
	\KwIn{$\Delta t$, $\sigma$, $\alpha$, $d$, $N$, $n_T$, $\mu$, $M$ and the function ${\EE}(\cdot)$}
	Generate $V_0^i$, $i=1,\ldots,N_0$ sample vectors uniformly on $\BS^{d-1}$; 
	
	Compute the variance $\Sigma_0$ of $V_0^i$ and set $N_0=N$;
	
	\For{$n=0$ {\bf to} $n_T$}{
		Generate $\Delta B_n^i$ independent normal random vectors $N(0,\Delta t)$;
		
		Compute $V_n^{\alpha, \EE}$ from \eqref{eq:valphaM} if $M \leq N_n$ otherwise use \eqref{eq:valpha};
		
		$\tilde V^i_{n+1}\gets V^i_n +\Delta t P(V_n^i)V_n^{\alpha, \EE} + \sigma |V_n^i - V_n^{\alpha, \EE}| P(V_n^i)\Delta B_n^i-\displaystyle\Delta t\frac{\sigma^2}{2}(V_n^i-V_n^{\alpha, \EE})^2(d-1)V_n^i$
		$V^i_{n+1}\gets \tilde V^i_{n+1}/|\tilde V^i_{n+1}|$, $i=1,\ldots,N_n$;
		
		Compute the quantity $\widehat{\Sigma}_{n+1}$ from $V_{n+1}^i$;
		
		Set $N_{n+1}\gets [\![N_n \left(1+\mu\left((\widehat{\Sigma}_{n+1}-\Sigma_{n})/\Sigma_{n}\right)\right)]\!]$ and discard uniformly $N_{n}-N_{n+1}$ samples;
		
		Compute the variance $\Sigma_{n+1}$ of $V_{n+1}^i$;
	}
	\caption{Fast KV-CBO}
	\label{algo:sKV-CBOfc}
\end{algorithm}
Typically, a minimum bound $N_{min}$ of the number of particles is adopted to guarantee that $N_n \geq N_{min}$ during the simulation and the variance reduction test is performed every fixed amount of iterations to avoid fluctuations effects.

\subsubsection*{Adaptive Parameters}

Our main theoretical result Theorem \ref{thm:mainresult} and condition \eqref{lamsig} establish that, once $N$ is large, for $\sigma$ small enough and $\alpha$ large enough,  Algorithm \ref{algo:sKV-CBO} will converge near to a global minimizer. One important aspect, as in many metaheuristic algorithms, concerns 
the choice of the parameters in the method. The adaptation of hyperparameters in multi-particle optimization is a well-known problem, which deserves a proper discussion, see, e.g., \cite{escalante2009:psms}.  In our case, we observed
that decreasing $\sigma$ and increasing $\alpha$ during the iterative
process leads to improved results in term of convergence and accuracy. 
One strategy, therefore, would be to start with a large $\sigma$ and  
to reduce it progressively over time 
as a function of a suitable indicator of
convergence, for example the average variance of the solution or the
relative variation of $V_{\alpha}$ over time. This can be realized starting from
$\sigma_0$ and by decreasing it as
\begin{equation}
\sigma_{n+1} = \frac{\sigma_{n}}{\tau}, 
\label{eq:sigmar}
\end{equation}
where $\tau>1$ is a constant. Other techniques, of course, can be used to decrease 
$\sigma$, for example following a cooling strategy as in the Simulated Annealing approach \cite{holley1988simulated}. 
In \cite{carrillo2019consensus} it has been proposed to reduce $\sigma$ independently
of the solution behavior, as a function of the initial value $\sigma_0$ and the 
number of iterations. This corresponds to take $\sigma_{n+1}={\sigma_n}/({\sigma_0}\log(n+1))$
in \eqref{eq:sigmar}. As a result of these strategies, the noise level in the system will
decrease in time. Note that, since we need $\lambda \gg \sigma^2(d-1)$ (see formula \eqref{lamsig} below) to achieve consensus
in the system, this approach allows to start initially with a larger $\sigma$ which permits to
explore the surrounding area well before entering the consensus regime. 

Similarly, it might not be beneficial to start with a large $\alpha$ from the beginning. In fact, in this case the $V_\alpha$ would right away equal the particle with the lowest energy and all the other particles will be forced to move towards this particle, with a lower impact on the initial exploration mechanism. Therefore, we can start with an initial value $\alpha_0$ and gradually increase it to a maximum value $\alpha_{\max}$ accordingly to an appropriate convergence indicator, or independently as a function of the number 
of iterations. In particular, large values of $\alpha$ at the end of the simulation process are essential to achieve high accuracy in the computation of the minimum.

\subsection{Numerical experiments for the Ackley function}
\subsubsection*{Minimizing the Ackley function in dimension $d=3$}

First we consider the behavior of the model and its mean field limit in the case $d=3$ for computing the minimum of the Ackley function\footnote{\it https://en.wikipedia.org/wiki/Test$\_$functions$\_$for$\_$optimization} constrained over the sphere
\be
\EE(V)= -A \exp\left(-a\sqrt{\frac{b^2}{d}\sum_{k=1}^{d} (V_k-v^*_k)^2}\right)-\exp\left(\frac1{d}\sum_{k=1}^{d} \cos(2\pi b(V_k-v^*_k))\right)+e+B,
\ee
with $A=20$, $a=0.2$, $b=3$, $B=20$ and $V=(V_1,\ldots,V_d)^T$ with $|V|=1$.

The global minimum is attained at $V=v^*$. In Figure \ref{fg:ackley} we report the Ackley function for $d=3$ over the half sphere $V_3 \geq 0$.
Note that, this problem differs from the standard minimization of the Ackley function over the whole space $\RR^d$ since KV-CBO  operates through unitary vectors over the hypersphere.

\begin{figure}[htb]
	\hskip .2cm
	\includegraphics[scale=0.18]{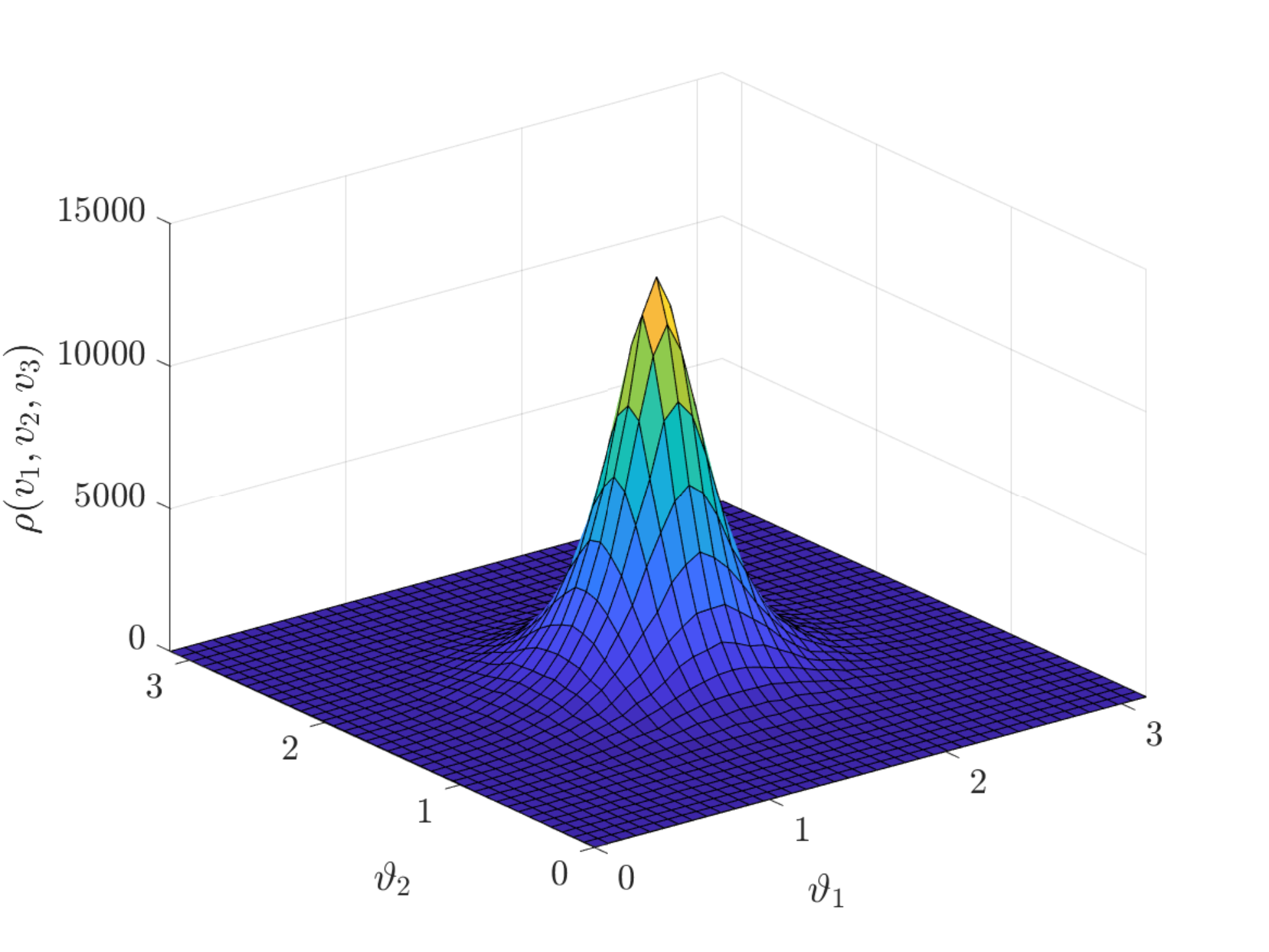}\hskip -.25cm
	\includegraphics[scale=0.18]{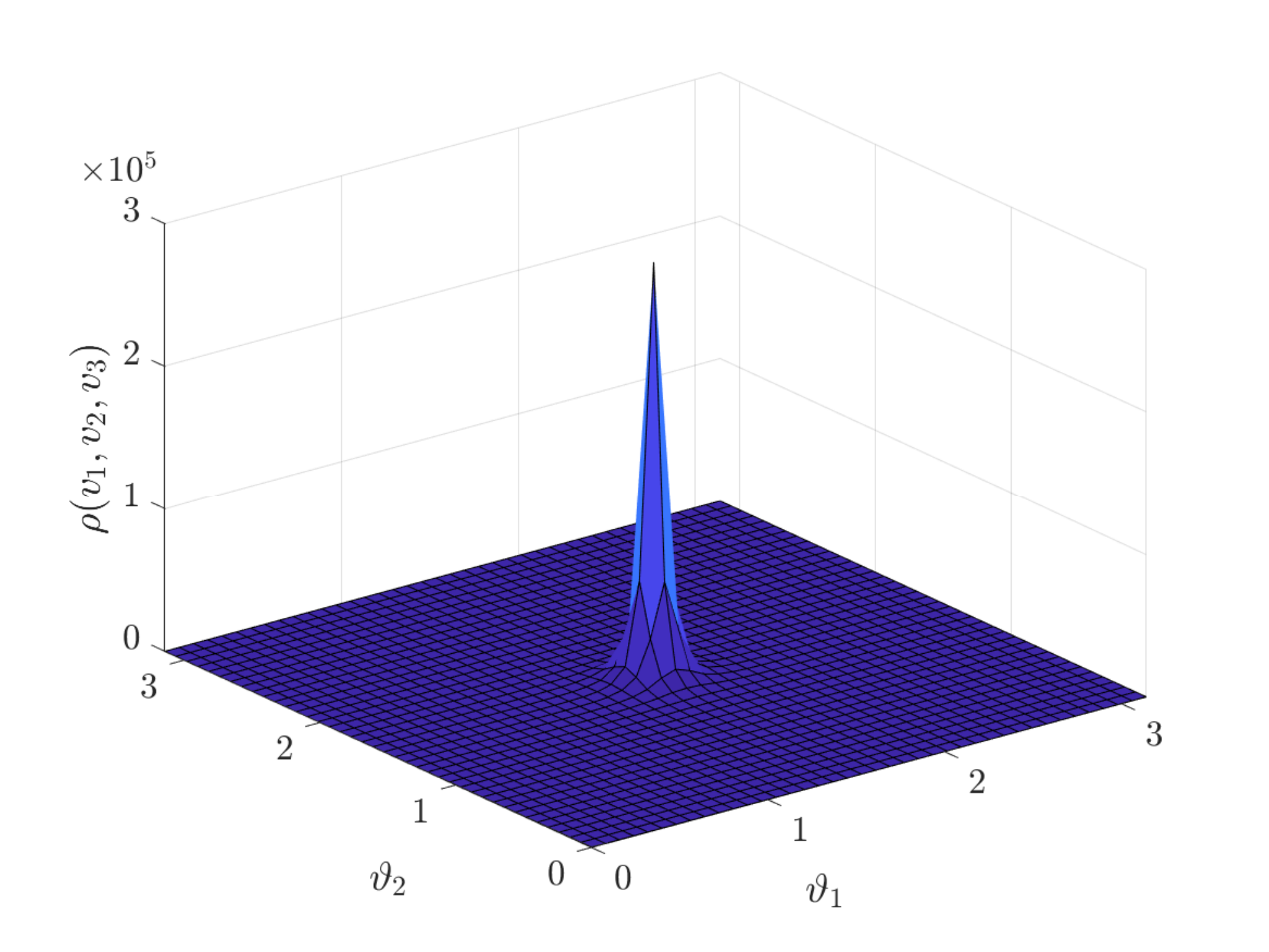}\hskip .4cm
	\includegraphics[scale=0.18]{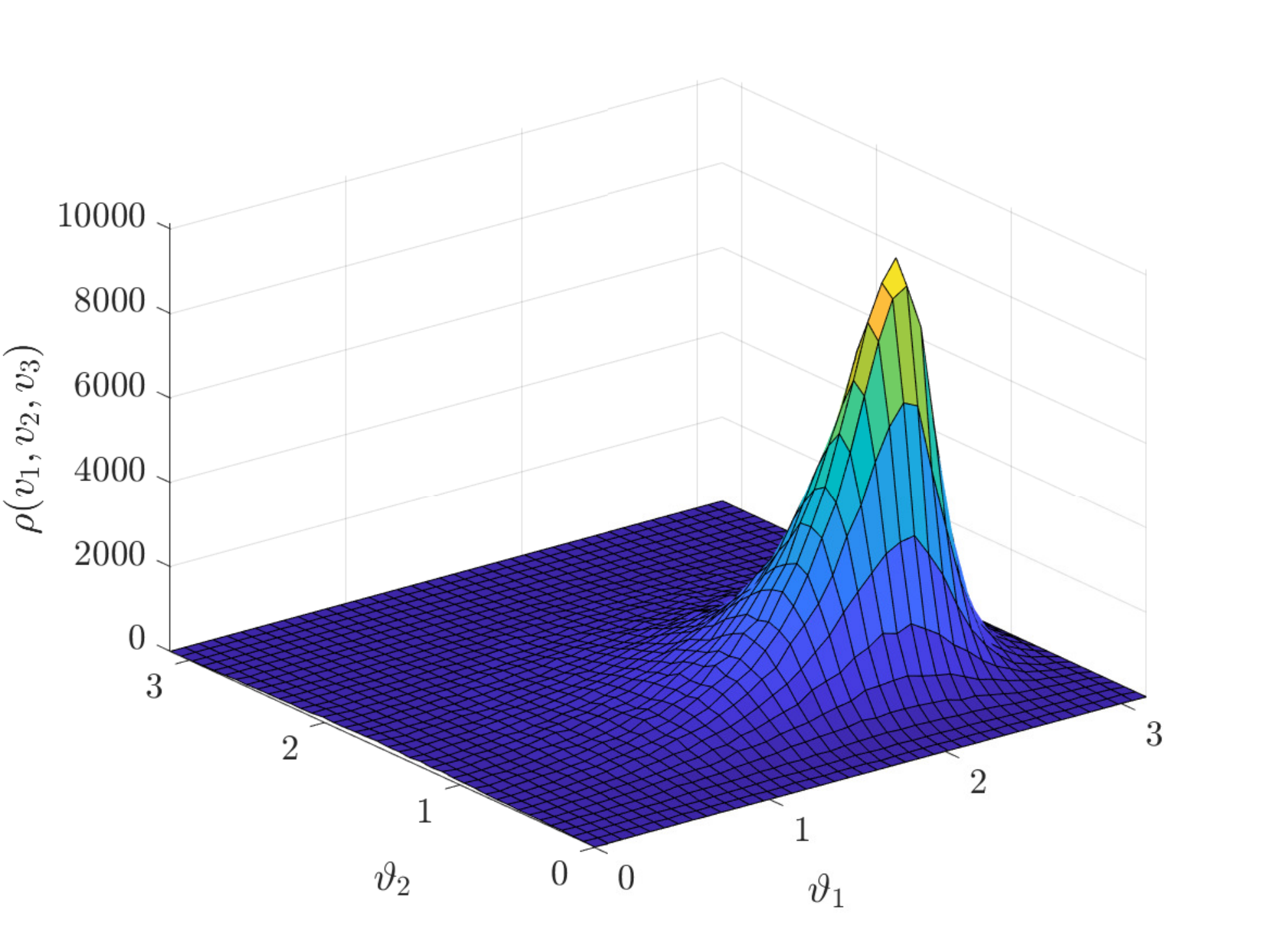}\hskip -.25cm
	\includegraphics[scale=0.18]{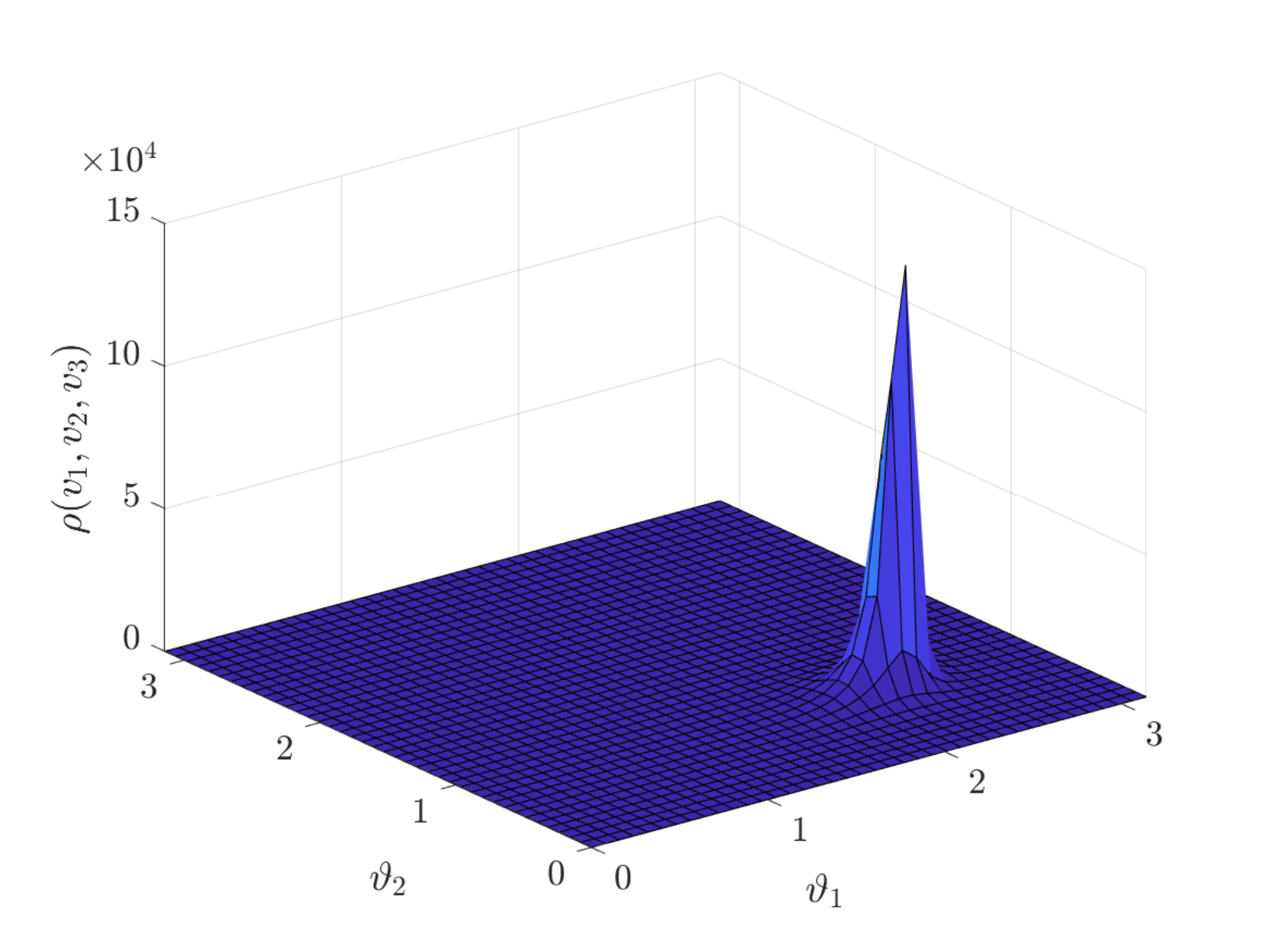}\\[-.2cm]
	\includegraphics[scale=0.41]{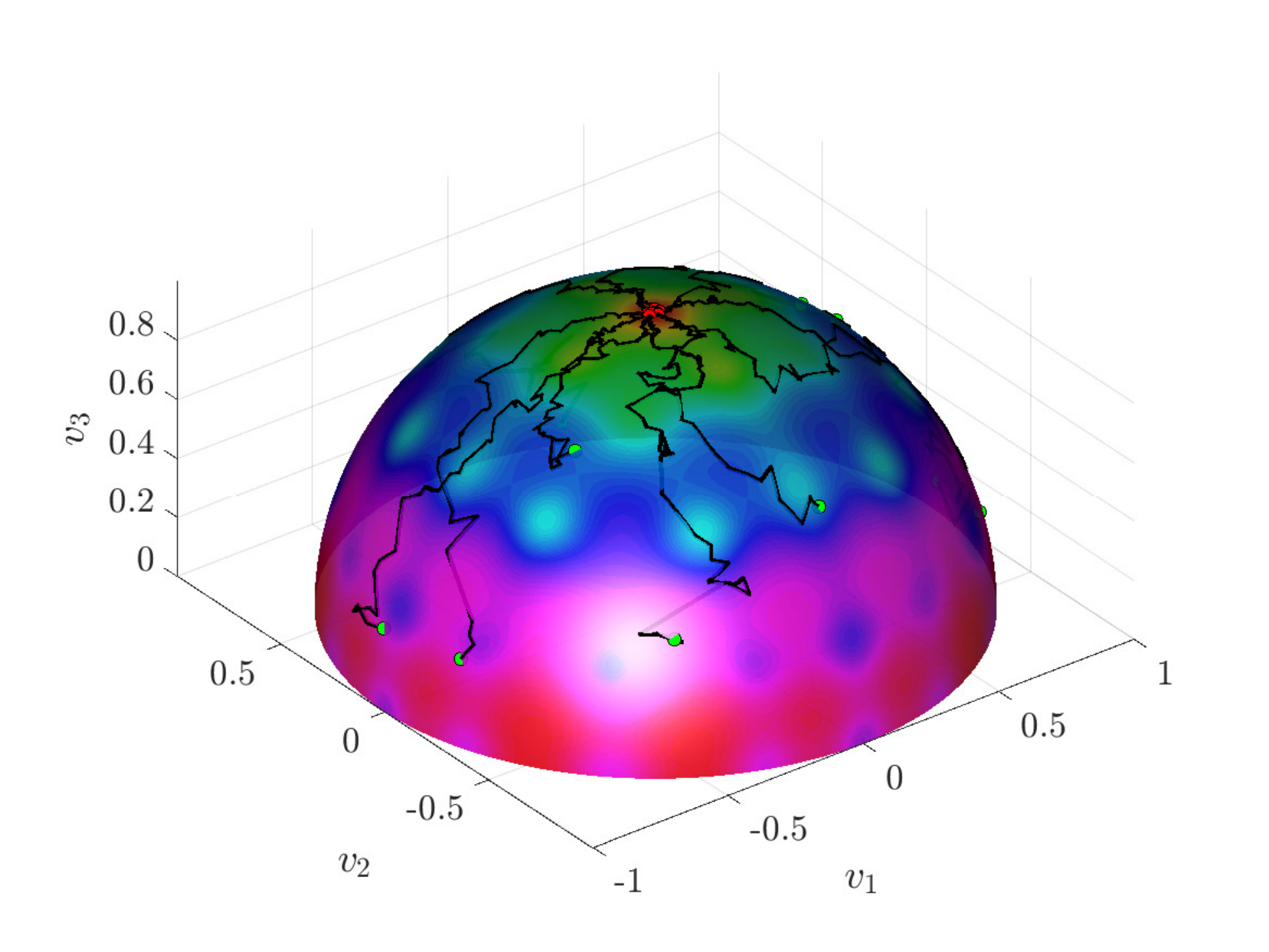}\,\hskip .2cm
	\includegraphics[scale=0.41]{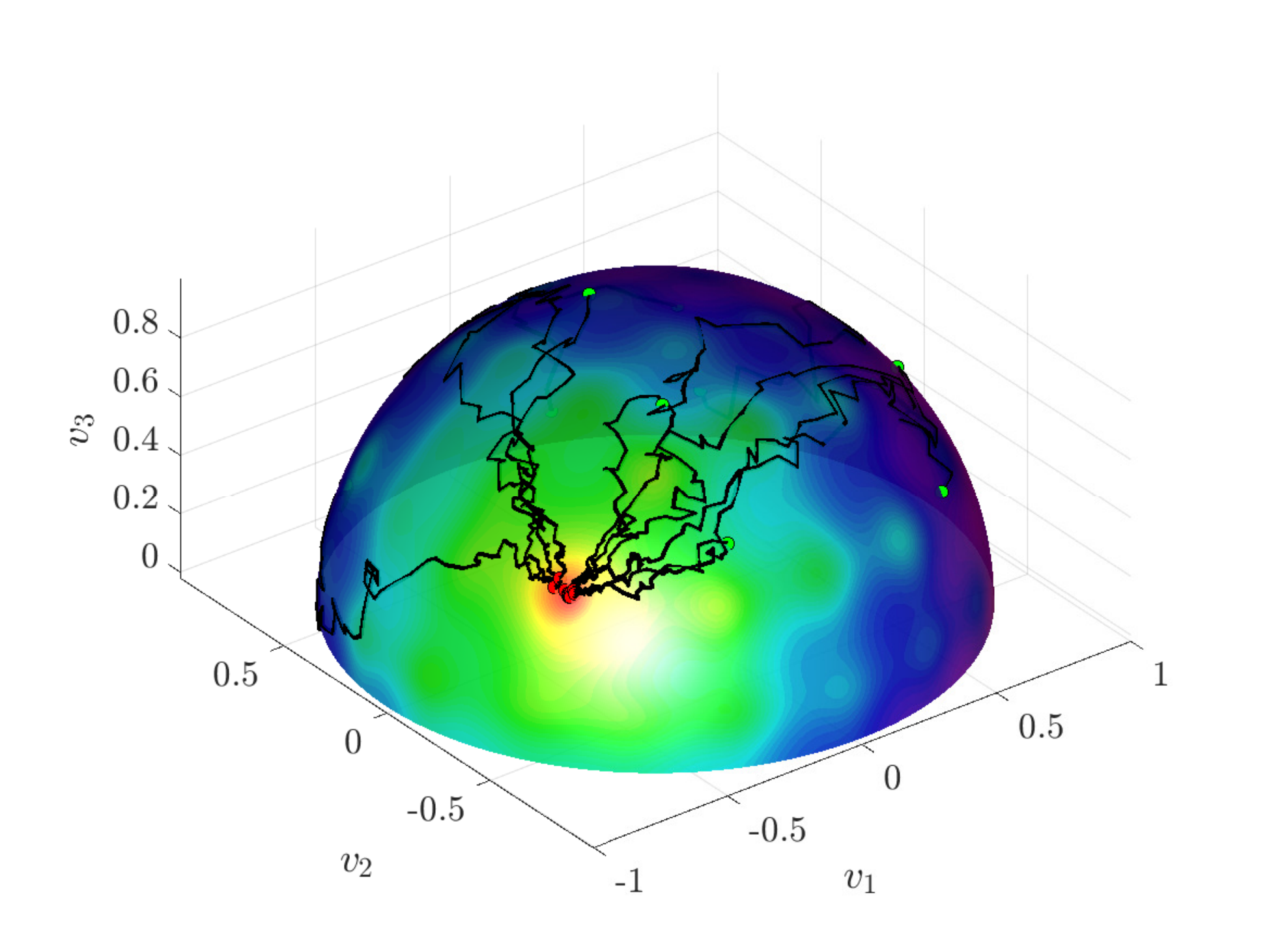}\,
	\caption{Particles trajectories along the simulation for the Ackley function in the case $d=3$, $N=20$ with minimum at $v^*=(0,0,1)^T$ (left) and $v^*=(-1/\sqrt{2},-1/2,1/2)^T$ (right). On the top corresponding time evolution of the particle distribution $\rho(v,t)$ in angular coordinates at $t=1$ and $t=2.5$ for $N=10^6$. The simulation parameters are $\Delta t=0.05$, $\sigma=0.25$ and $\alpha=50$.}
	\label{fg:ackley_traj}
\end{figure}

In all our simulations we initialize the particles with a uniform distribution over the half sphere characterized by $V_3\geq 0$ and employ the simple Euler-Maruyama scheme with projection. 
We report in Figure \ref{fg:ackley_traj} the particle trajectories for $t\in [0,5]$ in the case of $N=20$, $\Delta t=0.05$, $\sigma=0.25$ and $\alpha=50$. On the left we consider the case with minimum at $v^*=(0,0,1)^T$, on the right the case with minimum at $v^*=(1/\sqrt{2},-1/2,1/2)^T$. The time evolution of the particle distribution $\rho(v,t)$ in the numerical mean field limit for $N=10^6$ is also reported in the upper part of the same figure. 

Next in Figure \ref{fg:ackley_conv}, we consider the convergence to consensus measured using various indicators for $N=50$, $\Delta t=0.1$, $v^*=(0,0,1)^T$ and various values of $\sigma$ and $\alpha$. The results have been averaged $1000$ times with a success rate of $100\%$ in all test cases considered.
Following \cite{carrillo2019consensus, pinnau2017consensus}, we consider a run successful if $V_n^{\alpha,\EE}$ at the final time is such that 
\[
\|V_{n_T}^{\alpha,\EE}-v^*\|_\infty:=\sup_{k=1,\dots,d} |(V_{n_T}^{\alpha,\EE})_k-(v^*)_k| \leq 1/4.
\]
We also compute the expected error in the computation of the minimum by considering time averages of $\|V^{\alpha,\EE}-v^*\|_\infty$ and we report the quantity $|V^{\alpha,\EE}-v^*|^2/d$ used in \cite{carrillo2019consensus, pinnau2017consensus}.
As can be seen from Figure \ref{fg:ackley_conv} (top) where $\sigma=0.7$ the influence of large values $\alpha$ in the accuracy of the computation of the minimum is clear when passing from $\alpha=5$ to $\alpha=500$.

\begin{figure}[htbp]
	\includegraphics[scale=0.45]{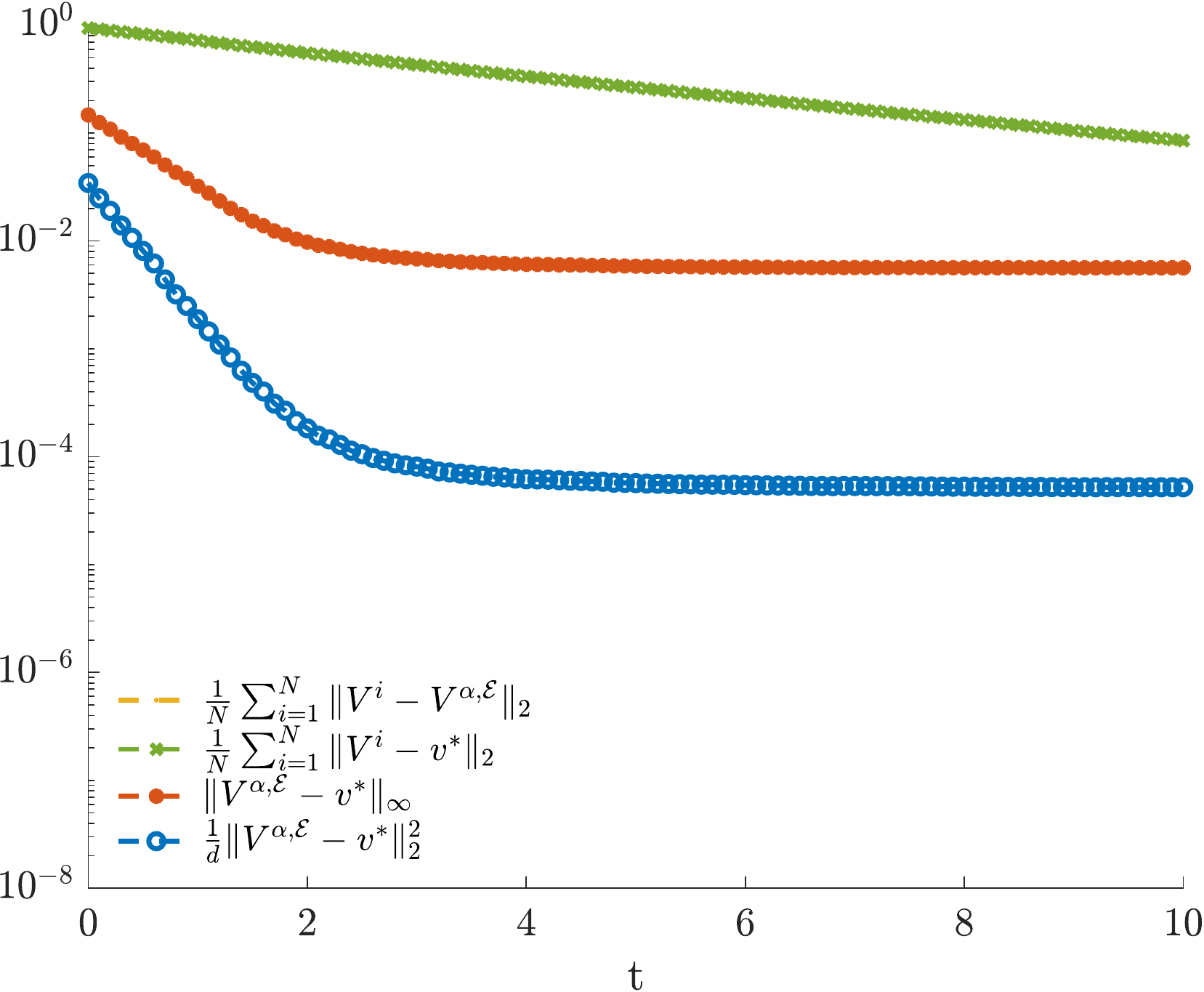}\,
	\includegraphics[scale=0.45]{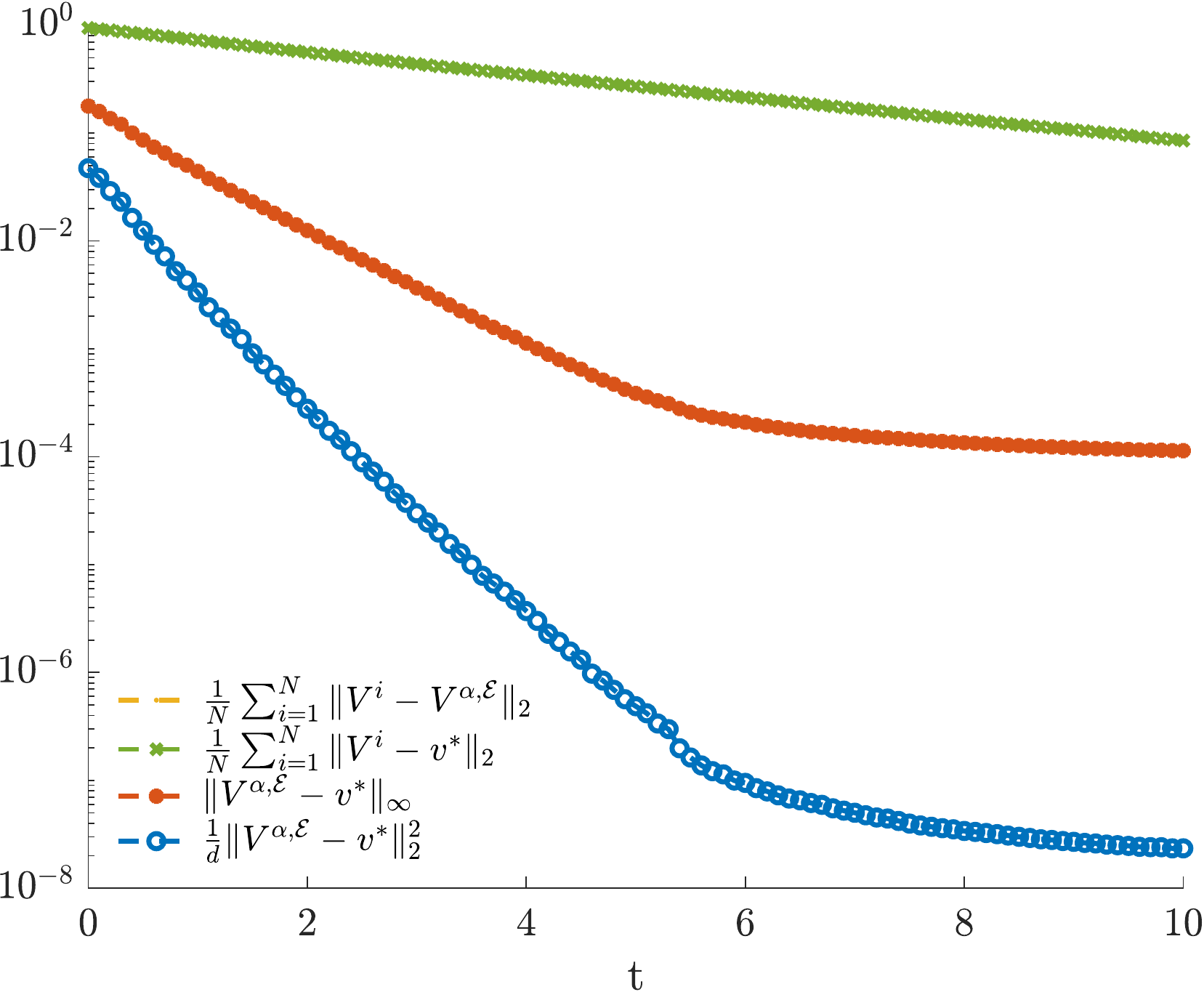}\\
	\includegraphics[scale=0.45]{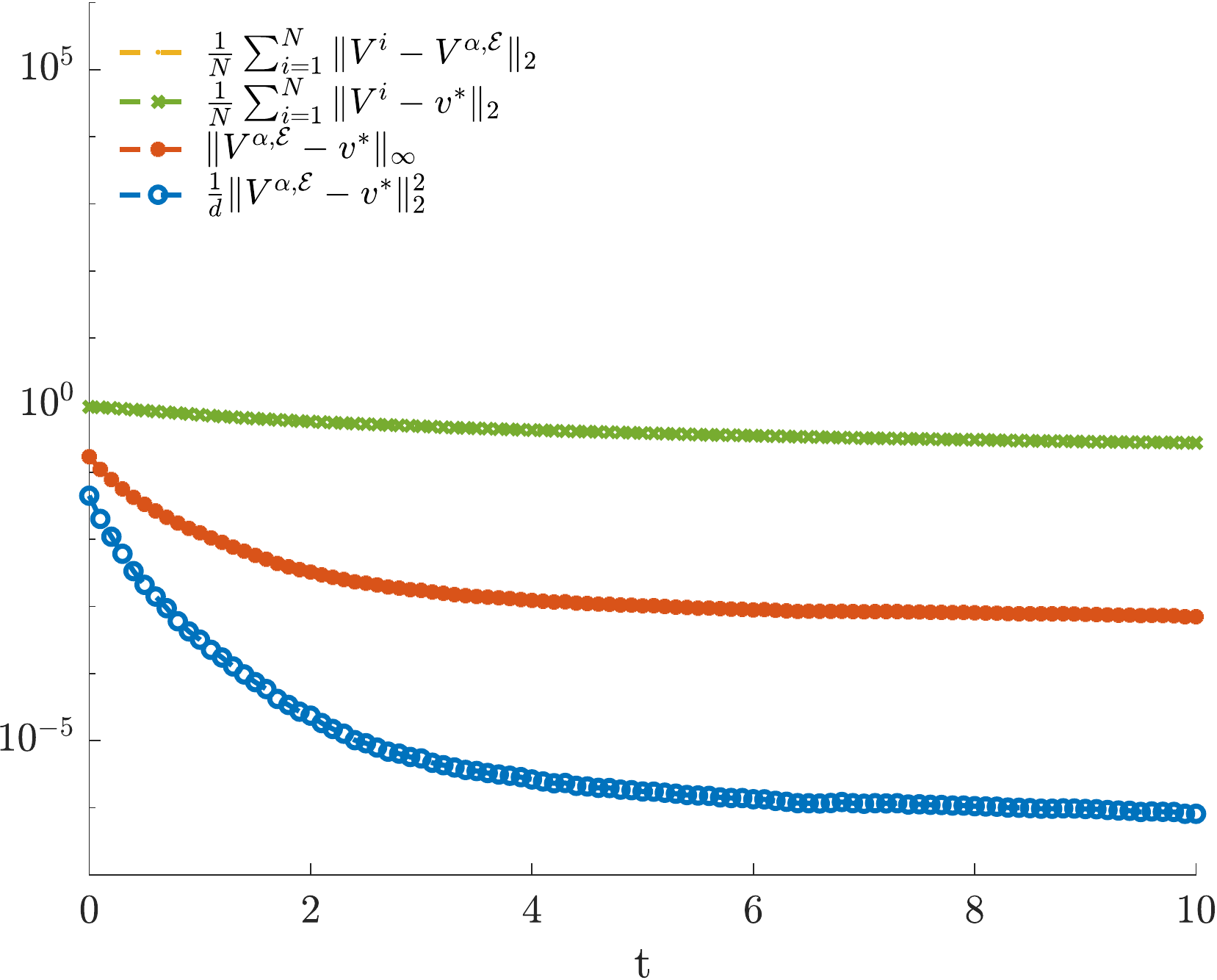}\,
	\includegraphics[scale=0.45]{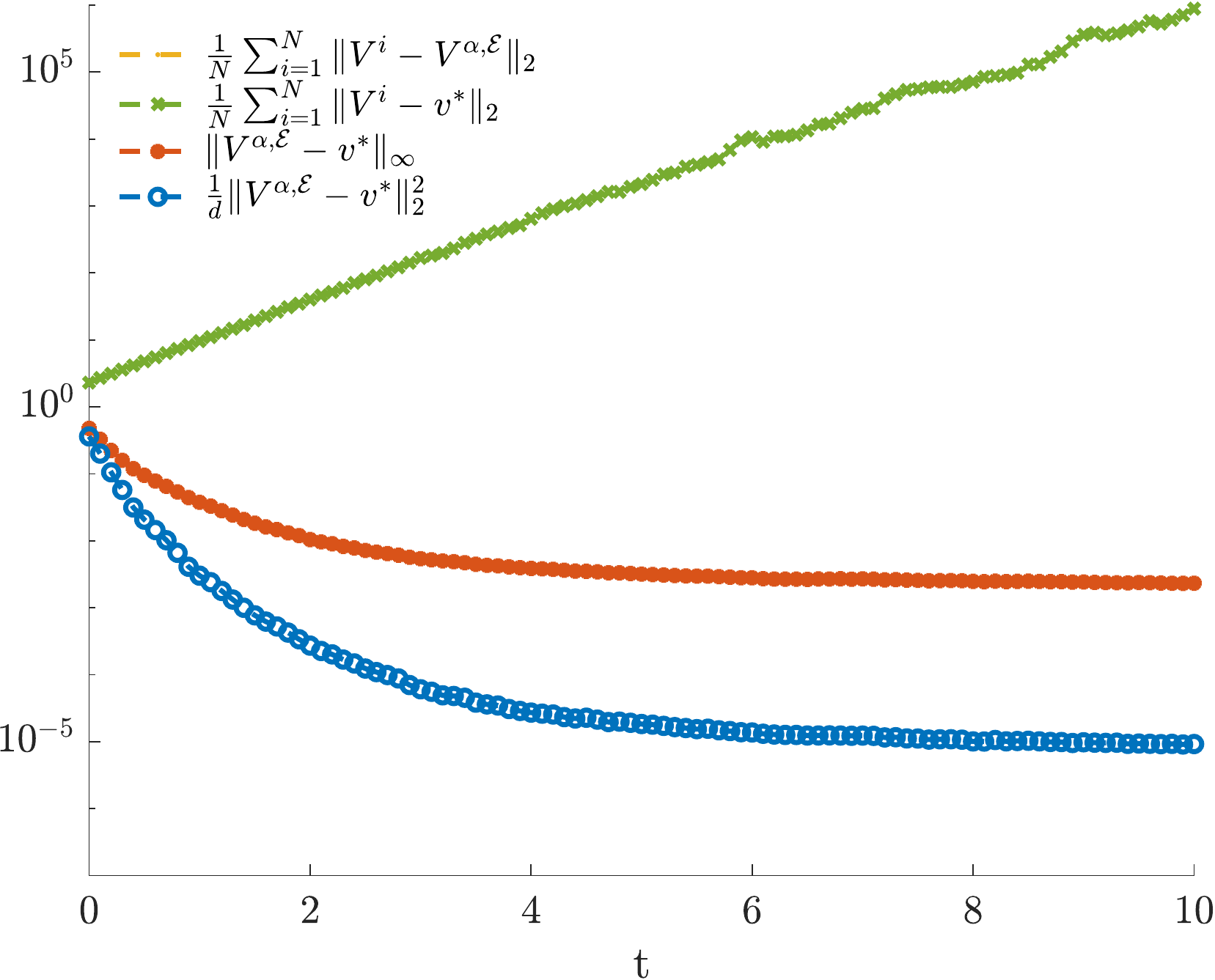}
	\caption{Behavior of various convergence indicators in time for the Ackley function in the case $d=3$ for $N=50$, $\Delta t=0.1$. The two graphs on top show the accuracy of KV-CBO for $\sigma=0.7$, which satisfies $\sigma^2(d-1) \ll \lambda=1$ as dictated by \eqref{lamsig}. We chose $\alpha=5$ (top, left) and $\alpha=500$ (top, right); the yellow line on the top right plot is superimposed by the green line. It is seen that the accuracy is much better for the choice $\alpha = 500$. For the two graphs on the bottom we chose $\sigma=2$, which violates $\sigma^2(d-1) \ll \lambda$, and used KV-CBO with $\alpha=30$ (bottom, left) and the CBO method from \cite{pinnau2017consensus} with $\alpha=30$ (bottom, right). Again, the green line is superimposing the yellow line. The results have been averaged $1000$ times with a success rate of $100\%$ in all cases.}
	\label{fg:ackley_conv}
\end{figure}
In Figure \ref{fg:ackley_conv} (bottom) we show the same computations for a larger value $\sigma=2$ of the diffusion coefficients,  which violate the consensus bound $\sigma^2(d-1) \ll \lambda$, see \eqref{lamsig}. We compare our results with the ones computed using the CBO method in \cite{pinnau2017consensus}.
Even if both methods yield a success rate of $100\%$, the methods clearly do not reach consensus, in the sense that the consensus error \eqref{eq:cons} is not diminishing in time. This behavior is common also to the CBO solvers in \cite{carrillo2019consensus} where the above quantity may even diverge since it is not bounded by the geometry of the sphere. 

\subsubsection*{Minimizing the Ackley function in dimension $d=20$}

Next we consider the more difficult case of the Ackley function in dimension $d=20$. 

In Table \ref{tb:1} we report the results for $\sigma=0.3$, $\Delta t=0.05$, $\alpha=5\times 10^4$, $T=100$ and various values of $N$ and $M$. The rate of success and the expectation of the error $|V^{\alpha,\EE}-v^*|^2$ have been measured over $100$ runs and the minimum has been considered in two different positions 
\[
v^*=(0,\ldots,0,1)^T,\qquad v^*=(d^{-1/2},\ldots,d^{-1/2})^T.
\]
In the first case the minimum is at the center of our initial distribution (so $V^{\alpha,\EE}$ initially is not too far from $v^*$) whereas the second choice is more difficult for the CBO solver since the minimum is shifted with respect to the center of the initial particle distribution, uniformly in all coordinates. 

In all test cases considered the success rate is close to $100\%$. In particular, let us observe (see Table \ref{tb:2}) that the fast method for $\mu=0.3$ and $\mu=0.2$ with $N_{min}=10$ permits to achieve better performances for a given computational cost. We have selected a final computation time lower than the optimal computation time that would have allowed us to achieve maximum precision in the computation of the minimum, this to avoid unnecessary iterations with a small number of particles that would have created a bias in the final average particle number $N_{avg}$.

\begin{table}[htb]
	\caption{Ackley function in $d=20$: $\mu=0$, $\sigma=0.3$, $\Delta t=0.05$ and $T=100$}
	\begin{center}
		\begin{tabular}{c|ccccc}
			$v^*$ & & $N=50$ & $N=100$ & $N=200$\\
			& & $M=40$ & $M=70$  & $M=100$\\
			\hline
			\hline
			& & & &\\[-.25cm]
			$(0,\ldots,0,1)^T$ & Rate & $100\%$ & $100\%$ & $100\%$\\
			& Error & $2.24118e-08$ & $1.3364e-09$ & $3.51083e-09$ \\
			\hline
			& & & &\\[-.25cm]
			$(d^{-1/2},\ldots,d^{-1/2})^T$ & Rate & $98\%$ & $99\%$ & $100\%$\\
			& Error & $1.15704e-06$ & $1.476e-09$ & $5.09216e-09$ \\
			\hline
			\hline
		\end{tabular}
	\end{center}
	\label{tb:1}
\end{table}%

\begin{table}[htb]
	\caption{Ackley function in $d=20$: $\sigma=0.3$, $\Delta t=0.05$ and $T=100$}
	\begin{center}
		\begin{tabular}{c|ccccc}
			$v^*$ & & $N_0=100$ & $N_0=200$ & $N_0=400$\\
			& & $M=70$ & $M=100$  & $M=150$\\
			\hline
			\hline
			& & & &\\[-.25cm]
			$(0,\ldots,0,1)^T$ & Rate & $100\%$ & $100\%$ & $100\%$\\
			$\mu=0.3$        & Error & $1.20639e-07$ & $3.73419e-08$ & $2.24362e-08$ \\
			& $N_{avg}$ & $21.6$ & $38.7$ & $71.4$ \\        
			\hline
			& & & &\\[-.25cm]
			$(d^{-1/2},\ldots,d^{-1/2})^T$  & Rate & $100\%$ & $100\%$ & $100\%$\\
			$\mu=0.2$       & Error & $1.34745e-06$ & $2.02787e-08$ & $8.06536e-09$ \\
			& $N_{avg}$ & $27.3$ & $53.1$ & $103.0$ \\ 
			\hline
			\hline
		\end{tabular}
	\end{center}
	\label{tb:2}
\end{table}%

\subsection{Challenging applications in signal processing and machine learning}

In this section we consider two  applications of KV-CBO, namely, the phase retrieval problem and the robust subspace detection problem.  For the former we consider only synthetic data, for the latter we consider synthetic as well as real-life data in dimension up to $d=2880$. 
The solution to these problems can be reformulated in terms of a high dimensional nonconvex optimization over the sphere with unique symmetric solutions.
Both these problems have by now {\it ad hoc} state of the art methods for their solution. The aim of this section is to demonstrate that Algorithms \ref{algo:sKV-CBO} or \ref{algo:sKV-CBOfc} can be used
in a versatile and scalable way to solve several and diverse problems and achieve state of the art performances by comparison with the more specific methods.

\subsubsection{Phase Retrieval}\label{sec:phaseretr}

Recently there has been growing interest in recovering an input vector $z^*\in \mathbb R^d$ from quadratic measurements 
\begin{equation}
y_i = |\langle z^*, a_i\rangle |^2+ w_i, \quad i=1,...,M \label{def yi}
\end{equation}
where $w_i$ is adversarial noise, and $a_i$ are a set of known vectors. Since only the magnitude of $\langle z^*,a_i\rangle$ is measured, and not the phase (or the sign, in the case of real valued vectors), this problem is referred to as phase retrieval. Phase retrieval problems arise in many areas of optics, where the detector can only measure the magnitude of the received optical wave.  Important applications of phase retrieval include X-ray crystallography, transmission electron microscopy and coherent diffractive imaging \cite{qui10,hurt2001phase,Harrison:93,wa63}. Several algorithms have been devised for robustly computing 
$z^*$ from measured information $y=(y_i)_{i=1,\dots,M}$ based on different principles, such as alternating projections, lifting and convex relaxation, and simple gradient descent for empirical risk minimization \cite{Gerchberg72,fien82,Yang:94,ca13,ca14,Chen_2019}. Despite the wide range of solutions, most of these algorithms fail to tackle robustly the crystallographic problem which is both the leading application and one of the hardest forms of phase retrieval \cite{els18}. One of the reasons is that the phase retrieval problem is intrinsically ill-posed for $M$ small. Recent work  \cite{mondelli2017fundamental} explains even by information theoretical arguments that no estimator can do better than a random estimator for $M \leq d - o(d)$. Uniqueness results of the solution $z^*$ of the real-valued phase retrieval problem in the case of no noise has been established in \cite{BALAN2006345} for sets of measurement vector $\{a_i: i=1, \dots,M \}$ forming a frame for $\mathbb{R}^d$, i.e., there are constants $0<A \leq B < \infty$ such that
\begin{equation}
A|z|^2 \leq \sum_{i=1}^M |\langle z,a_i \rangle|^2 \leq B |z|^2
\end{equation}
holds for any $z \in \mathbb{R}^d$. Specifically, \cite[Theorem 2.2]{BALAN2006345} ensures that for generic frames unique identifiability occurs for $M\geq 2d -1$, as the map $\mathbb R^d \backslash \{\pm 1\} \ni z\to y(z):=(|\langle z, a_i\rangle |^2)_{i=1,\dots,M}$ is in fact injective.
In order to tackle the robust identifiability, empirical risk minimization has been considered in \cite{ELDAR2014473}, i.e., the minimization of the discrepancy
\begin{equation}\label{empiricalrisk}
\EE(z) = \sum_{i=1}^M  \left | |\langle z, a_i\rangle|^2 - y_i\right |^2. 
\end{equation}
Guarantees of stable reconstruction via empirical risk minimization are obtained under the assumption that the measurements vectors $\{a_i: i=1, \dots,M \}$ fulfill the stability property
\begin{equation}\label{stabilityprop}
\sum_{i=1}^M \left | |\langle z, a_i\rangle|^2 - |\langle \hat z, a_i\rangle|^2\right | \geq \kappa |z-\hat z||z+ \hat z|,
\end{equation}
for all $z, \hat z \in \RR^d$ and some fixed $\kappa>0$. In particular, \cite[Theorem  2.4]{ELDAR2014473} ensures that for  measurement
vectors $\{a_i: i=1, \dots,M \}$ generated at random, e.g.,  as i.i.d. Gaussian vectors, for $M \geq \gamma d$ , the stability estimate \eqref{stabilityprop} holds for a suitable $\kappa>0$ with high probability depending on the constant $\gamma>0$. As a broad disquisition about the phase retrieval problem is not the focus of this paper, we omit here  details about stability under adversarial noise and we refer to \cite{BANDEIRA2014106,ELDAR2014473} for further insights.
However, we should notice at this point that the empirical risk  $\EE$ in \eqref{empiricalrisk} fulfills then all the requests of Assumptions \ref{assumas} below, in particular the stability estimate \eqref{stabilityprop} naturally induces the {\it inverse continuity property} 4. of Assumptions \ref{assumas}. Hence, the minimization of \eqref{empiricalrisk}  is a challenging nonconvex optimization problem, which falls precisely in the realm of problems for which Algorithm \ref{algo:sKV-CBO} or Algorithm \ref{algo:sKV-CBOfc} are expected to work at best. Before presenting numerical experiments of the use of Algorithm \ref{algo:sKV-CBO} or Algorithm \ref{algo:sKV-CBOfc} and comparisons with state of the art methods, we should perhaps clarify that the empirical risk minimization can without loss of generality be  restricted to vectors on the sphere  as soon as the lower frame constant $A$ is known: for the sake of simplicity, let us assume again that the noise $w\equiv 0$ and we observe that 
\begin{equation}
\sum_{i=1}^M y_i = \sum_{i=1}^M |\langle z^*,a_i \rangle|^2 \geq A|z^*|^2 \quad \text{and} \quad |z^*| \leq \sqrt{\frac{1}{A}\sum_{i=1}^M y_i} =: R
\end{equation}
where we take $A$ to be the optimal lower frame bound. We define the vectors $\tilde{a}_{i}$ by one zero padding, i.e.,  
\begin{equation}
\tilde{a}_i = [a_i, 0] \in \mathbb{R}^{d+1},
\end{equation}
and we further denote
\begin{align}
\tilde z^* = [z^*, \sqrt{R^2 - |z^*|^2}] \in R~\BS^{d}, \quad  v^*= \frac{\tilde{z}^*}{R} \in \BS^{d}, \quad \mbox{and} \quad  \tilde y_i = \frac{y_i}{R^2}. \label{zbar}
\end{align}

\begin{figure}[tb]
	\centering
	\includegraphics[width = 6cm]{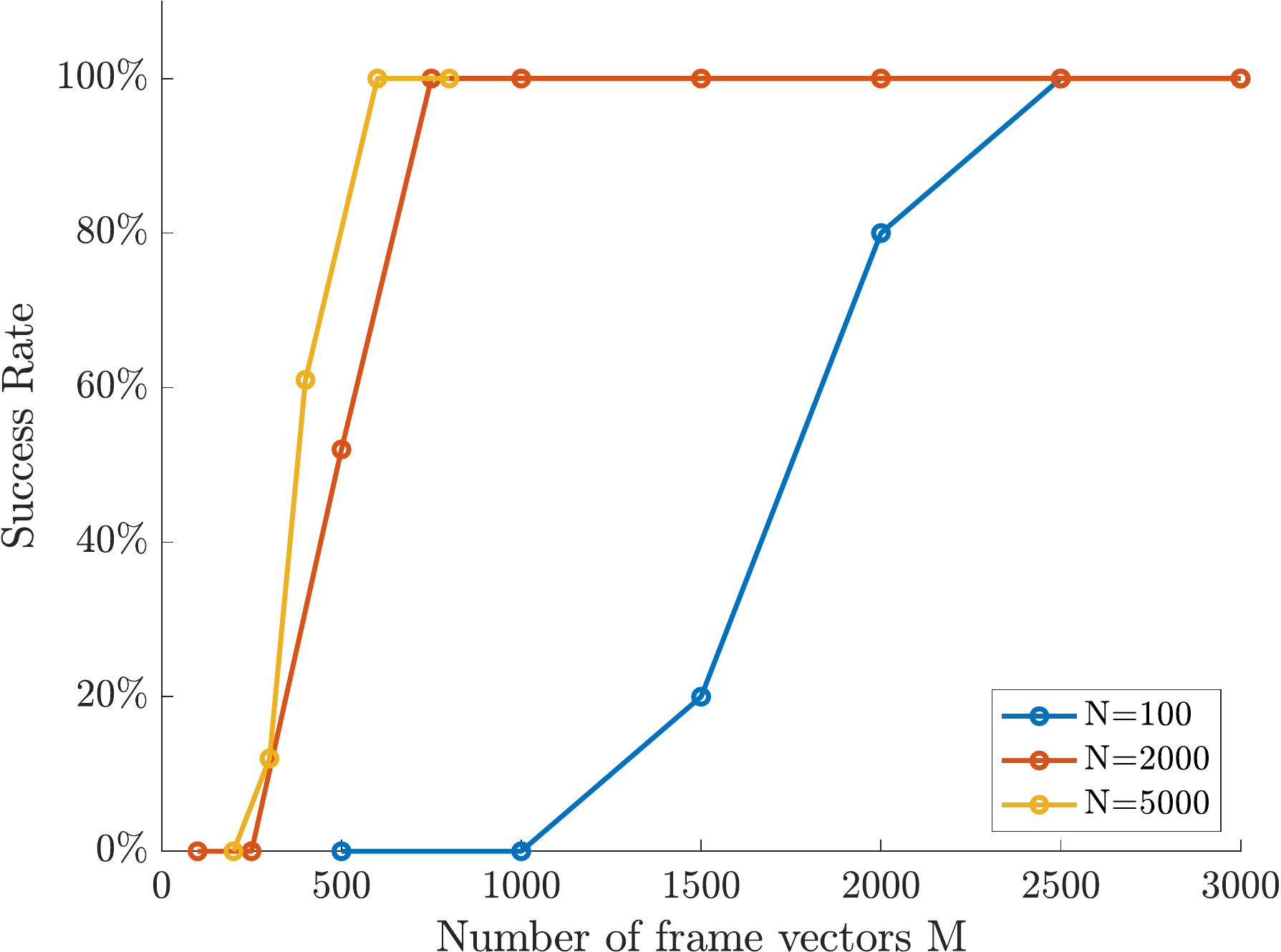}
	\includegraphics[width = 7.5cm]{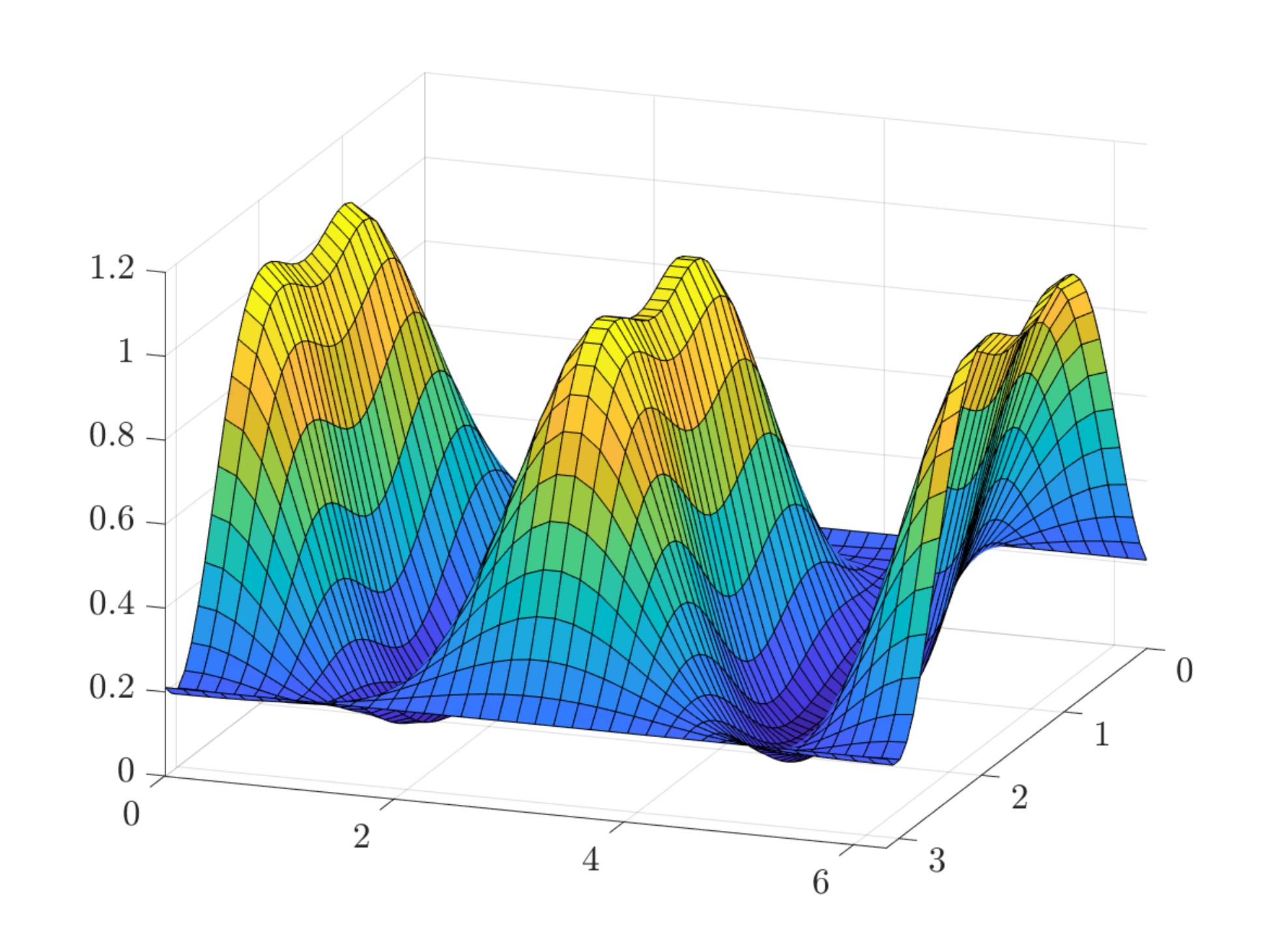}
	\caption{Left: Success rate for different numbers of frame vectors $M$ in dimension $d=100$. We have used the following parameters: $\alpha =2\cdot 10^{15}, \lambda=1, \Delta t = 0.1$ and $\sigma =0.11$ for $N=100$ and $\sigma = 0.08$ for $N=2000$. Right: Plot of the energy defined in \eqref{empiricalrisk} in  $d=2$. It is evident that the energy may exhibit saddle points, but no spurious minimizers appear. This is the reason for a vanilla gradient descent method to work so well for such a problem \cite{ca14,Chen_2019,recht19}. \label{phase retrieval plots}}
\end{figure}

With these notations, \eqref{def yi} can be equivalently recast in the form 
$$
\tilde y_i=|\langle v^*, \tilde{a}_i \rangle |^2, \quad v^* \in \BS^{d}.
$$ 
Hence, the unconstrained minimization of $\EE$ can be equivalently solved by the constrained minimization of
\begin{equation}
\widetilde{\EE}(v) := \sum_{i=1}^M \left | |\langle v,\tilde a_i\rangle|^2 - \tilde y_i \right |^2, \label{empiricalrisk2}
\end{equation}
over the sphere $\mathbb S^{d}$. In fact, the first $d$ components of the minimizing vector $v^*$ must coincide with $z^*/R$. So from now on we implicitly assume that the problem is transformed into one of the type \eqref{typrob}. 

We tested KV-CBO for dimension $d=100$ for the function defined in \eqref{empiricalrisk2}, where the vectors $a_i$ are sampled from a uniform distribution over the sphere. We computed the success rate for reconstructing the vector $z^*$ in terms of the number $M$ of vectors $a_i$. We count a run as successfull if the computed $\bar z$ by Algorithm \ref{algo:sKV-CBO} or Algorithm \ref{algo:sKV-CBOfc} fulfills
\begin{equation}
\min\{|z^* - \bar z|, |z^* +\bar z|\}<0.05\,.
\end{equation}
The phase transitions of success recovery are shown in on the left-hand-side of Figure \ref{phase retrieval plots}. 
We can observe that the success rate improves with the number $N$ of particles used by Algorithm \ref{algo:sKV-CBO} or Algorithm \ref{algo:sKV-CBOfc} and best success is obtained by $M\geq \gamma d$ as predicted by theory.  We notice that the optimization via KV-CBO is evidently not affected by the curse of dimension. On the right-hand-side we depict the typical cost function landscape with saddle-points and symmetric global minimizers. 

In the following, we compare Algorithm \ref{algo:sKV-CBOfc} with three relevant state of the art methods for phase retrieval:
\begin{itemize}
	\item Wirtinger Flow (fast gradient descent method) \cite{ca14,Chen_2019};
	\item Hybrid Input Output/Gerchberg-Saxton's Alternating Projections (alternating projection methods) \cite{Gerchberg72,fien82,Yang:94};
	\item PhaseMax/PhaseLamp (convex relaxation and its multiple iteration version) \cite{ca13}.
\end{itemize}
For the comparsion we used the Matlab toolbox PhasePack\footnote{\it https://www.cs.umd.edu/$\sim$tomg/projects/phasepack/} \cite{chandra2017phasepack} and our own code\footnote{{\it 
		https://github.com/PhilippeSu/KV-CBO}}.

\begin{figure}[tb]
	\centering
	\includegraphics[width = 7cm]{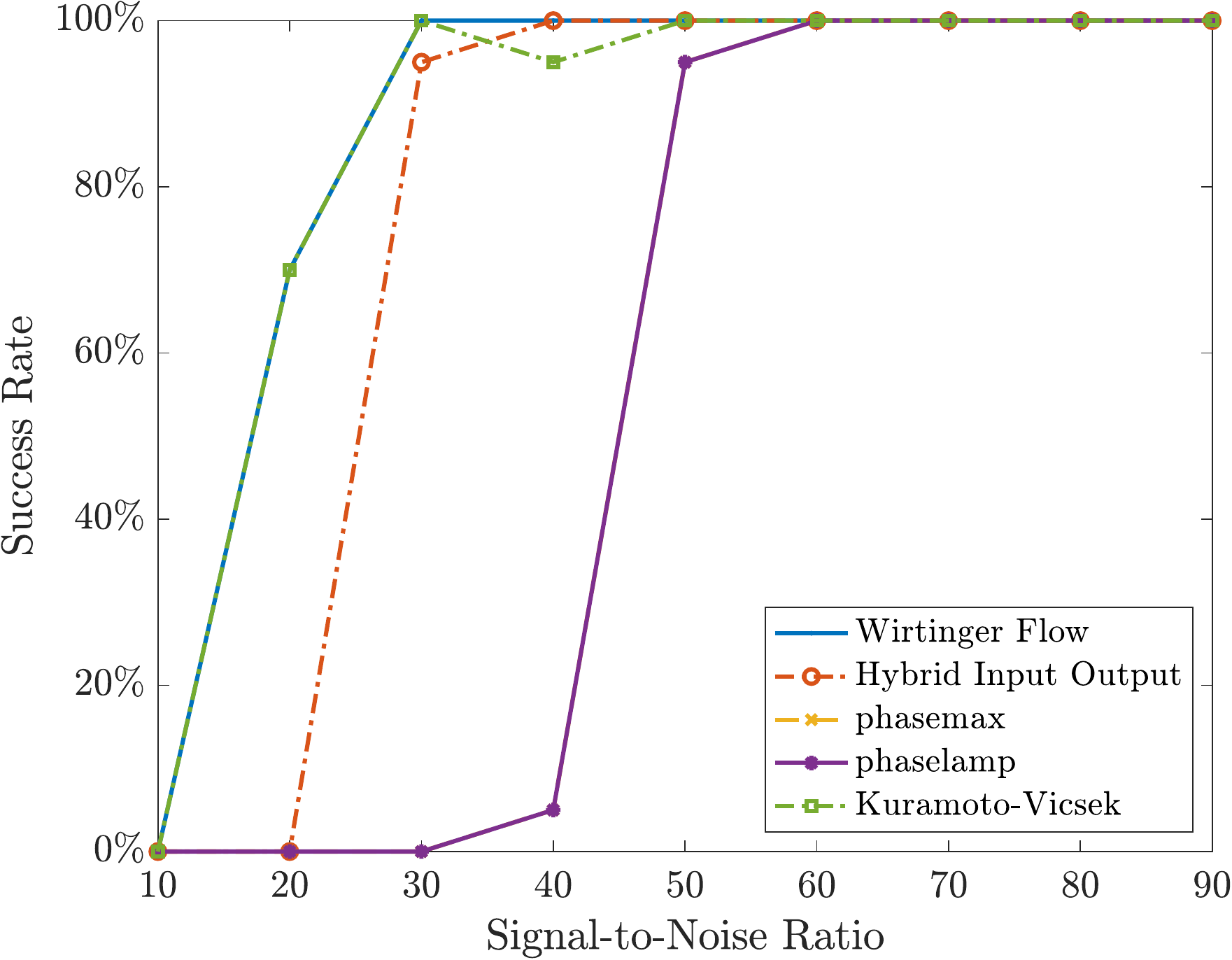}
	\includegraphics[width = 7cm]{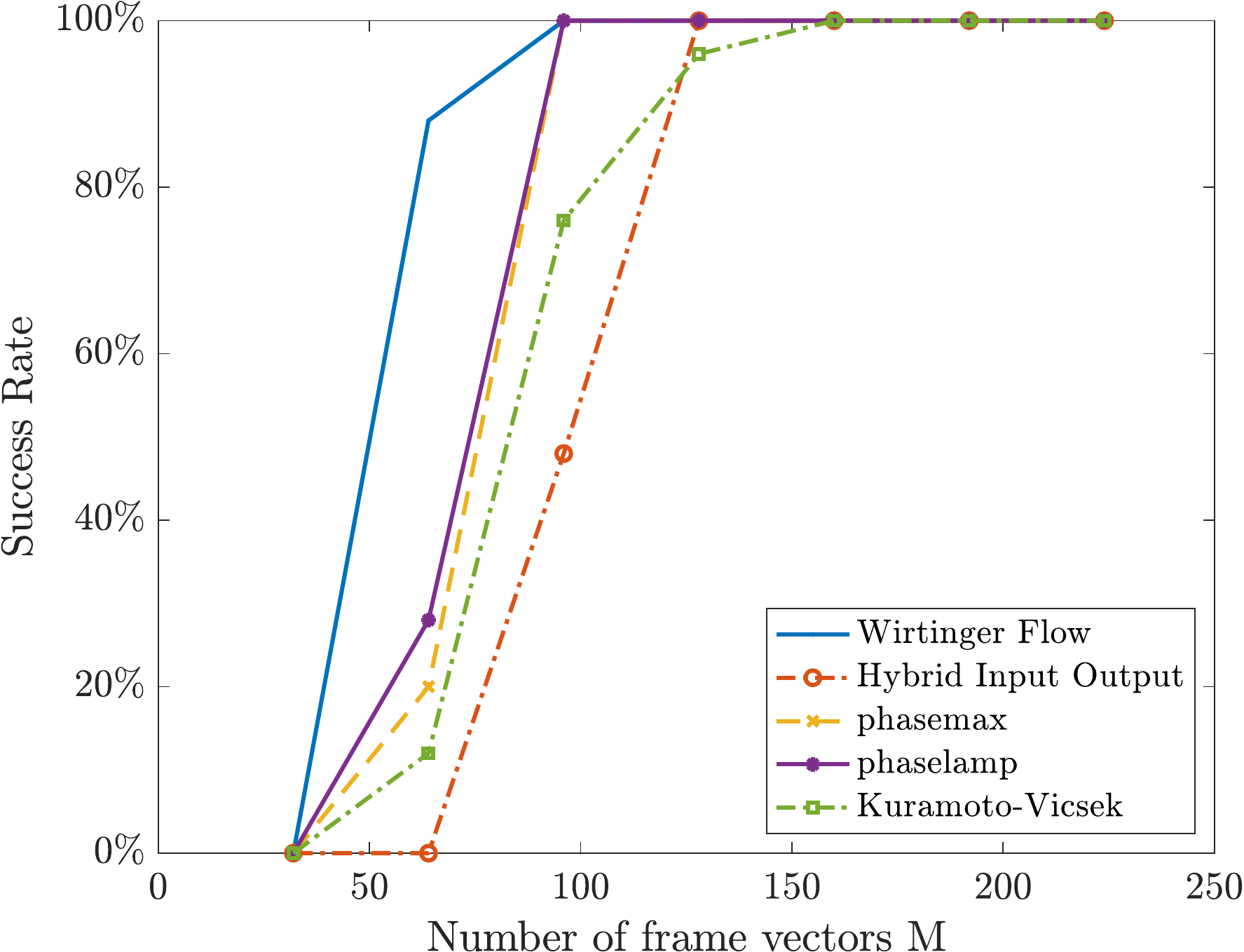}
	\caption{Left: Success rate in terms of the Signal-to-Noise Ratio in dimension $d=32$ for $M=4d$ Gaussian vectors. The green dashed curve representing KV-CBO is exactly superimposed with the light blue curve of the Wirtinger Flow. Right: Phase transitions for different numbers of Gaussian vectors $M$ in dimension $d=32$ (the yellow curve is superimposed by the purple curve). We used $\sigma = 0.2, \Delta t = 0.1, N=10^4$ and chose the parameter $\alpha$ adaptively, with initial $\alpha_{0} = 2000$ and final  $\alpha_{max} = 1e15$. The results are averaged 25 times.}\label{fig:compPR}
\end{figure}

In Figure \ref{fig:compPR} we demonstrate on the left that KV-CBO is exactly as robust as Wirtinger Flow with respect to adversarial noise and on the right we compare phase transition diagrams of success rate, which show that KV-CBO has a slight delay in perfect recovery with respect to  Wirtinger Flow and PhaseMax/PhaseLamp, but it is comparable with Hybrid Input Output/Gerchberg-Saxton's Alternating Projections. The delayed perfect recovery indirectly confirms that the  {\it inverse continuity property} 4. of Assumptions \ref{assumas} needs to be fulfilled for the method to work optimally. (We reiterate that if $M$ is large enough, then the stability property \eqref{stabilityprop} holds with high probability and as a consequence also the inverse continuity property.)

\subsubsection{Robust Subspace Detection}\label{sec:robsub}

Let us consider a cloud of points $\mathcal Q= \{x^{(i)} \in \RR^d: i=1,...,M \}$ in an Euclidean space with $d \gg 1$. We assume without loss of generality that the point cloud is centered, that is, the mean of the point cloud is zero. Subspace detection is about finding a lower dimensional linear subspace $V \subset \RR^d$ that fits the data at best, in the sense that the sum of the squared norms of the orthogonal projection of the points $x^{(i)}$ to $V^\perp$ is minimal.  In the simplest case of a one-dimensional subspace, the cost function to be minimized is given by $\EE_2(v):=\sum_{i=1}^M | x^{(i)} - \langle x^{(i)},v \rangle v|^2$ where each summand is the squared norm of the orthogonal projection of one point $x^{(i)}$ to the space $\operatorname{span}(v)^\perp$.  It is well-know that the minimizer $v^*=\arg \min_{v \in \BS^{d-1}} \EE_2(v) = \arg\max_{v \in \BS^{d-1}}  | X v|^2 $ represents the direction of maximal variance of the point cloud, see, e.g., Figure \ref{point cloud and energy 1} (left), and coincides with the right singular vector associated to the operator norm of the matrix $X=({x^{(i)}}^T)_{i=1, \dots M}$ whose rows are the vectors $x^{(i)}$'s. Despite the nonconvexity of the cost, the computation of the best fitting subspace can be conveniently done by singular value decomposition (SVD) also for subspaces of higher dimension. In this case the cost would simply read $\EE_2(V):=\sum_{i=1}^M |P_{V^\perp} x^{(i)}|^2$.
The drawback of the energy $\EE_2(v)$ is the fact that it is quadratic, thus the summand $| x^{(i)} - \langle x^{(i)},v \rangle v|^2$ will be particular large if $x^{(i)}$ is an outlier, far from the subspace where most of the other points may cluster. The aim of {\it robust} subspace detection \cite{lerman15,Lerman_2017,lerman19} is finding the principal direction of a point cloud without assigning too much weight to outliers. We therefore introduce the more general energy
\begin{equation}
\EE_p(V):=\sum_{i=1}^M |P_{V^\perp} x^{(i)}|^p, \quad V \subset \RR^d, \quad \operatorname{dim}(V)=k \ll d,
\end{equation}
where $0<p\leq 2$. Even in the simplest one dimensional case, the minimization of the energy 
$$
\EE_p(v):=\sum_{i=1}^M | x^{(i)} - \langle x^{(i)},v \rangle v|^p =\sum_{i=1}^{M} \big( |x^{(i)} |^{2}-|\langle x^{(i)}, v \rangle |^{2}\big)^{p/2}, \quad v \in \BS^{d-1},
$$
turns out for $0<p\ll2$ to be a rather nontrivial nonconvex optimization problem. On the right of Figure \ref{point cloud and energy 2}, Figure \ref{point cloud and energy 3}, and Figure \ref{point cloud and energy 4} we illustrate some cost function landscapes in dimension $d=2$. One can immediately notice how $\EE_p$ becomes in fact rougher and exhibits all of the sudden several spurious local minimizers (compare with the case of $p=2$ in Figure \ref{point cloud and energy 1}). Hence, the success of a simple gradient descent method is far less obvious than for the phase retrieval problem, where the energy may have saddle-points, but it has generically no local minimizers, see Figure \ref{phase retrieval plots} and refer to \cite{ca14,Chen_2019,recht19} for details.
\\

\noindent
In the following we test KV-CBO for  clouds of synthetic data points and a cloud of real-life photos from the \textit{10K US Adult Faces Database} \cite{10kUSAdultFaces}. We discuss the performance of the method both for $p=2$ and $0<p<2$. In the former case, we can compute the exact minimizer of the energy $\EE_2(v)$ by SVD. For $0<p<2$ we compare the result with the state of the art algorithm Fast Median Subspace (FMS) \cite{Lerman_2017} as benchmark. We mention that FMS is proven in general to converge to stationary points of the cost function only, which are in special  data models very close to global minimizers with high probability. The synthetic point cloud models we use for comparison below are in part fitting the existing guarantees of global optimization for FMS.
In these cases, we analyze different sets of parameters and dimensionality of the problem and we discuss the success rate for different parameters such as numbers of particles and $\sigma>0$. In fact, the choice of the parameter $\sigma>0$ is perhaps a bit tricky. From our theoretical findings, it would be sufficient that $\lambda/(d-1) \gg \sigma^2$, see \eqref{lamsig}, thus $\sigma>0$ needs simply to decrease with growing dimension $d$. However, in the pure particle simulation $\sigma$ cannot be taken too small otherwise randomness won't be enough to explore the space in a reasonable computational time. In Figure \ref{success ratio sigma plot} we report the success rate in terms of $\sigma$ for different dimensions. We further chose $\alpha = 2\cdot 10^{15}$ and $\Delta t = 0.25$. 

\begin{figure}[H]
	\centering
	\includegraphics[width = 7cm, height = 5cm]{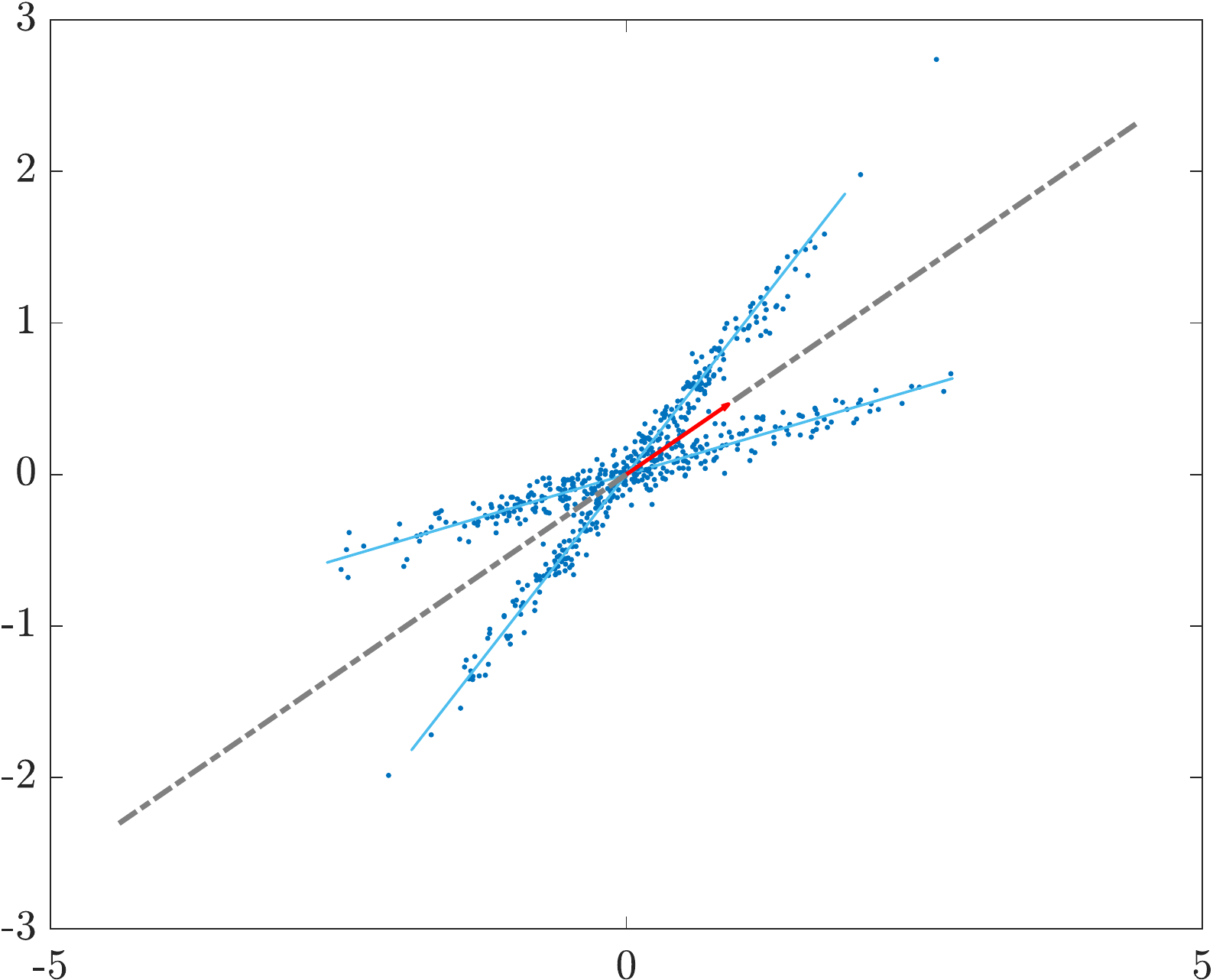}
	\includegraphics[width = 7cm, height = 5cm]{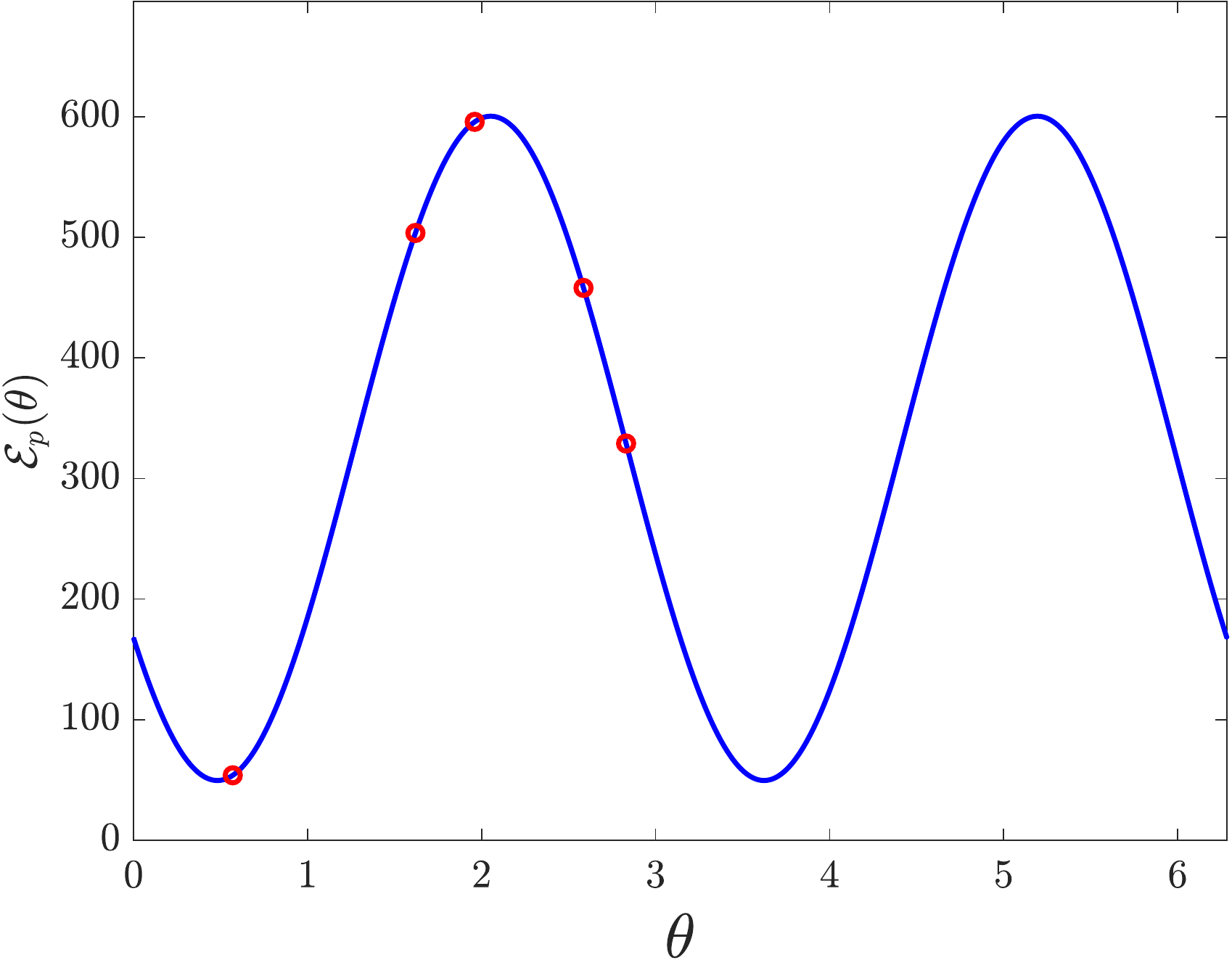}
	\caption{Left: Point cloud with $N_{sp}=2$ one-dimensional subspaces in dimension $d=2$ with Gaussian noise of $0.01$. The red vector shows the principal direction. Right: Energy $\EE_2(v(\theta))$ for the point cloud on the left for $\theta \in [0,2\pi)$, where $v(\theta)$ is expressed in polar coordinates. The particles are shown in red. \label{point cloud and energy 1}}
\end{figure}

\begin{figure}[H]
	\centering
	\includegraphics[width = 7cm, height = 5cm]{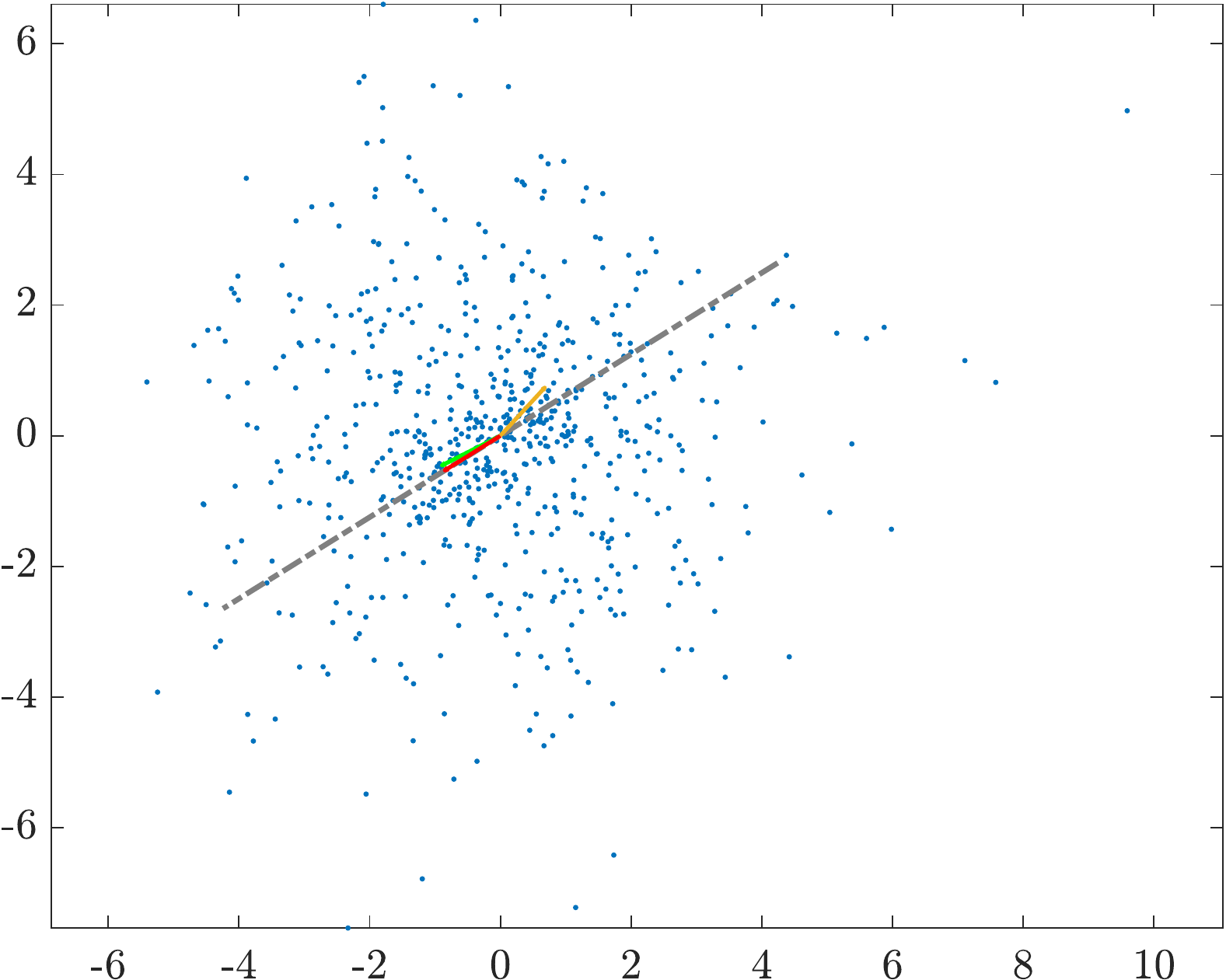}
	\includegraphics[width = 7cm, height = 5cm]{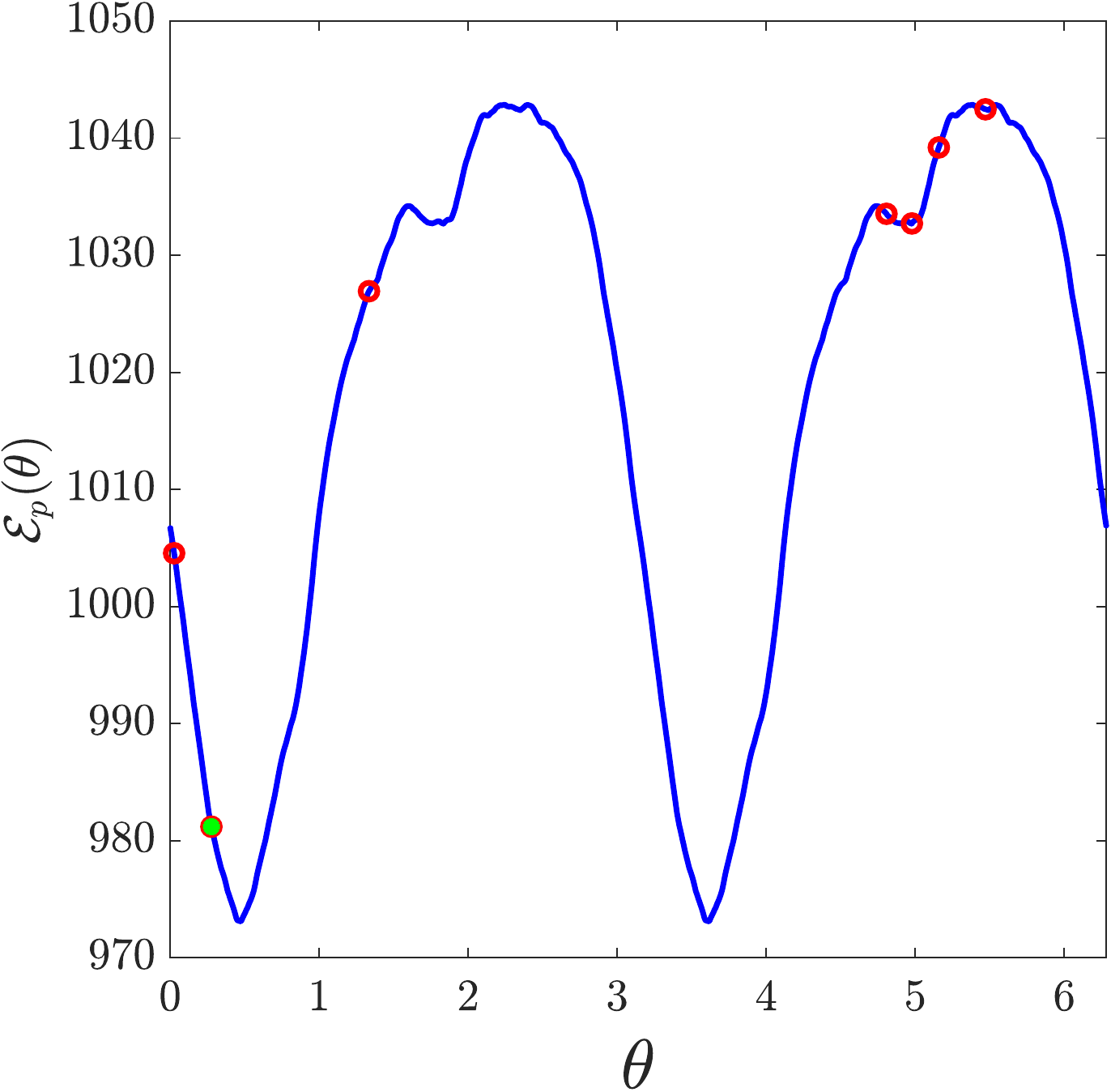}
	\caption{Left: Point cloud in dimension $d=2$ with $N_{sp}=2$ subspaces with 100 points each and Gaussian noise of $0.25$. We further have added $500$ outliers. The orange/ red vector shows the principal component computed by SVD of the point cloud with/ without the outliers. The green vector is the principal component compoted by KV-CBO. Right: Energy $\EE_p(\theta)$ for $\theta \in [0,2\pi)$ for the point cloud on the left and $p=1$. The particles are shown in red, $V_0^{\alpha, \EE}$ is shown in green. \label{point cloud and energy 2}}
\end{figure}

\begin{figure}[H]
	\centering
	\includegraphics[width = 7cm, height = 5cm]{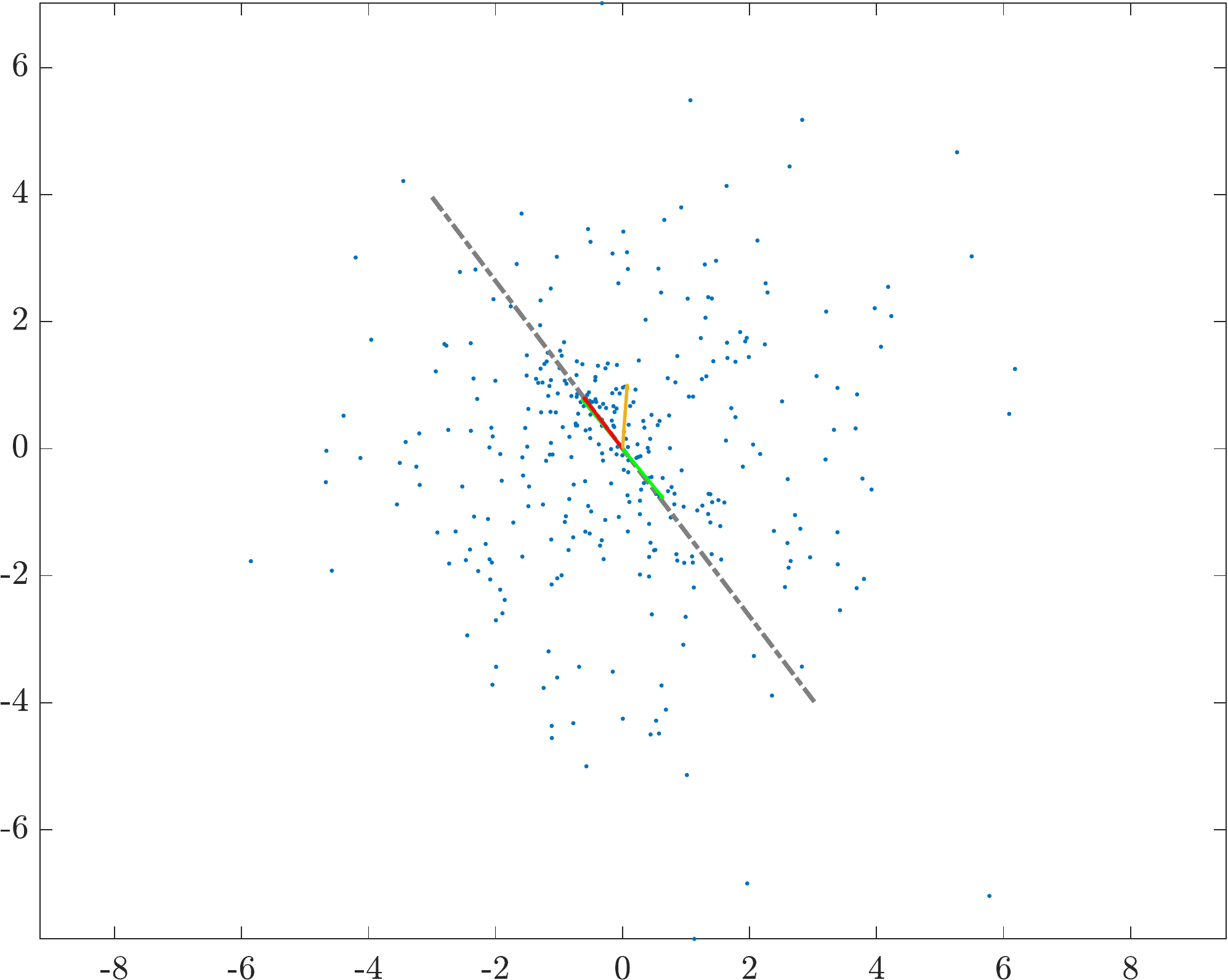}
	\includegraphics[width = 7cm, height = 5cm]{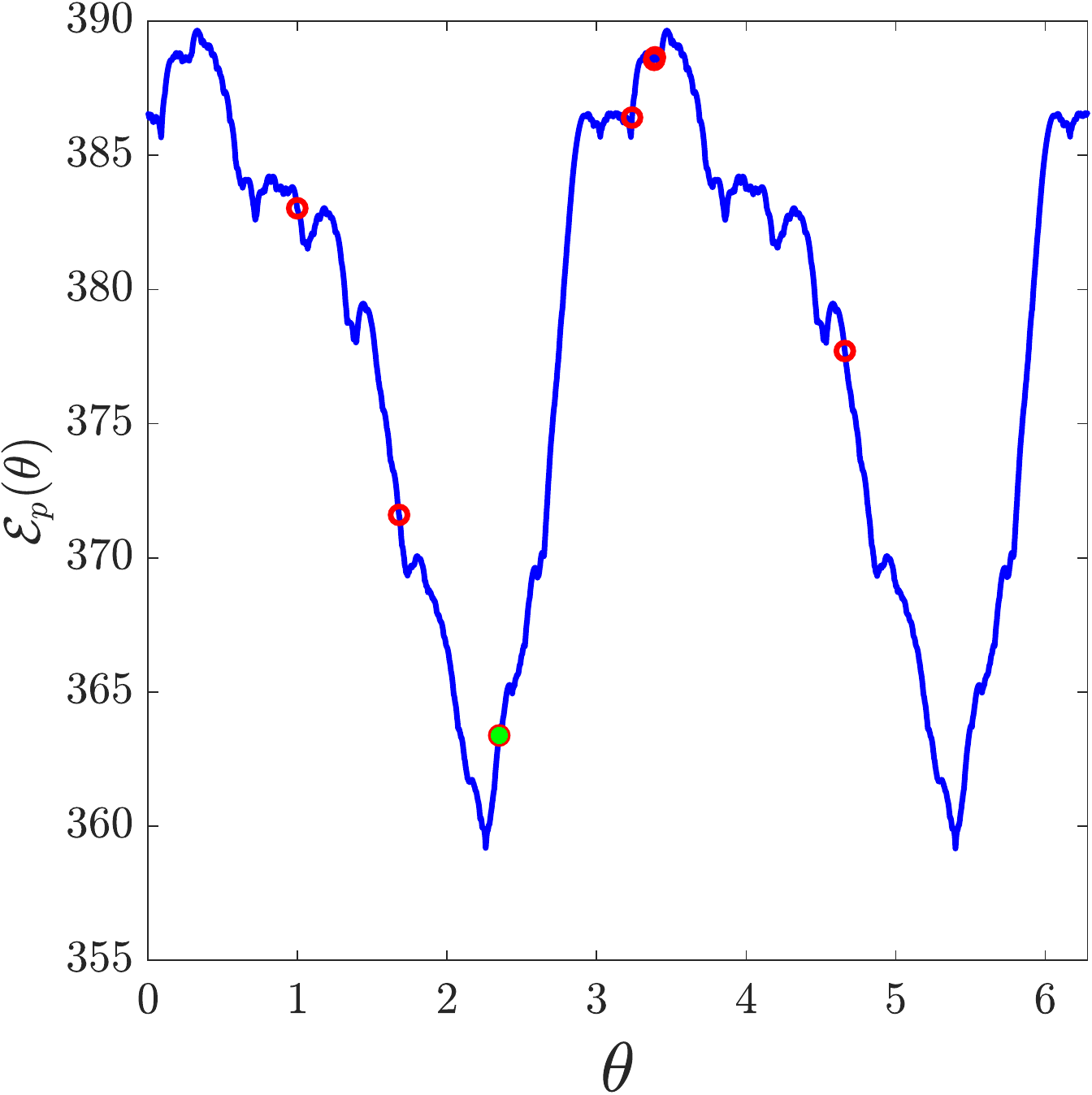}
	\caption{Left: Point cloud in dimension $d=2$ with Gaussian noise of $0.25$ on a one-dimensional subspace with $100$ points and $250$ outliers. The orange/ red vector shows the principal component computed by SVD of the point cloud with/ without the $250$ outliers. The green vector is the principal component computed by KV-CBO. Right: Energy $\EE_p(\theta)$ for $\theta \in [0,2\pi)$ for the point cloud on the left and $p=0.5$. The particles are shown in red, $V_0^{\alpha, \EE}$ is shown in green. \label{point cloud and energy 3}}
\end{figure}

\begin{figure}[H]
	\centering
	\includegraphics[width = 7cm, height = 5cm]{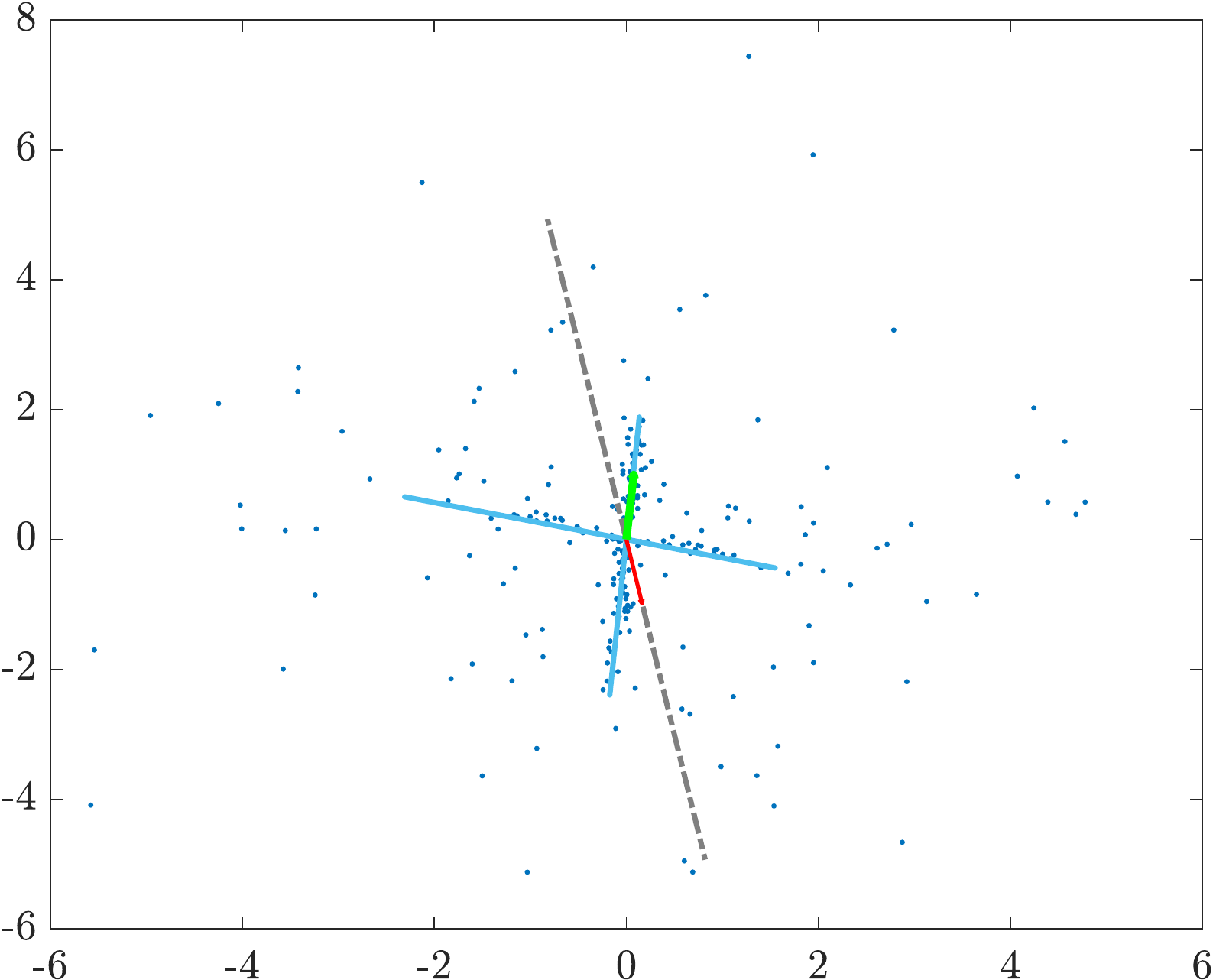}
	\includegraphics[width = 7cm, height = 5cm]{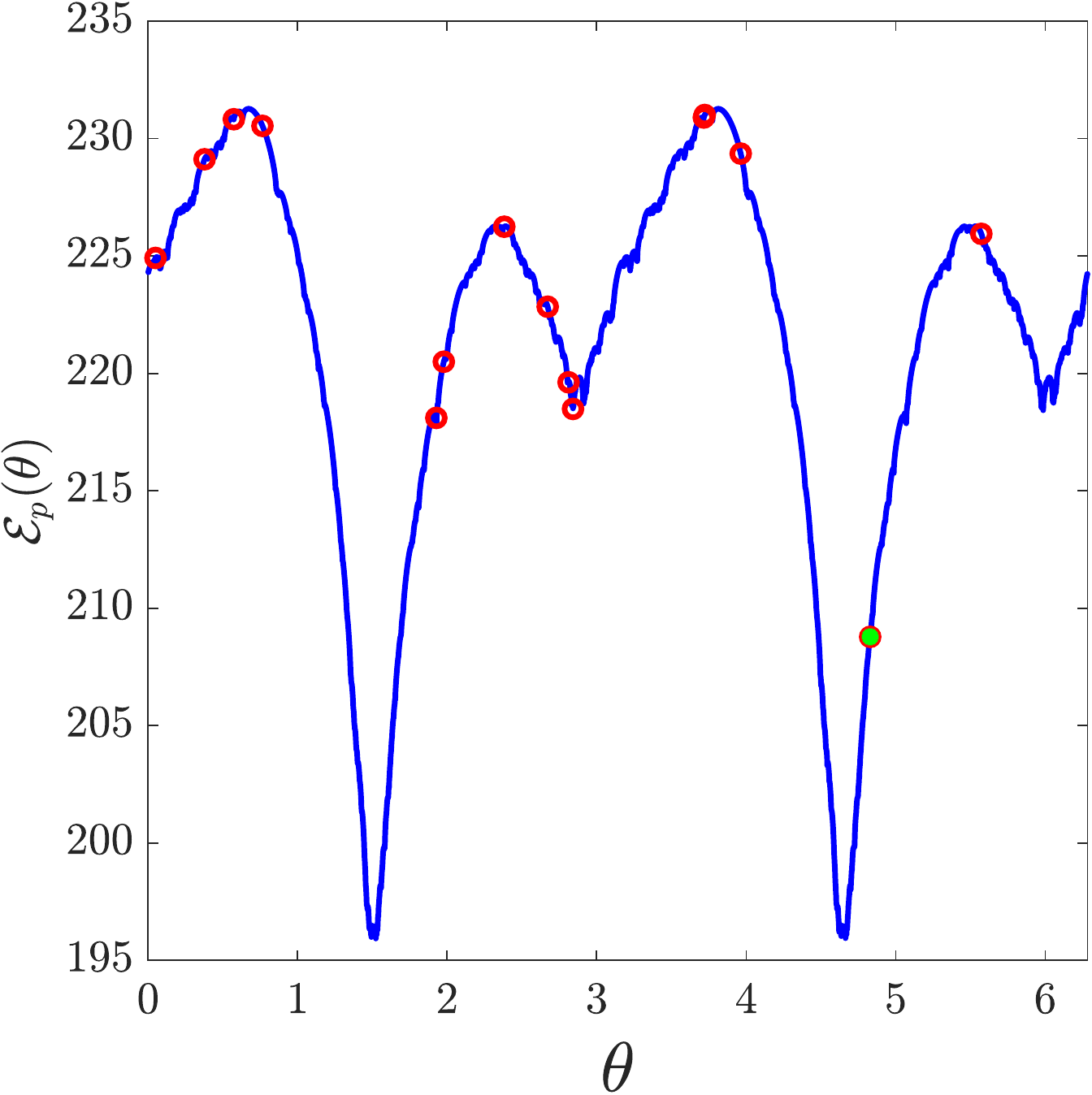}
	\caption{Left: Point cloud with $N_{sp}=2$ one-dimensional subspaces with $100$ points on the first subspace and $40$ points on the second with Gaussian noise of $0.01$. Further, we have added $100$ outliers. The red vector shows the principal direction for the point cloud without the outliers. The green vectors shows the direction computed by KV-CBO. It matches the one-dimensional cluster with $100$ points. Right: Energy $\EE_p(\theta)$ for $\theta \in [0,2\pi)$ and $p=0.2$ for the point cloud on the left. The particles are shown in red.  The initial $V_{0}^{\alpha, \EE}$ is shown in green, superimposing the particle with the smallest energy. \label{point cloud and energy 4}}
\end{figure}

\subsubsection*{Synthetic Data}

In this section we discuss numerical tests for synthetic point clouds in dimensions up to $d=200$ for $p=2$ and $0<p<2$. In  Figures \ref{point cloud and energy 1} to \ref{point cloud and energy 4} we report plots of energies in $d=2$ for different values of $p$. \\

We test the method for point clouds laying on $N_{sp}=25$ nearly parallel one dimensional subspaces and point clouds laying $N_{sp}=25$ randomly chosen subspaces, each with Gaussian noise of $0.01$. The latter point clouds do not have an obvious principal direction, as opposed to the case of nearly parallel subspaces (see Figure \ref{success ratio sigma plot} on the right). In this case a larger number of initial particles is needed to find the minimizer.\\

\noindent
\textbf{Case $p=2$}\\

\noindent
For the case $p=2$ we compare the minimizer $V_{n_T}^{\alpha, \EE}$ computed by KV-CBO  with the minimizer $v^*$ computed by SVD. In Figure \ref{average error easy and difficult} we plot the average error $|V^{ \alpha, \EE}_{n} - v^*|$ for $n=0,...,n_T$ for $25$ runs. 
In the plot on the right we show the success rate for different numbers of particles for different dimensions. We count a run as successful if
\begin{equation}
\min\{|V^{\alpha, \EE}_{n_T} - v^* |, |V^{\alpha, \EE}_{n_T} + v^* |\} \leq 0.01
\end{equation}
where $n_T$ is the final time step. We observe that for  point clouds with nearly parallel one-dimensional subspaces, a very small number of particles already yields good results.  For the point clouds with randomly chosen one-dimensional subspaces, corresponding to a flatter spectrum, the number of particles $N$ has to be chosen larger in order to obtain good results. Still, KV-CBO can certainly be considered an interesting, robust, and efficient alternative method for computing SVD's.\\

\begin{figure}[tb]
	\centering
	\includegraphics[ width = 7cm, height = 5cm]{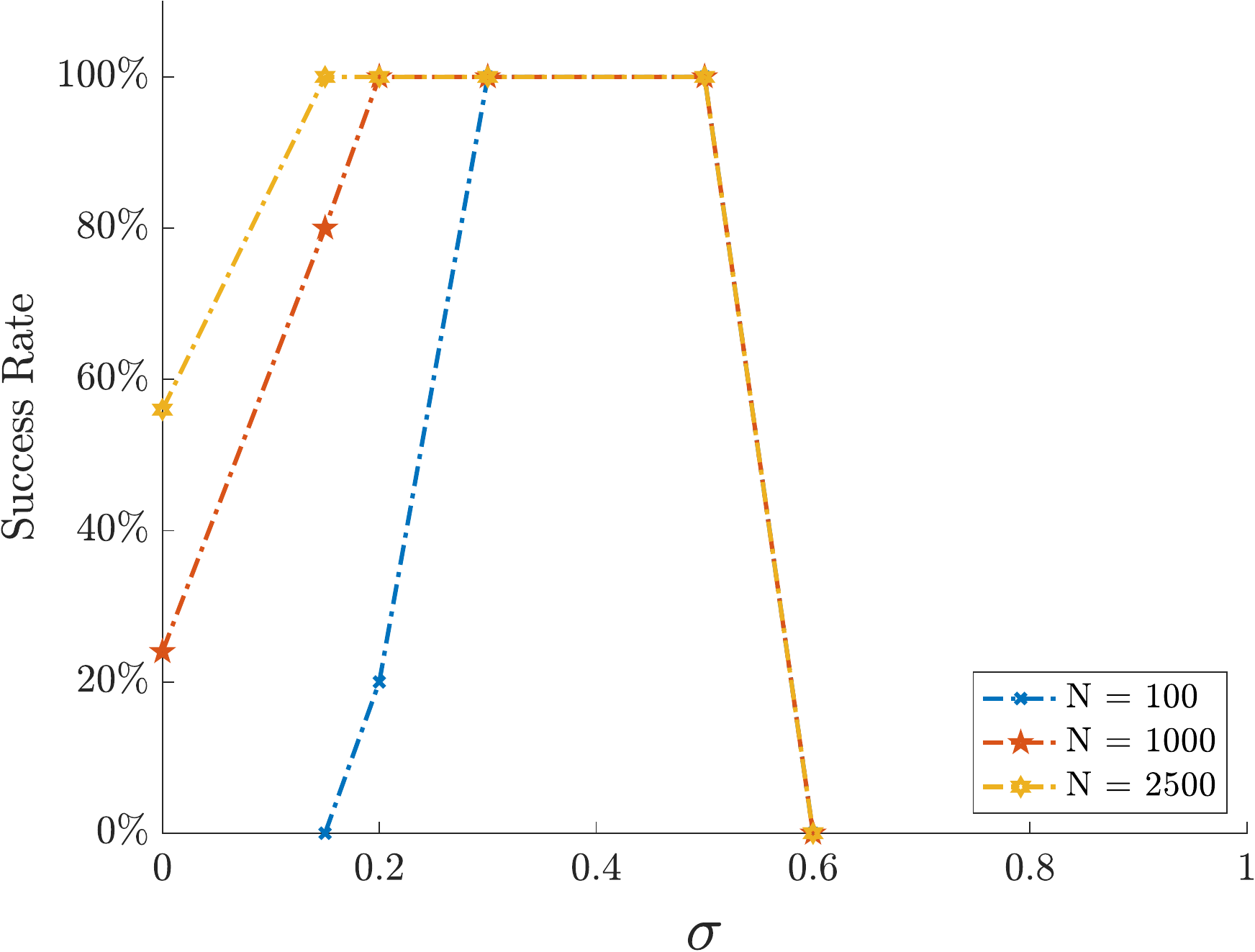}
	\includegraphics[width = 7cm, height = 4.7cm]{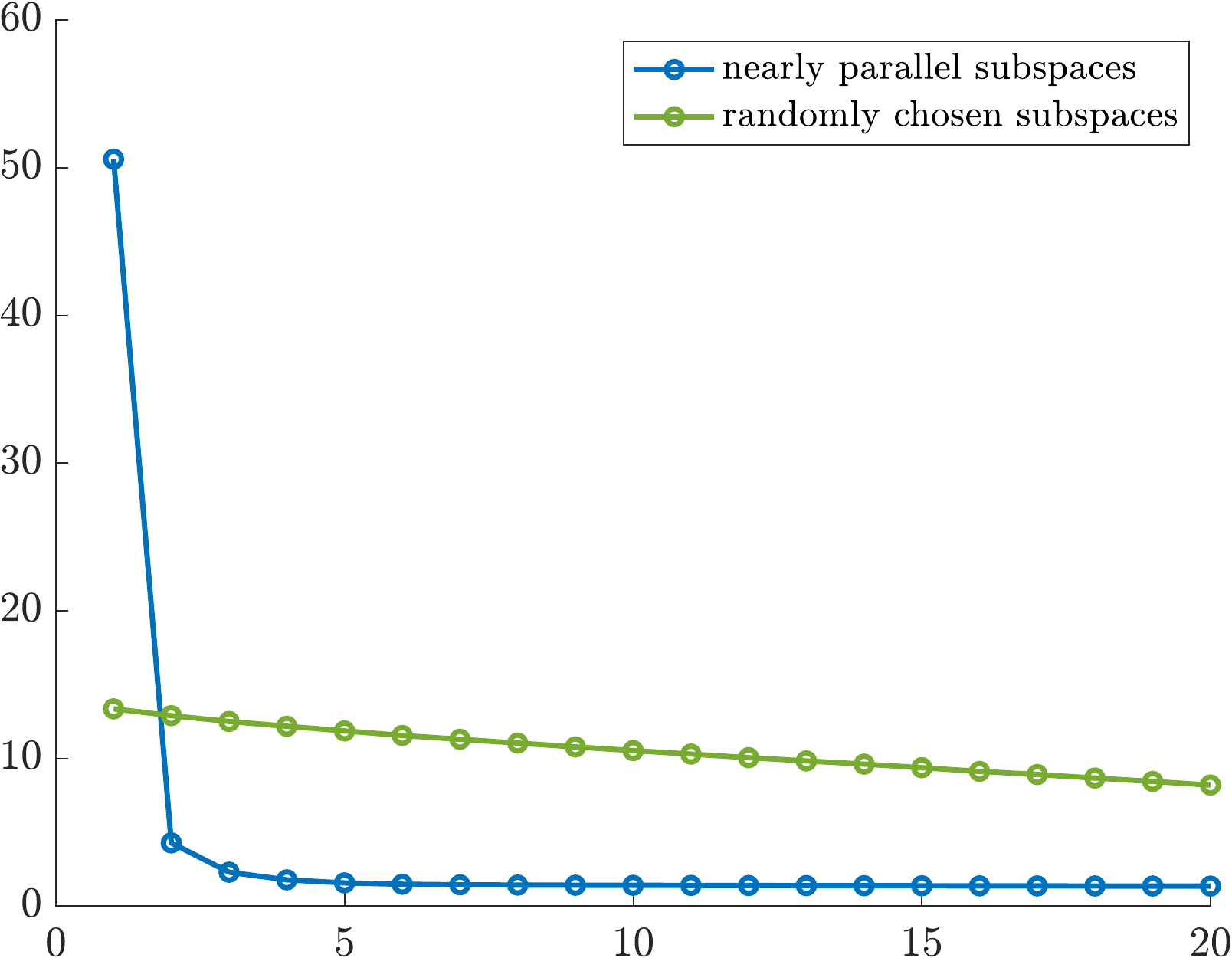}
	\caption{Left: Success rate for different values of $\sigma$ and $N$ in  dimensions $d=10$. Choosing a larger number of particles clearly widens the window from which $\sigma$ can be chosen. Note that for $\sigma = 0$ (deterministic Kuramoto-Vicsek model) we have a success rate of nearly $60\%$ for $N=2500$ particles. Right: Singular value decay of the point cloud with $N_{sp}=25$ nearly parallel (blue) and randomly chosen one-dimensional subspaces (green). \label{success ratio sigma plot}}
\end{figure}

\setlength{\extrarowheight}{0.12cm}
\begin{table}[H]
	\begin{center}
		\begin{tabular}{c|ccccc}
			Dimension & & $d=10$ & $d=100$ & $d=200$ \\
			& & $N_0=1000$ & $N_0=2500$ & $N_0=5000$ \\
			\hline
			\hline
			Relative Error $\leq 10^{-2}$ & Rate & $100\%$ & $100\%$ & $100\%$\\
			\hline
			\hline
			& Rate & $63\%$ & $13\%$ & $0\%$\\
			$\EE_{p, 0}(V_{n_T}^{\alpha, \EE})\leq\EE_{p,0}(V_{FMS})$         & Absolute Error & $2.8413e-12$ & $2.0026e-12$ & $-$ \\
			& Relative Error & $7.6774e-15$ & $5.8669e-15$ & $-$ \\
			\hline
			& Rate & $37\%$ & $87\%$ & $100\%$\\
			$\EE_{p, 0}(V_{n_T}^{\alpha, \EE})>\EE_{p,0}(V_{FMS})$     & Absolute Error & $2.9066e-12$ & $71218e-12$ & $1.3387e-11$ \\
			& Relative Error & $7.9272e-15$ & $2.0745e-14$ & $3.9490e-14$ \\
			\hline
			\hline
		\end{tabular}
		\caption{Numerical comparison of KV-CBO and the FMS method for a point cloud with $N_{sp}=25$ nearly parallel one-dimensional subspaces with Gaussian noise of $0.01$. The relative error is defined in \eqref{sKV vs FMS equally good}. The results are averaged over $100$ runs. \label{Table sKV vs FMS easy}}
	\end{center}
\end{table}%

\setlength{\extrarowheight}{0.12cm}
\begin{table}[H]
	\begin{center}
		\begin{tabular}{c|ccccc}
			Dimension & & $d=10$ & $d=100$ & $d=200$ \\
			& & $N_0=1000$ & $N_0=2500$ & $N_0=5000$ \\
			\hline
			\hline
			Relative Error $\leq 10^{-2}$ & Rate & $100\%$ & $100\%$ & $100\%$\\
			\hline
			\hline
			& Rate & $79\%$ & $14\%$ & $15\%$\\
			$\EE_{p, 0}(V_{n_T}^{\alpha, \EE})\leq\EE_{p,0}(V_{FMS})$         & Absolute Error & $0.0814$ & $4.6205$ & $5.9582$ \\
			& Relative Error & $4.4532e-5$ & $0.0024$ & $0.0031$ \\
			\hline
			& Rate & $21\%$ & $86\%$ & $85\%$\\
			$\EE_{p, 0}(V_{n_T}^{\alpha, \EE})>\EE_{p,0}(V_{FMS})$     & Absolute Error & $0.3312$ & $0.5628$ & $1.4965$ \\
			& Relative Error & $1.8466e-4$ & $2.8669e-4$ & $7.6525e-4$ \\
			\hline
			\hline
		\end{tabular}
		\caption{Numerical comparison of KV-CBO and the FMS method for a point cloud with $N_{sp}=25$ randomly chosen one-dimensional subspaces with Gaussian noise of $0.01$. The relative error is defined in \eqref{sKV vs FMS equally good}. The results are averaged over $100$ runs.\label{Table sKV vs FMS difficult}}
	\end{center}
\end{table}%

\noindent
\textbf{Case $p=1$}\\

\noindent
For $0<p<2$ the energy $\EE_p(v)$ is not smooth enough to fulfill the regularity conditions of Assumptions \ref{assumas} below. In order to fit the experiment to our theoretical findings, we may consider the smoothed energy 
\begin{equation}
\EE_{p,\delta}(v) = \sum_{i=1}^M (|x_i - \langle x_i, v \rangle v|^2 + \delta^2)^{p/2}
\end{equation}
where we chose $\delta = 10^{-7}$ (as $\delta>0$ is chosen so small, it is actually irrelevant from a numerical precision point view). We again test KV-CBO on synthetic point clouds with $N_{sp}=25$ one-dimensional subspaces with $100$ points each, thus $M=2500$. We run the experiment $100$ times in dimension $d\in \{10,100,200\}$ and count one run as successful if the relative error of the function values is less than $1\%$, that is,
\begin{equation}
\frac{|\EE_{p, 0}(V_{n_T}^{\alpha, \EE})-\EE_{p,0}(V_{FMS})|}{\min \{\EE_{p, 0}(V_{n_T}^{\alpha, \EE}), \EE_{p,0}(V_{FMS})\}} \leq 10^{-2} \label{sKV vs FMS equally good},
\end{equation} 
where $V_{FMS}$ denotes the minimum of $\EE_{p,0}(v)$ computed by the FMS method. We note that $V_{n_T}^{\alpha, \EE}$ is the minimizer of the function $\EE_{p,\delta}(v)$ for $\delta \neq 0$ computed by KV-CBO.  We further report the average absolute and relative errors of the function values for the runs for which $\EE_{p, 0}(V_{n_T}^{\alpha, \EE})\leq \EE_{p,0}(V_{FMS})$ as well as $\EE_{p, 0}(V_{n_T}^{\alpha, \EE})>\EE_{p,0}(V_{FMS})$. In the stopping creterium  for KV-CBO \eqref{eq:cons} we chose $\varepsilon = 10^{-10}$, as maximal amount of iterations $n_T=10^4$, and use Algorithm \ref{algo:sKV-CBOfc} to speed up the method. For the FMS method we chose $\varepsilon = 10^{-10}$ and $n_T=10^2$, as FMS method converges to a good minimizer after fewer iterations than KV-CBO.  In Tables \ref{Table sKV vs FMS easy} and \ref{Table sKV vs FMS difficult} we show that \eqref{sKV vs FMS equally good} is fulfilled in $100\%$ of the cases. In other words: KV-CBO  and  state of the art FMS  perform equally good on point clouds with nearly parallel one-dimensional subspaces as well as randomly chosen one-dimensional subspaces. For the former the maximal relative error is in the order of $10^{-14}$ in dimension $d=200$.

\begin{figure}[H]
	\begin{minipage}{0.49\textwidth}
		\centering
		\includegraphics[width = 7cm, height=4cm]{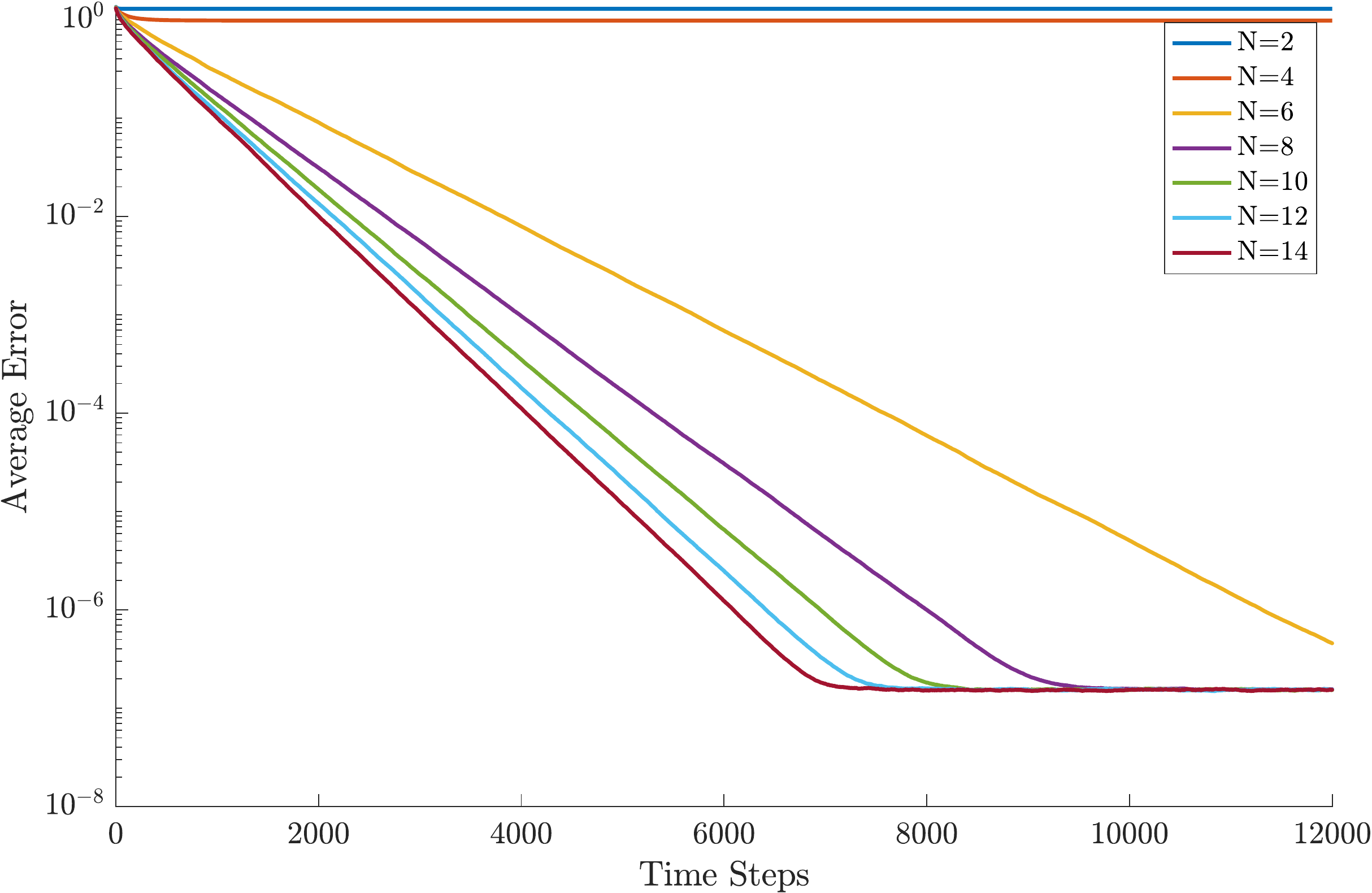}
		\includegraphics[width = 7cm, height=4cm]{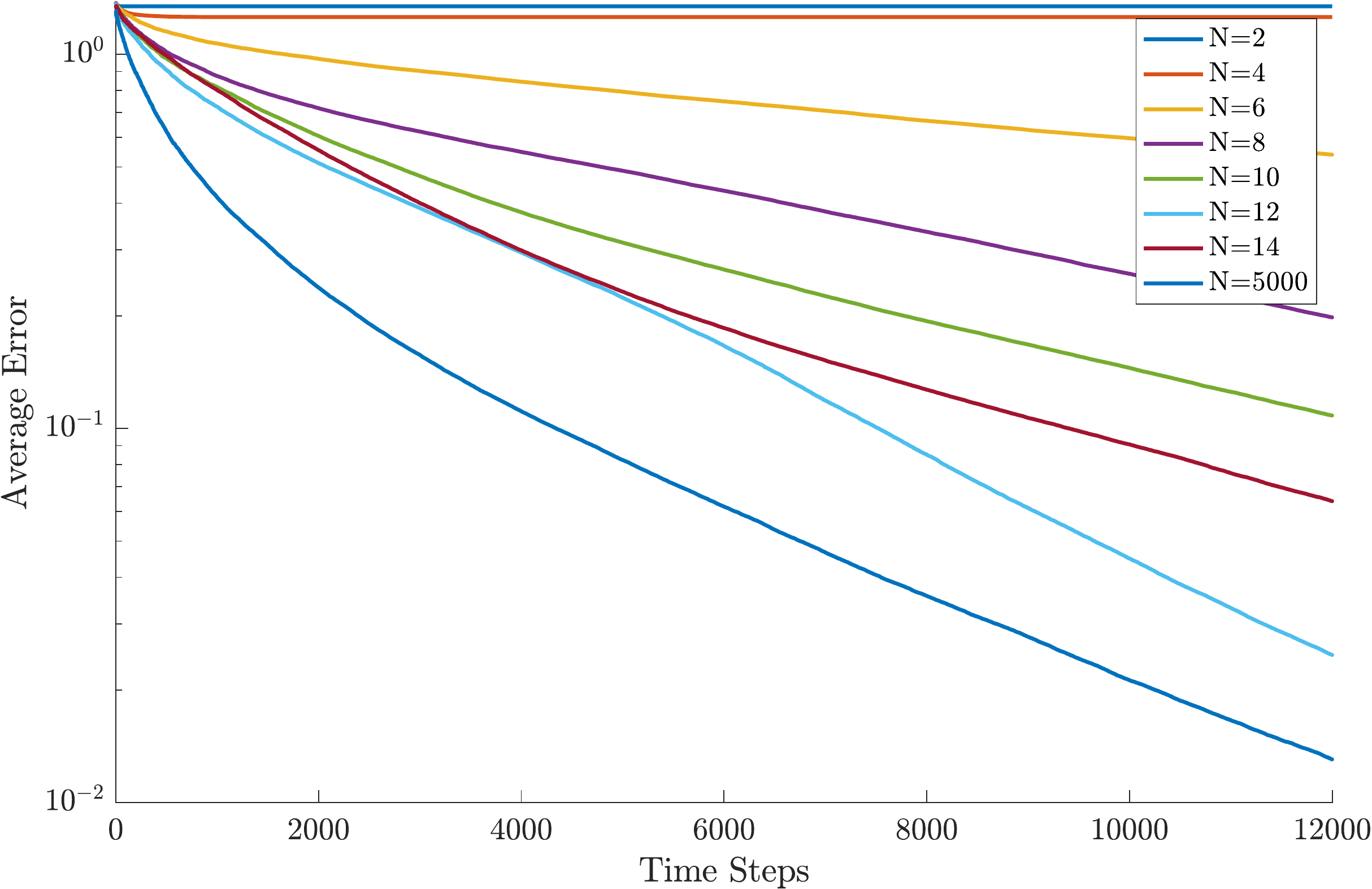}
	\end{minipage}
	\begin{minipage}{0.49\textwidth}
		\centering
		\includegraphics[width = 7cm, height=4cm]{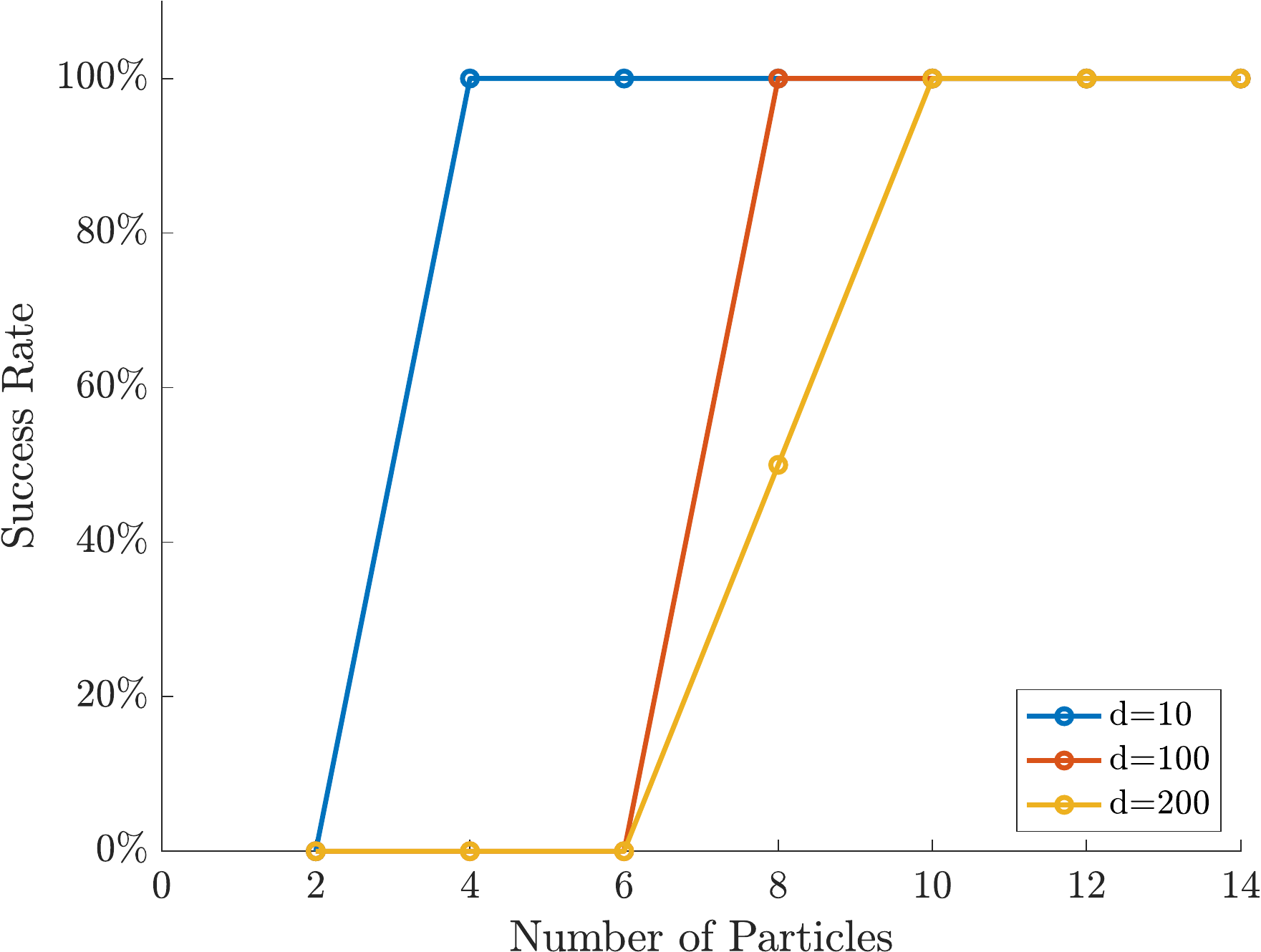}
		\includegraphics[width = 7cm, height=4cm]{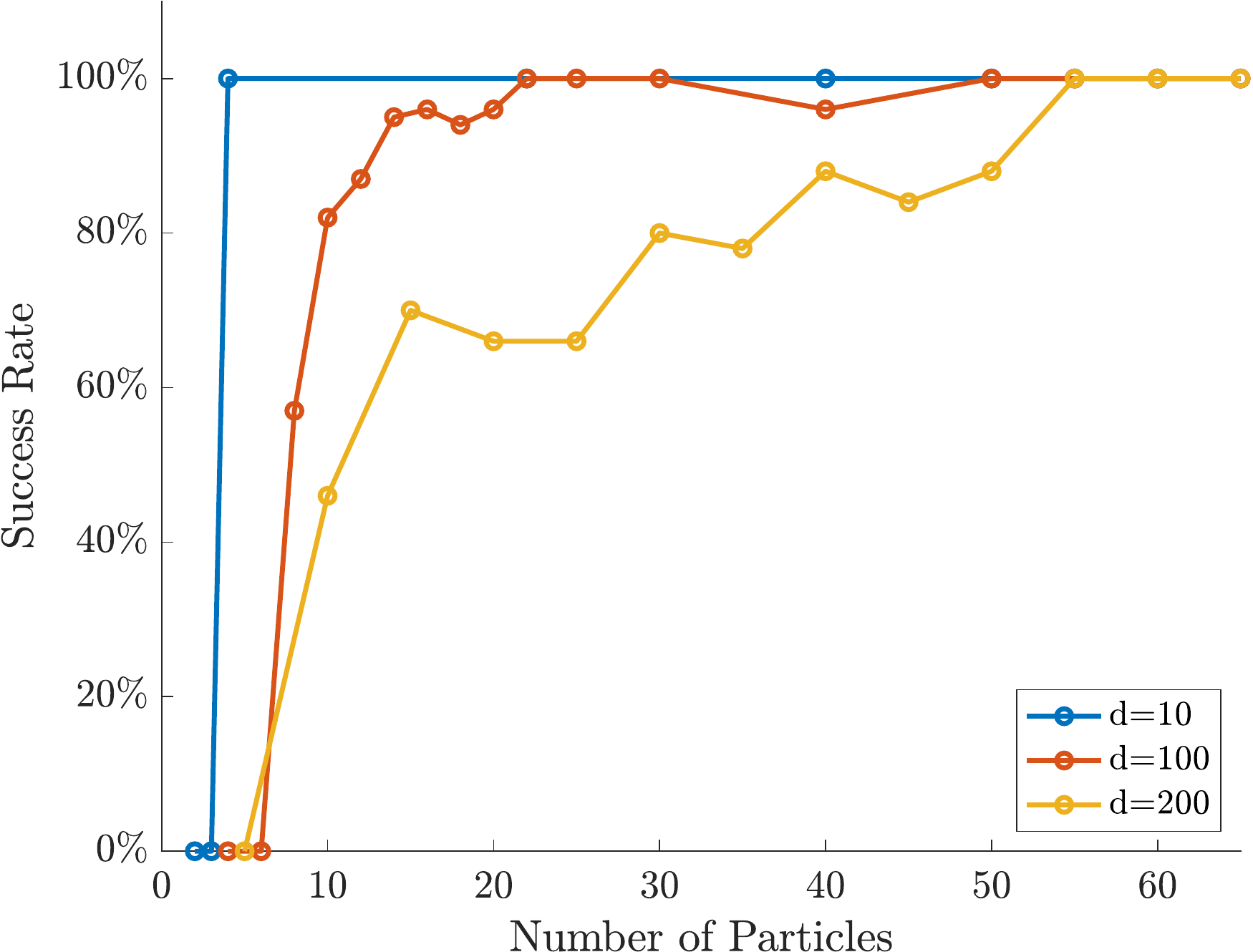}
	\end{minipage}
	\caption{Average error for a point cloud with $N_{sp}=25$ nearly parallel (top left) respectively randomly chosen (bottom left) one-dimensional subspaces in dimension $d=200$ with Gaussian noise of $0.1$ for different numbers of particles. We chose $\alpha = 2\cdot 10^{15}, \sigma = 0.08, \Delta t = 0.25$. The curve for $N = 5000$ particles has been calculated with Algorithm \ref{algo:sKV-CBOfc}. Right: success rate for the same point clouds in dimension $d\in \{10, 100, 200\}$. The results have been averaged 25 times and we count one run as successful if $|V_{n_T}^{\alpha, \EE}-v^* | \leq 0.01$. \label{average error easy and difficult}}
\end{figure}

\subsubsection*{Robust computation of eigenfaces}

In this section we discuss the numerical results of KV-CBO on real-life data. We chose a subset of $M=421$ similar looking pictures of the \textit{10K US Adult Faces Database} \cite{10kUSAdultFaces}. We converted this subset to gray scale images and reduced the size of each picture by factor $4$. We finally extract at a subset of $M=421$ pictures of size $64 \times 45$, which yields a point cloud $X \in \mathbb{R}^{2880\times 421}$. 

\begin{figure}[H]
	\begin{minipage}{0.49\textwidth}
		\includegraphics[height = 2cm]{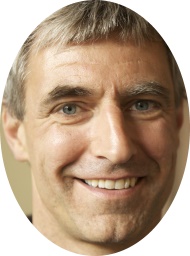}
		\includegraphics[height = 2cm]{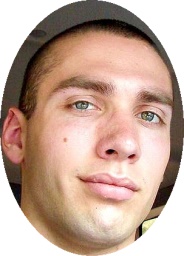}
		\includegraphics[height = 2cm]{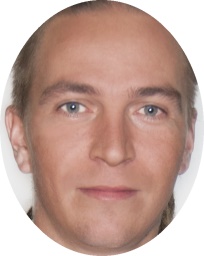}
		\includegraphics[height = 2cm]{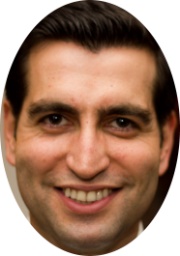}
	\end{minipage}
	\begin{minipage}{0.49\textwidth}
		\includegraphics[height = 2cm]{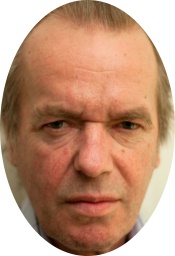}
		\includegraphics[height = 2cm]{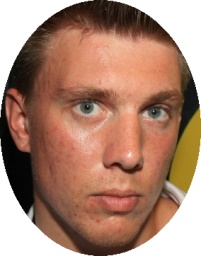}
		\includegraphics[height = 2cm]{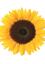}
		\includegraphics[height = 2cm]{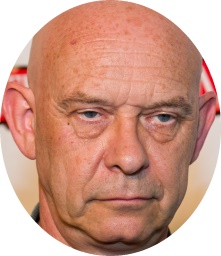}
	\end{minipage}
	\caption{Samples from the \textit{10K US Adult Faces Database} \cite{10kUSAdultFaces} and one instance of outlier.}
\end{figure}

The eigenfaces computed by SVD and KV-CBO are shown in Figure \ref{fig:faces}(a) and Figure  \ref{fig:faces}(b). The computed eigenfaces are visually indistinguishable and  the final error is in the order of $10^{-3}$. We then added $6$ outliers (pictures of different plants and animals on a white background) to the point cloud and again computed the eigenface by SVD (see Figure 14(c)) and by KV-CBO with $p=1$ and $N=5000$ particles (see Figure  \ref{fig:faces}(d)). The eigenface computed by SVD still retain some features, but the difference to the original eigenface (without outliers) is clearly perceivable. Instead, the eigenface computed by the KV-CBO  still looks very similar to the eigenface of the point cloud without outliers. We quantify the accuracy of the results by Peak Signal-to-Noise Ratio (see caption of Figure  \ref{fig:faces}). We then added further $6$ outliers (amounting to a total of $12$ outliers) to the point cloud and again computed the eigenface by SVD (see Figure  \ref{fig:faces}(e)) and  KV-CBO  with $p=0.5$ and $N=50000$ particles (see Figure  \ref{fig:faces}(f)). The difference of both eigenfaces to the original eigenface (without outliers) is clearly visible. The eigenface computed by SVD lost most of the original features. On the other hand, the eigenface computed by KV-CBO still retains the main features. We reiterate that the energy landscape $\EE_{p,\delta}(v)$ is much more complex for $0<p<1$ than for $p\in [1,2]$ (see Figures \ref{point cloud and energy 1} to \ref{point cloud and energy 4}). An increase of the number of particles $N$ did not yield better results.\\

\begin{figure}[H]
	\begin{minipage}{0.32\textwidth}
		\begin{minipage}{0.49\textwidth}
			\centering
			\includegraphics[height = 3cm]{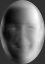}
			(a)
		\end{minipage}
		\begin{minipage}{0.49\textwidth}    \centering
			\includegraphics[height = 3cm]{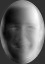}
			{(b)}
		\end{minipage}
	\end{minipage}
	\begin{minipage}{0.32\textwidth}
		\begin{minipage}{0.49\textwidth}
			\centering
			\includegraphics[height = 3cm]{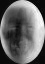}
			{(c)}
		\end{minipage}
		\begin{minipage}{0.49\textwidth}
			\includegraphics[height = 3cm]{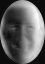}    \centering
			{(d)}
		\end{minipage}
	\end{minipage}
	\begin{minipage}{0.32\textwidth}
		\begin{minipage}{0.49\textwidth}
			\includegraphics[height = 3cm]{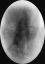}    \centering
			{(e)}
		\end{minipage}
		\begin{minipage}{0.49\textwidth}
			\includegraphics[height = 3cm]{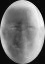}    \centering
			{(f)}
		\end{minipage}
	\end{minipage}
	\caption{Eigenface for the point cloud of faces without outliers computed by SVD (a), and KV-CBO (b). Eigenface for point cloud with 6 outliers computed by SVD (c), and KV-CBO with $p=1$ (d). Eigenface for point cloud with 12 outliers computed by SVD (e), and KV-CBO with $p=0.5$ (f). We used the following parameters: $\alpha = 2\cdot 10^{15}, \lambda = 1, \sigma = 0.019, \Delta t = 0.25, T=25000$, $N = 5000$ and $N_{min}=150$ (see algorithm \ref{algo:sKV-CBOfc}) for (b) and (d). For (f) we used $p=0.5$, $N=50000$ and $N_{min}=5000$. For $p<2$ we used $ \delta = 10^{-7}$. For the case (b) the error to the SVD eigenface was $0.00071$. The Peak Signal-to-Noise Ratio is: 61.4214 for (a) and (b), 15.9764 for (a) and (c), 20.7344 for (a) and (d), 12.3109 for (a) and (e) and 14.2892 for (a) and (f).}\label{fig:faces}
\end{figure}

\begin{figure}[H]
	\centering
	\includegraphics[scale=0.4]{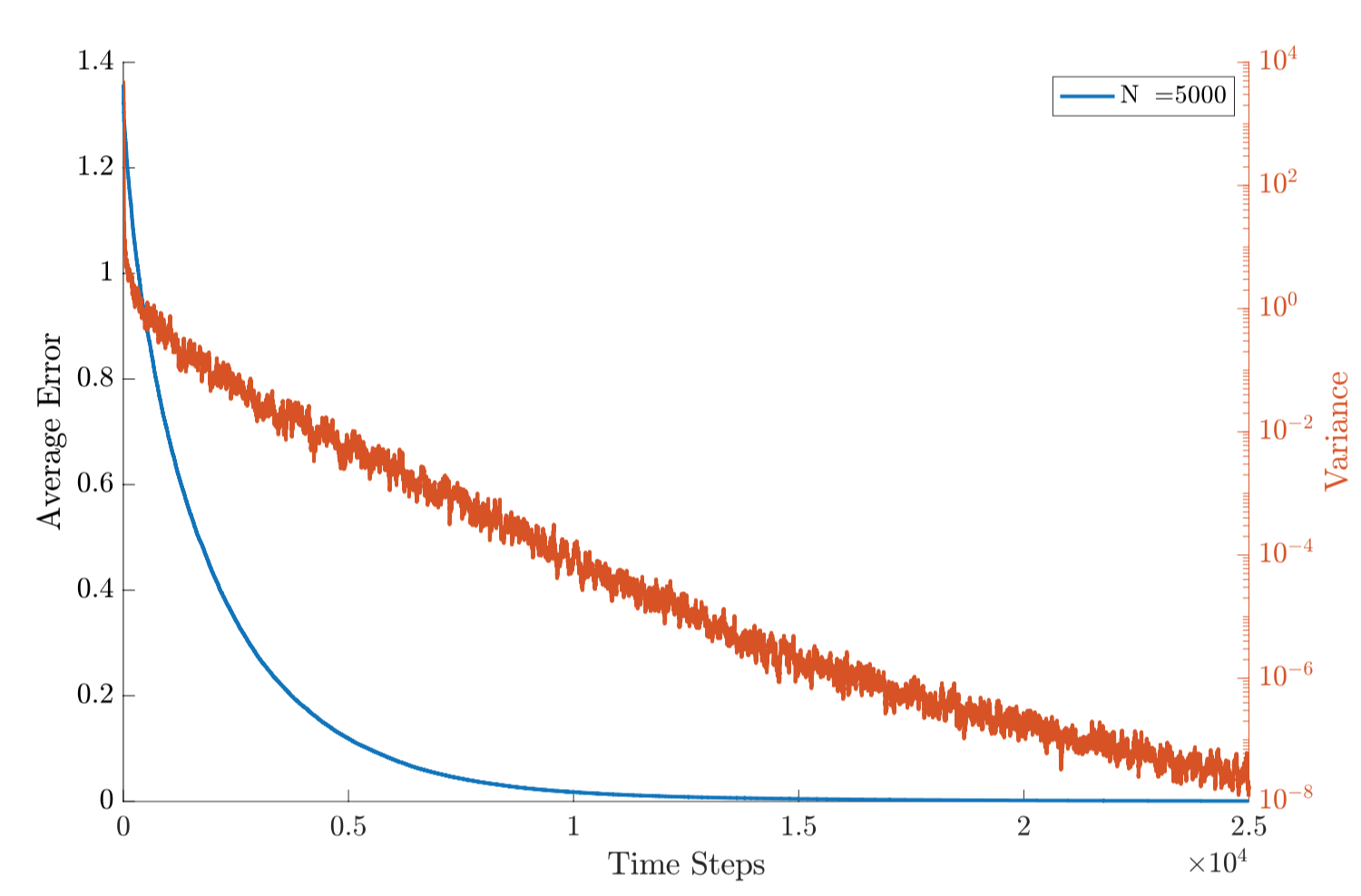}
	\caption{Average error (blue) and variance (red) for the computation of the eigenface (b) in the figure above.}
\end{figure}

\section{Global optimization guarantees}\label{analsec}

\subsection{Main result}
In this section, we address the convergence of the stochastic Kuramoto-Vicsek particle  system \eqref{stochastic Kuramoto-Vicsek} to global minimizers of some cost function $\EE$. In view of the {already} established mean-field limit result \eqref{rateN}, it is actualy sufficient to analyze the large time behavior of the solution $\rho(t,v)$ to the corresponding mean-field PDE \eqref{PDE}. Let us  rewrite \eqref{PDE} as
\begin{equation}\label{PDE2}
\partial_t \rho_t = \Delta_{\BS^{d-1}} (\kappa_t \rho_t) + \nabla_{\BS^{d-1}} \cdot (\eta_t \rho_t)\,,
\end{equation}
where $\kappa_t :=\frac{\sigma^2}{2}|v-v_{\alpha,\EE}(\rho_t)|^2 \in \mathbb{R}$ and $\eta_t:=\lambda \langle v_{\alpha,\EE}(\rho_t),v\rangle v-\lambda v_{\alpha,\EE}(\rho_t)\in \mathbb{R}^d$. We also introduce the auxiliary self-consistent nonlinear SDE  
\begin{align} \label{selfprocess}
d\overline V_t=\lambda  P(\overline V_t)v_{\alpha,\EE} (\rho_t) dt + \sigma |\overline V_t - v_{\alpha,\EE}(\rho_t)  | P(\overline V_t)dB_t-\frac{(d-1)\sigma^2}{2}(\overline V_t-v_{\alpha,\EE} (\rho_t) )^2\frac{\overline V_t}{|\overline V_t|^2}dt\,,
\end{align}
with the initial data $\OV_0$ distributed according to $\rho_0\in\mc{P}(\BS^{d-1})$. Here $\rho_t=\mbox{law}(\overline V_t)$ is also  the unique solution of the PDE \eqref{PDE2}, see \cite[Section 2.3]{fhps20-1}. The well-posedness of \eqref{selfprocess} is shown in \cite[Theorem 2.2]{fhps20-1}.
We now define the expectation and variance of $\rho_t$ as
\begin{equation}
E(\rho_t):=\int_{\BS^{d-1}} v d \rho_t (v) \quad V(\rho_t):=\frac{1}{2}\int_{\BS^{d-1}} |v-E(\rho_t)|^2 d\rho_t (v).
\end{equation}
In the following, we show that, under suitable smoothness requirements, see Assumptions \ref{assumas} below, for any $\epsilon>0$ there exists suitable parameters $\alpha,\lambda,\sigma$ and well-prepared initial distributions $\rho_0$ such that for $T^*>0$ large enough the expected value of the distribution
$E(\rho_{T^*})=\int v d \rho_{T^*}(v)$ is in an $\epsilon$-neightborhood of a global minimizers $v^*$ of $\EE$. The convergence rate is exponential in time and the rate depends on the parameters $\epsilon, \alpha,\lambda,\sigma$ (see Proposition \ref{mainp}). As mentioned in the introduction, this approximation together with \eqref{rateN} and classical results of the convergence of numerical methods for SDE \cite{Platen} yield the convergence of Algorithm \ref{algo:sKV-CBO}. {In particular, 
we shall address the proof of the main result Theorem \ref{mainresult00} and of the quantitative estimate \eqref{mainresultXX} at the end of this section.}\\

In order to formalize the result we state our fundamental assumptions: Throughout this section, the objective function $\EE\in \mathcal{C}^2(\BS^{d-1})$ satisfies the following properties
\begin{assu}\label{assumas}\quad
	\begin{itemize}
		\item[1.] $\EE$ is bounded and $0\leq \underline{\EE}:=\inf \EE \leq \EE \leq  \sup \EE=:\overline \EE < \infty$;
		\item[2.] $\|\nabla \EE\|_\infty\leq c_1$;
		\item[3.] $\max \left\{\|\nabla^2 \EE \|_\infty, \|\Delta \EE \|_\infty\right\} \leq c_2$;
		\item [4.] For any $v \in \BS^{d-1}$ 
		there exists  a minimizer $v^*\in \BS^{d-1}$ of $\EE$ (which may depend on $v$) such that  it holds 
		\begin{equation*}
		|v-v^\ast| \leq  C_0|\EE(v)-\underline \EE|^\beta\,,
		\end{equation*}
		where  $\beta, C_0$ are some positive constants.
	\end{itemize}
\end{assu}
While the assumptions 1.-3. are all automatically fulfilled as soon as smoothness is provided, requirement 4. - which we call {\it inverse continuity assumption} - is a bit more technical and needs to be verified, depending on the specific application. In Section \ref{sec:phaseretr} we provided the concrete example of the phase retrieval problem for which all the conditions are in fact verifiable.
The request of smoothness is exclusively functional to the proof of well-posedness and mean-field limit \cite{fhps20-1} and the proof of asymptotic convergence. As a matter of fact Algorithm \ref{algo:sKV-CBO} and Algorithm \ref{algo:sKV-CBOfc} are implementable even if $\EE$ admits just pointwise evaluations, e.g., $\EE$ is just a continuous function with no further regularity. 
Below we denote $C_{\alpha,\EE}=e^{\alpha(\overline{\EE}-\underline{\EE})}$ and $C_{\sigma,d}=\frac{(d-1)\sigma^2}{2}$. 
\begin{definition}\label{def:wellprep}
	For any given $T>0$ and $\alpha>0$, we say that the initial datum and the parameters are well-prepared if $\rho_0\in \mc{P}_{ac}(\BS^{d-1})\cap L^2(\BS^{d-1})$, and parameters  $\lambda$, $d$, $\beta>0$, $0<\varepsilon\ll 1$,  $0<\delta\ll 1$, $0<\theta<\delta$ satisfy
	\begin{align}
	&C_{\alpha,\EE}^{2\max \{1, \beta\}}\left(V(\rho_0)+\frac{\lambda C_T}{\lambda\theta- 4C_{\alpha,\EE}C_{\sigma,d}}\delta^{\frac{d-2}{4}} \right)^{\frac{1}{2}\min\{1,\beta\}} +\varepsilon^\beta<\frac{\delta-\theta}{C^\ast}\,, \label{wellprep}\\
	&V(\rho_0)+\frac{\lambda C_T}{\lambda\theta- 4C_{\alpha,\EE}C_{\sigma,d}}\delta^{\frac{d-2}{4}}\leq \min\left\{T^{-1}\| \omega_\EE^\alpha\|_{L^1(\rho_0)}^2, T^{-1}\lambda^{-2}\| \omega_\EE^\alpha\|_{L^1(\rho_0)}^4,\frac{3}{8} \right\}
	\end{align}
	and 
	\begin{equation}\label{lamsig}
	\lambda\theta- 4C_{\alpha,\EE}C_{\sigma,d}>0\,,
	\end{equation}
	where $C_T$ is a constant depending only on $\lambda$, $\sigma$, $T$ and $\|\rho_0\|_2$, and  $C^\ast>0$ is a constant depending only on $c_1,\beta, C_0$ ($c_1,\beta,C_0$ are used in Assumption \ref{assumas}).  Both $C_T$ and $C^*$ need to be subsumed from the proof of Proposition \ref{mainp} and they are both dimension independent.
\end{definition}

We shall prove first the following result. 
\begin{theorem}\label{thm:mainresult}
	Let us fix $\varepsilon_1>0$ small and assume that the initial datum $\rho_0$ and parameters $\{d,\beta,\varepsilon, \delta, \theta,\lambda,\sigma\}$ are well-prepared for a time horizon $T^*>0$  and parameter $\alpha^*>0$ large enough. Then $E(\rho_{T^*})$  well approximates a minimizer $v^*$ of $\EE$, and the following quantitative estimate holds 
	\begin{equation}\label{locest}
	\left |E(\rho_{T^*})-v^* \right |\leq  \epsilon,
	\end{equation}
	for
	\begin{equation}
	\epsilon:=C(C_0,c_1,\beta)\left((1+C_{\alpha^*,\EE}^{\beta})\left(\frac{\lambda C_{T^*}}{\lambda\theta- 4C_{\alpha^*,\EE}C_{\sigma,d}}\delta^{\frac{d-2}{4}}+\varepsilon_1\right )^{\min\left \{1,\frac{\beta}{2}\right \}}+\varepsilon^\beta\right)\,.
	\end{equation}
\end{theorem}

\begin{remark}
	The conditions of well-preparation \eqref{wellprep} require that the initial datum $\rho_0$ is both well-concentrated and at the same time $v_{\alpha^*, \EE}(\rho_0)$ already approximates well a global minimizer. Technically this is enforced by requiring that the product $C_{\alpha^*,\EE}^{2\max \{1, \beta\}} V(\rho_0)^{1/2}$ is small for $\alpha^*$ large. Of course, this condition is fulfilled for any initial density $\rho_0$, which is well-concentrated in the near of a global minimizer. Hence, the conditions \eqref{wellprep} of well-preparation of $\rho_0$ may have a locality flavour. However, in the case the function $\EE$ is symmetric, i.e., $\EE(v)= \EE(-v)$ (as it happens in numerous applications, in particular the ones we present in this paper), then the condition is generically/practically satisfied at least for one of the two global minimizers $\pm v^*$, yielding essentially a global result.
\end{remark}

The proof of Theorem \ref{thm:mainresult} is based on showing the monotone decay of the variance $V(\rho_t)$ under the assumption of well-preparation (Definition \ref{def:wellprep}) and simultaneously by using the Laplace principle \eqref{Laplace} and the inverse continuity property 4. of Assumptions \ref{assumas} to derive the quantitative estimate
\begin{equation}\label{estinvcont}
\left|\frac{E(\rho_t)}{|E(\rho_t)|}- v^\ast\right|\leq C(C_0,c_1,\beta)\left((C_{\alpha,\EE})^{\beta}V(\rho_t)^{\frac{\beta}{2}}+\varepsilon^\beta\right).
\end{equation}
The monotone decay of the variance is deduced by computing and estimating explicitly its derivative:
\begin{align*}
\frac{d}{dt}V(\rho_t)&=-\lambda V(\rho_t)\la E(\rho_t),v_{\alpha,\EE}\ra -\frac{\lambda}{2}\frac{v_{\alpha, \EE}^2+1}{2}2V(\rho_t)+\frac{\lambda}{4}\int_{\BS^{d-1}} (E(\rho_t)-v)^2 (v-v_{\alpha, \EE})^2 d\rho_t\\
&\quad+C_{\sigma,d}\int_{\BS^{d-1}}(v-v_{\alpha, \EE})^{2}\la E(\rho_t),v \ra d\rho_t\\
&\leq -\lambda V(\rho_t)\left(\la E(\rho_t),v_{\alpha,\EE}\ra +\frac{v_{\alpha, \EE}^2+1}{2}\right)\\ &\quad+\frac{\lambda}{4}\int_{\BS^{d-1}} (E(\rho_t)-v)^2 (v-v_{\alpha, \EE})^2 d\rho_t+4C_{\alpha,\EE}C_{\sigma,d}V(\rho_t)\,.
\end{align*}
The idea is to balance all the terms on the right-hand side by using the parameters $\lambda,\sigma$ in such a way of obtaining a negative sign. This also requires to show that, as soon as $V(\rho_t)$ is small enough, 
$|E(\rho_t)|\approx \la E,v_{\alpha,\EE}\ra \approx |v_{\alpha, \EE}| \approx 1$ and these estimates are worked out in Lemma \ref{lemv}. For ease of notation, for any vector $v \in \mathbb R^d$ we may write $v^2$ to mean $|v|^2$. 

\subsection{Auxiliary lemmas}
A simple computation yields $2V(\rho_t)=1-E(\rho_t)^2$. In particular,  as soon as $V(\rho_t)$ is small $E(\rho_t)^2 \approx 1$ and below we will silently apply the assignment $\EE(E(\rho_t)):=\EE\left ( \frac{E(\rho_t)}{|E(\rho_t)|} \right)$ by normal extension.
Since  $E(\rho_t)=\mathbb{E}[\OV_t]$, it follows from \eqref{selfprocess} that
\begin{equation}
\frac{d}{dt} E(\rho_t) = -\int_{\BS^{d-1}} \eta_td\rho_t - \int_{\BS^{d-1}}\frac{(d-1)\sigma^2}{2}(v-v_{\alpha, \EE})^2 v d\rho_t.
\end{equation}

In the following lemma, we summarize some useful estimates of $v_{\alpha,\EE}(\rho_t)$, $E(\rho_t)$ and $V(\rho_t)$. Here we recall the definition
\begin{equation}
v_{\alpha,\EE}(\rho_t):=\frac{\int_{\BS^{d-1}} v \omega_\alpha^{\EE}(v) d \rho_t(v)}{\|\omega_\alpha^{\EE}\|_{L^{1}(\rho_t)}}=\frac{\int_{\BS^{d-1}} v e^{-\alpha \EE( v )} d \rho_t(v)}{\|e^{-\alpha \EE}\|_{L^{1}(\rho_t)}}\,.
\end{equation} 
\begin{lemma}\label{lemv}
	Let $v_{\alpha,\EE}(\rho_t)$ be defined as above. It holds that
	\begin{enumerate}
		\item $\int_{\BS^{d-1}} |v-v_{\alpha,\EE}(\rho_t)|^2d\rho_t\leq 4C_{\alpha,\EE}V(\rho_t)$ and $\int_{\BS^{d-1}} |v-v_{\alpha,\EE}(\rho_t)|d\rho_t \leq 2C_{\alpha,\EE}V(\rho_t)^{\frac{1}{2}}$;
		\item  $v_{\alpha,\EE}(\rho_t)^2\geq 1-4C_{\alpha,\EE}^2V(\rho_t)$;
		\item $|v_{\alpha,\EE}(\rho_t)-E(\rho_t)|^2\leq (4C_{\alpha,\EE}^2-2)V(\rho_t)$;
	\end{enumerate}
	where $C_{\alpha,\EE}=e^{\alpha(\overline{\EE}-\underline{\EE})}$.
\end{lemma}

Before proving the key estimate \eqref{estinvcont},  we need a lower bound on the norm of the weights $\|\omega_\EE^\alpha \|_{L^1(\rho_t)}$, which is ensured by the
following auxiliary result.
\begin{lemma}\label{lemome}
	Let $c_1, c_2$ be the constants from the assumptions on $\EE$. Then we have
	\begin{equation}\label{eqL1}
	\frac{d}{dt}\| \omega_\EE^\alpha\|_{L^1(\rho_t)}^2 \geq -b_1(\sigma,d,\alpha,c_1,c_2,\underline{\EE})V(\rho_t)-b_2(d,\alpha,c_1,\underline{\EE})\lambda V(\rho_t)^{\frac{1}{2}} 
	\end{equation}
	with $0\leq b_1,b_2\leq 1$  and $b_1,b_2\to 0$ as $\alpha\to\infty$.
\end{lemma}

\subsection{Proof of the large time asymptotic result}

\begin{proposition}\label{thmE}For any fixed $T>0$, assume that 
	$$\overline{\mc{V}}_T:=\sup\limits_{0\leq t\leq T}V(\rho_t)\leq \min\left\{T^{-1}\| \omega_\EE^\alpha\|_{L^1(\rho_0)}^2, T^{-1}\lambda^{-2}\| \omega_\EE^\alpha\|_{L^1(\rho_0)}^4,\frac{3}{8} \right\}\,.$$
	Then for any $\varepsilon>0$, there exists a minimizer $v^*$ of $\EE$ such that 
	\begin{equation}
	\left|\frac{E(\rho_t)}{|E(\rho_t)|}- v^\ast\right|\leq C(C_0,c_1,\beta)\left((C_{\alpha,\EE})^{\beta}V(\rho_t)^{\frac{\beta}{2}}+\varepsilon^\beta\right)\quad \mbox{ for all }t\in[0,T]
	\end{equation}
	holds for any $\alpha>\alpha_0$ with some $\alpha_0\gg1$, where $C_{\alpha,\EE}=e^{\alpha(\overline \EE-\underline \EE)}$, and $C_0$, $c_1$, $\beta$ are used in Assumption \ref{assumas}. Moreover, as soon as $|E(\rho_t)| \geq 1/2$
	\begin{equation}\label{comva}
	\left|v_{\alpha,\EE}(\rho_t)-\frac{E(\rho_t)}{|E(\rho_t)|}\right|^2\leq  \left (8C_{\alpha,\EE}^2-\frac{4}{3} \right)V(\rho_t)\,.
	\end{equation}
\end{proposition}
As it is needed in the proof of this proposition, for readers' convenience, we give a brief introduction of the Wasserstein metric in the following definition, we refer, e.g., to \cite{ambrosio2008gradient} for more details. 
\begin{definition}[Wasserstein Metric] For any $1\leq p < \infty$,
	let $\mc{P}_p(\RR^{d})$ be the space of Borel probability measures on $\RR^{d}$ with finite $p$ moment. We equip this space with the Wasserstein distance 
	\begin{equation}
	W_p^{p}(\mu, \nu):=\inf\left\{\int_{\RR^d\times \RR^d} |z-\hat{z}|^{p}\ d\pi(\mu, \nu)\ \big| \ \pi \in \Pi(\mu, \nu)\right\}
	\end{equation}
	where $\Pi(\mu, \nu)$ denotes the collection of all Borel probability measures on $\RR^d\times \RR^d$ with marginals $\mu$ and $\nu$ in the first and second component respectively. If $\mu,\nu \in \mc{P}(\RR^{d})$ have bounded support, then the $1$-Wasserstein distance can be equivalently expressed in terms of the dual formulation
	\begin{equation}
	W_1(\mu, \nu):=\sup\left\{\int_{\RR^d} f(v) d(\mu- \nu)(v) | f \in \operatorname{Lip}(\RR^d), \operatorname{Lip}(f) \leq 1  \right\}
	\end{equation}
\end{definition}
\begin{proof}({\bf Proposition \ref{thmE}})
	It follows from Lemma \ref{lemome} that
	\begin{align*}
	\| \omega_\EE^\alpha\|_{L^1(\rho_t)}^2 &\geq \| \omega_\EE^\alpha\|_{L^1(\rho_0)}^2-b_1(\alpha)\int_0^tV(\rho_s)ds-b_2(\alpha)\lambda\int_0^tV(\rho_s)^{\frac{1}{2}}ds\\
	&\geq  \| \omega_\EE^\alpha\|_{L^1(\rho_0)}^2-b_1(\alpha)\overline{\mc{V}}_TT-b_2(\alpha)\lambda\overline{\mc{V}}_T^{\frac{1}{2}}T\\
	&\geq   \| \omega_\EE^\alpha\|_{L^1(\rho_0)}^2-b_1(\alpha)\| \omega_\EE^\alpha\|_{L^1(\rho_0)}^2-b_2(\alpha)\| \omega_\EE^\alpha\|_{L^1(\rho_0)}^2\,,
	\end{align*}
	where we have used the assumption  
	\begin{equation}
	\overline{\mc{V}}_T:=\sup\limits_{0\leq t\leq T}V(\rho_t)\leq      \min\left\{ T^{-1}\| \omega_\EE^\alpha\|_{L^1(\rho_0)}^2,T^{-1} \lambda^{-2}\| \omega_\EE^\alpha\|_{L^1(\rho_0)}^4\right\}\,.
	\end{equation}
	The above inequality implies
	\begin{align*}
	-\frac{1}{\alpha}\log \|\omega_\EE^{\alpha}\|_{L^1(\rho_t)}\leq -\frac{1}{\alpha}\log \|\omega_\EE^{\alpha}\|_{L^1(\rho_0)}-\frac{1}{2\alpha}\log\left(1-b_1(\alpha)-b_2(\alpha)\right)\,.
	\end{align*}
	
	The Laplace principle states
	\begin{equation}
	\lim_{\alpha \to \infty} -\frac{1}{\alpha} \log \|\omega_\EE^{\alpha}\|_{L^1(\rho_0)} = \underline\EE\,,
	\end{equation}
	which implies the existence of an $\alpha_1\gg 1$ such that  any $\alpha > \alpha_1$ it holds
	\begin{equation}
	-\frac{1}{\alpha} \log \|\omega_\EE^{\alpha}\|_{L^1(\rho_0)}- \underline \EE< \frac{\varepsilon}{2}
	\end{equation}
	for any $\varepsilon > 0$. Together with the fact that $b_1(\alpha),b_2(\alpha)\to 0$ as $\alpha \to \infty$, it yields
	that
	\begin{align*}
	-\frac{1}{\alpha}\log \|\omega_\EE^{\alpha}\|_{L^1(\rho_t)} -\underline \EE \leq -\frac{1}{\alpha}\log \|\omega_\EE^{\alpha}\|_{L^1(\rho_0)}-\underline{\EE}-\frac{1}{2\alpha}\log\left(1-b_1(\alpha)-b_2(\alpha)\right)\leq \varepsilon\,,
	\end{align*}
	for any $\alpha>\alpha_2$ with some $\alpha_2\gg1$. Let us assume that $\overline{\mc{V}}_T\leq \frac{3}{8}$, then $$\frac{1}{2}\leq |E(\rho_t)|\leq 1\,.$$
	By the dual representation of $1$-Wasserstein distance $W_1$, we know that
	\begin{align}\label{com}
	&\left|\|\omega_\EE^{\alpha}\|_{L^1(\rho_t)}-\omega_\EE^{\alpha}\left(\frac{E(\rho_t)}{|E(\rho_t)|} \right)\right|=\left|\int_{\RR^d} e^{-\alpha\EE(v)} d(\rho_t(v)-\delta_{\frac{E(\rho_t)}{|E(\rho_t)|}}(v))\right|\notag\\
	\leq &\alpha e^{-\alpha\underline{\EE}}\|\nabla\EE\|_\infty W_1(\rho_t,\delta_{\frac{E(\rho_t)}{|E(\rho_t)|}})\leq \alpha c_1 e^{-\alpha\underline{\EE}}W_2(\rho_t,\delta_{\frac{E(\rho_t)}{|E(\rho_t)|}})\leq  2\sqrt{\frac{2}{3}}\alpha c_1 e^{-\alpha\underline{\EE}}V(\rho_t)^{\frac{1}{2}}\,.
	\end{align}
	Here we have used the fact that
	\begin{equation}\label{EdE}
	W_2(\rho_t,\delta_{\frac{E(\rho_t)}{|E(\rho_t)|}})^2\leq \int_{\BS^{d-1}}\left |v-\frac{E(\rho_t)}{|E(\rho_t)|} \right |^2d\rho_t=2-2|E(\rho_t)|=\frac{4V(\rho_t)}{1+|E(\rho_t)|}\leq \frac{8}{3}V(\rho_t)\,.
	\end{equation}
	Above \eqref{com} leads to
	\begin{align*}
	&\left|-\frac{1}{\alpha}\log \|\omega_\EE^{\alpha}\|_{L^1(\rho_t)} -\EE\left({\frac{E(\rho_t)}{|E(\rho_t)|}}\right )\right|= \left|-\frac{1}{\alpha}\left(\log \|\omega_\EE^{\alpha}\|_{L^1(\rho_t)}-\log \omega_\EE^{\alpha}\left({\frac{E(\rho_t)}{|E(\rho_t)|}} \right) \right) \right|\\
	\leq & \frac{e^{\alpha\overline{\EE}}}{\alpha}\left|\|\omega_\EE^{\alpha}\|_{L^1(\rho_t)}-\omega_\EE^{\alpha}\left(\frac{E(\rho_t)}{|E(\rho_t)|} \right)\right|\leq  2\sqrt{\frac{2}{3}} c_1C_{\alpha,\EE}V(\rho_t)^{\frac{1}{2}}\,.
	\end{align*}
	Hence we have
	\begin{align*}
	0\leq \EE\left ({\frac{E(\rho_t)}{|E(\rho_t)|}} \right )-\underline \EE &\leq  \EE\left({\frac{E(\rho_t)}{|E(\rho_t)|}} \right)-\frac{-1}{\alpha}\log \|\omega_\EE^{\alpha}\|_{L^1(\rho_t)} +\frac{-1}{\alpha}\log \|\omega_\EE^{\alpha}\|_{L^1(\rho_t)} -\underline \EE\\
	&\leq   2\sqrt{\frac{2}{3}}c_1C_{\alpha,\EE}V(\rho_t)^{\frac{1}{2}}+\varepsilon\,,
	\end{align*}
	which yields that
	\begin{align*}
	\left|\frac{E(\rho_t)}{|E(\rho_t)|}- v^*\right|\leq C_0\left|\EE\left(\frac{E(\rho_t)}{|E(\rho_t)|} \right)-\underline \EE\right|^\beta \leq  C(C_0,c_1,\beta)\left((C_{\alpha,\EE})^{\beta}V(\rho_t)^{\frac{\beta}{2}}+\varepsilon^\beta\right)\,.
	\end{align*}
	by the inverse continuity $4.$ in Assumption \ref{assumas}, where $v^*$ is a minimizer of $\EE$. Next we compute
	\begin{align*}
	&\left|v_{\alpha,\EE}(\rho_t)-\frac{E(\rho_t)}{|E(\rho_t)|}\right|^2=\int_{\BS^{d-1}}\left |v_{\alpha,\EE}(\rho_t)-v+v-\frac{E(\rho_t)}{|E(\rho_t)|} \right|^2d\rho_t(v)\\
	=&\int_{\BS^{d-1}}|v_{\alpha,\EE}(\rho_t)-v|^2d\rho_t+\int_{\BS^{d-1}}\left |v-\frac{E(\rho_t)}{|E(\rho_t)|}\right |^2d\rho_t+2\int_{\BS^{d-1}}\left \la v_{\alpha,\EE}(\rho_t)-v,v-\frac{E(\rho_t)}{|E(\rho_t)|} \right\ra d\rho_t\\
	\leq&4C_{\alpha,\EE}^2V(\rho_t)+\frac{8}{3}V(\rho_t)+2|E(\rho_t)|-2+\left (2-\frac{2}{|E(\rho_t)|} \right )\la v_{\alpha,\EE}(\rho_t) ,E(\rho_t)\ra \\
	\leq &4C_{\alpha,\EE}^2V(\rho_t)+\frac{8}{3}V(\rho_t) -2V(\rho_t)+\left (2-\frac{2}{|E(\rho_t)|} \right)\la v_{\alpha,\EE}(\rho_t) ,E(\rho_t)\ra \\
	=&(4C_{\alpha,\EE}^2+\frac{2}{3})V(\rho_t)+\left  (2-\frac{2}{|E(\rho_t)|} \right)\la v_{\alpha,\EE}(\rho_t) ,E(\rho_t)\ra\,,
	\end{align*}
	where we have used \eqref{EdE} and $\frac{1}{2}\leq |E(\rho_t)|\leq 1$. Notice that
	\begin{align*}
	\left (2-\frac{2}{|E(\rho_t)|} \right)\la v_{\alpha,\EE}(\rho_t) ,E(\rho_t)\ra&=(2-\frac{2}{|E(\rho_t)|})\frac{ v_{\alpha,\EE}(\rho_t)^2+E(\rho_t)^2-| v_{\alpha,\EE}(\rho_t)-E(\rho_t)|^2}{2}\\
	&\leq \left (\frac{2}{|E(\rho_t)|}-2 \right)\frac{| v_{\alpha,\EE}(\rho_t)-E(\rho_t)|^2}{2}\leq (4C_{\alpha,\EE}^2-2)V(\rho_t)\,.
	\end{align*}
	Thus we have
	$$
	\left|v_{\alpha,\EE}(\rho_t)-\frac{E(\rho_t)}{|E(\rho_t)|}\right|^2\leq  \left (8C_{\alpha,\EE}^2-\frac{4}{3} \right)V(\rho_t)\,.
	$$
	Hence we complete the proof.
\end{proof}
The next ingredient is proving the monotone decay of the variance $V(\rho_t)$ under assumptions of well-preparation (see Definition \ref{def:wellprep}).

\begin{proposition}\label{mainp}Let us fix any $T>0$ and choose $\alpha$ large enough and assume that the parameters and the initial datum are well-prepared in the sense of Definition \ref{def:wellprep}. Then it holds
	\begin{equation}
	V(\rho_t)\leq V(\rho_0)e^{-(\lambda\theta- 4C_{\alpha,\EE}C_{\sigma,d})t}+\frac{\lambda C_T}{\lambda\theta- 4C_{\alpha,\EE}C_{\sigma,d}}\delta^{\frac{d-2}{4}}\quad\mbox{ for all }t\in[0,T] \,.
	\end{equation}
\end{proposition}
\begin{proof}
	Let us compute the derivative of the variance (where $C_{\sigma,d}=\frac{(d-1)\sigma^2}{2}$)
	\begin{align*}
	\frac{d}{dt}V(\rho_t) &= \frac{1}{2}\frac{d}{dt}\bigg( \int_{\BS^{d-1}} v^{2}d\rho_t-E(\rho_t)^{2}\bigg)=\frac{1}{2}\frac{d}{dt}\bigg( 1-E(\rho_t)^{2}\bigg)= - E(\rho_t)\frac{d}{dt}E(\rho_t)\\
	&= E(\rho_t)\int_{\BS^{d-1}}\eta_t d\rho_t + C_{\sigma,d}\int_{\BS^{d-1}}(v-v_{\alpha, \EE})^{2}\la E(\rho_t),v \ra d\rho_t \\
	&= \lambda \int_{\BS^{d-1}}\langle v_{\alpha, \EE}, v \rangle  \la E(\rho_t), v\ra -\la E(\rho_t), v_{\alpha, \EE}\ra d\rho_t +  C_{\sigma,d}\int_{\BS^{d-1}}(v-v_{\alpha, \EE})^{2}\la E(\rho_t),v \ra d\rho_t\,.
	\end{align*}
	Notice that 
	\begin{align*}
	\la E(\rho_t), v\ra = \frac{1}{2} (E(\rho_t)^2+v^2-|E(\rho_t)-v|^2)=\frac{1}{2} (E(\rho_t)^2+1-(E(\rho_t)-v)^2)\,.
	\end{align*}
	Then one has
	\begin{align*}
	\frac{d}{dt}V(\rho_t)&=\lambda \left (\frac{E(\rho_t)^2+1}{2}-1 \right)\la E(\rho_t),v_{\alpha,\EE}(\rho_t)\ra-\frac{\lambda}{2}\int_{\BS^{d-1}}\la v_{\alpha, \EE}, v \ra (E(\rho_t)-v)^2 d\rho_t\\
	&\quad+C_{\sigma,d}\int_{\BS^{d-1}}(v-v_{\alpha, \EE})^{2}\la E(\rho_t),v \ra d\rho_t\\
	&=-\lambda V(\rho_t)\la E(\rho_t),v_{\alpha,\EE}(\rho_t)\ra -\frac{\lambda}{2}\int_{\BS^{d-1}}\la v_{\alpha, \EE}, v \ra (E(\rho_t)-v)^2 d\rho_t\\
	&\quad+C_{\sigma,d}\int_{\BS^{d-1}}(v-v_{\alpha, \EE})^{2}\la E(\rho_t),v \ra d\rho_t\,,
	\end{align*}
	where we have used the fact that $2V(\rho_t)=1-E(\rho_t)^2$. Moreover, since
	\begin{align*}
	\la v_{\alpha, \EE}, v\ra = \frac{1}{2} (v_{\alpha, \EE}^2+v^2-|v_{\alpha, \EE}-v|^2)=\frac{1}{2} (v_{\alpha, \EE}^2+1-(v_{\alpha,\EE}(\rho_t)-v)^2)
	\end{align*}
	and $\int_{\BS^{d-1}}(E(\rho_t)-v)^2d\rho_t=2V(\rho_t)$,
	we have
	\begin{align*}
	\frac{d}{dt}V(\rho_t)&=-\lambda V(\rho_t)\la E(\rho_t),v_{\alpha,\EE}(\rho_t)\ra -\frac{\lambda}{2}\frac{v_{\alpha, \EE}^2+1}{2}2V(\rho_t)\\
	&\quad+\frac{\lambda}{4}\int_{\BS^{d-1}} (E(\rho_t)-v)^2 (v-v_{\alpha, \EE})^2 d\rho_t+C_{\sigma,d}\int_{\BS^{d-1}}(v-v_{\alpha, \EE})^{2}\la E(\rho_t),v \ra d\rho_t\\
	&\leq -\lambda V(\rho_t)\left(\la E(\rho_t),v_{\alpha,\EE}(\rho_t)\ra +\frac{v_{\alpha, \EE}^2+1}{2}\right)\\ &\quad+\frac{\lambda}{4}\int_{\BS^{d-1}} (E(\rho_t)-v)^2 (v-v_{\alpha, \EE})^2 d\rho_t+4C_{\alpha,\EE}C_{\sigma,d}V(\rho_t)\,,
	\end{align*}
	where we have  used  estimate \eqref{511'} in the last inequality.
	
	Next we observe that
	\begin{align*}
	&\int_{\BS^{d-1}}(v-v_{\alpha, \EE})^{2}d\rho_t=\int_{\BS^{d-1}}(v-E(\rho_t)+E(\rho_t)-v_{\alpha, \EE})^{2}d\rho_t\\
	=&\int_{\BS^{d-1}}(v-E(\rho_t))^{2}d\rho_t+(E(\rho_t)-v_{\alpha, \EE})^2=2V(\rho_t)+E(\rho_t)^2+v_{\alpha, \EE}^2-2\la E(\rho_t), v_{\alpha, \EE}\ra \,.
	\end{align*}
	So it holds
	\begin{align}
	\la E(\rho_t), v_{\alpha, \EE}\ra&=V(\rho_t)+\frac{E(\rho_t)^2+v_{\alpha, \EE}^2}{2}-\frac{1}{2}\int_{\BS^{d-1}}(v-v_{\alpha, \EE})^{2}d\rho_t \notag\\
	&\geq V(\rho_t)+\frac{E(\rho_t)^2+v_{\alpha, \EE}^2}{2}-2C_{\alpha,\EE}V(\rho_t)\,,
	\end{align}
	where we have used \eqref{511'} again. Thus we obtain that
	\begin{align*}
	\frac{d}{dt}V(\rho_t)&\leq -\lambda V(\rho_t)\left(V(\rho_t)+\frac{2v_{\alpha, \EE}^2+1+E(\rho_t)^2}{2}-2C_{\alpha,\EE}V(\rho_t)\right)+4C_{\alpha,\EE}C_{\sigma,d}V(\rho_t)\notag\\
	&\quad +\frac{\lambda}{4}\int_{\BS^{d-1}} (E(\rho_t)-v)^2 (v-v_{\alpha, \EE})^2 d\rho_t\notag\\
	&= -\lambda V(\rho_t)\left(v_{\alpha, \EE}^2+1-2C_{\alpha,\EE}V(\rho_t)\right)\notag\\
	&\quad+\frac{\lambda}{4}\int_{\BS^{d-1}} (E(\rho_t)-v)^2 (v-v_{\alpha, \EE})^2 d\rho_t+4C_{\alpha,\EE}C_{\sigma,d}V(\rho_t) \notag\\
	&\leq -\lambda V(\rho_t)\left(2-2C_{\alpha,\EE}V(\rho_t)-4C_{\alpha,\EE}^2V(\rho_t)\right)\notag\\
	&\quad+\frac{\lambda}{4}\int_{\BS^{d-1}} (E(\rho_t)-v)^2 (v-v_{\alpha, \EE})^2 d\rho_t+4C_{\alpha,\EE}C_{\sigma,d}V(\rho_t) \,,
	\end{align*}
	where we have used $2V(\rho_t)=1-E(\rho_t)^2$ in the second equality and $2)$ from Lemma \ref{lemv} in the last inequality.
	
	Let $v^*$ be the minimizer used in Proposition \ref{thmE}, and one has
	\begin{align*}
	&\int_{\BS^{d-1}} (E(\rho_t)-v)^2 (v-v_{\alpha, \EE})^2 d\rho_t\\
	=&\int_{\BS^{d-1}} (E(\rho_t)-v)^2 (v-v^*)^2 d\rho_t+\int_{\BS^{d-1}} (E(\rho_t)-v)^2 (v_{\alpha, \EE}-v^*)^2 d\rho_t\\
	&+2\int_{\BS^{d-1}} (E(\rho_t)-v)^2 \la v-v^*,v^*-v_{\alpha, \EE}\ra d\rho_t \\
	\leq& \int_{\BS^{d-1}} (E(\rho_t)-v)^2 (v-v^*)^2 d\rho_t+2 \left (v_{\alpha, \EE}-\frac{E(\rho_t)}{|E(\rho_t)|} \right)^2 \int_{\BS^{d-1}} (E(\rho_t)-v)^2 d\rho_t\\
	&+2 \left (\frac{E(\rho_t)}{|E(\rho_t)|}-v^*\right )^2\int_{\BS^{d-1}} (E(\rho_t)-v)^2  d\rho_t  \\
	&+4 \left  |v_{\alpha, \EE}-\frac{E(\rho_t)}{|E(\rho_t)|} \right| \int_{\BS^{d-1}} (E(\rho_t)-v)^2 d\rho_t+4 \left |\frac{E(\rho_t)}{|E(\rho_t)|}-v^* \right |\int_{\BS^{d-1}} (E(\rho_t)-v)^2  d\rho_t\\
	\leq&\int_{\BS^{d-1}} (E(\rho_t)-v)^2 (v-v^*)^2 d\rho_t+2 (8C_{\alpha,\EE}^2-\frac{4}{3})V(\rho_t)^2\\
	&+2C(C_0,c_1,\beta)\left((C_{\alpha,\EE})^{2\beta}V(\rho_t)^{\beta}+\varepsilon^{2\beta}\right)V(\rho_t)\\
	&+4(8C_{\alpha,\EE}^2-\frac{4}{3})^{\frac{1}{2}}V(\rho_t)^{\frac{1}{2}}V(\rho_t)+4C(C_0,c_1,\beta)\left((C_{\alpha,\EE})^{\beta}V(\rho_t)^{\frac{\beta}{2}}+\varepsilon^\beta\right)V(\rho_t)\,,
	\end{align*}
	where we have used estimate \eqref{comva} and Proposition \ref{thmE} for $\alpha >\alpha_0$.
	This implies that
	\begin{align*}
	\frac{d}{dt}V(\rho_t)&\leq -\lambda V(\rho_t)\bigg(2-2C_{\alpha,\EE}V(\rho_t)-4C_{\alpha,\EE}^2V(\rho_t)-\frac{1}{2} (8C_{\alpha,\EE}^2-\frac{4}{3})V(\rho_t)\\
	&\quad-\frac{1}{2}C(C_0,c_1,\beta)\left((C_{\alpha,\EE})^{2\beta}V(\rho_t)^{\beta}+\varepsilon^{2\beta}\right)-(8C_{\alpha,\EE}^2-\frac{4}{3})^{\frac{1}{2}}V(\rho_t)^{\frac{1}{2}}\\
	&\quad-C(C_0,c_1,\beta)\left((C_{\alpha,\EE})^{\beta}V(\rho_t)^{\frac{\beta}{2}}+\varepsilon^\beta\right)
	\bigg)\\
	&\quad+\frac{\lambda}{4}\int_{\BS^{d-1}} (E(\rho_t)-v)^2 (v-v^*)^2 d\rho_t+4C_{\alpha,\EE}C_{\sigma,d}V(\rho_t) \\
	&\leq  -\lambda V(\rho_t)\left(2-C^*\left(C_{\alpha,\EE}^{2\max \{1, \beta\}}\overline{\mc{V}}_T^{\min1/2\{1,\beta\}}+\varepsilon^{\beta}\right)\right)\\
	&\quad+\frac{\lambda}{4}\int_{\BS^{d-1}} (E(\rho_t)-v)^2 (v-v^*)^2 d\rho_t+4C_{\alpha,\EE}C_{\sigma,d}V(\rho_t)\,,
	\end{align*}
	where $\overline{\mc{V}}_T:=\sup\limits_{0\leq t\leq T}V(\rho_t)\leq \frac{1}{2}$, and $C^*>0$ is a constant depending only on $c_1,\beta$ and $C_0$.

	Now we treat the term $\int_{\BS^{d-1}} (E(\rho_t)-v)^2 (v-v^*)^2 d\rho_t$, which can be split into two parts
	\begin{align*}
	&\int_{\BS^{d-1}} (E(\rho_t)-v)^2 (v-v^*)^2 d\rho_t\\
	=&\int_{\mc{D}_\delta} (E(\rho_t)-v)^2 (v-v^*)^2 d\rho_t+\int_{\BS^{d-1}/ \mc{D}_\delta} (E(\rho_t)-v)^2 (v-v^*)^2 d\rho_t\,,
	\end{align*}
	for some $\delta>0$, where 
	$$\mc{D}_\delta:=\left\{v\in\BS^{d-1}\big|\, -1\leq \la v,v^*\ra\leq -1 +\delta \right\}\,.$$
	This means that 
	\begin{equation}
	\frac{\lambda}{4}\int_{\BS^{d-1}/\mc{D}_\delta} (E(\rho_t)-v)^2 (v-v^*)^2 d\rho_t\leq \lambda (2-\delta) V(\rho_t)\,.
	\end{equation}
	Hence one can conclude
	\begin{align}
	\frac{d}{dt}V(\rho_t)&\leq  -\lambda V(\rho_t)\left(\delta-C^*\left(C_{\alpha,\EE}^{2\max \{1, \beta\}}\overline{\mc{V}}_T^{\min1/2\{1,\beta\}}+\varepsilon^{\beta}\right)\right)\notag\\
	&\quad +\frac{\lambda}{4}\int_{\mc{D}_\delta} (E(\rho_t)-v)^2 (v-v^*)^2 d\rho_t+4C_{\alpha,\EE}C_{\sigma,d}V(\rho_t)\,,
	\end{align}
	where we emphasize that $\delta >0$.
	
	Notice that $\mc{D}_\delta$ can be understood as a small cap on the sphere that is on the opposite side of the minimizer $v^*$. By the  assumption that $\rho_0\in L^2(\BS^{d-1})$ (see Definition \ref{def:wellprep}), we have the solution $\rho_t$ is not just a measure but it is a function, and for any given $T>0$ it satisfies $\rho\in L^\infty([0,T];L^2(\BS^{d-1}))$. This can be proved through a standard argument of PDE theory, which we provide in Theorem \ref{thmregularity}. Thus we have
	\begin{equation}\label{cap1}
	\int_{\mc{D}_\delta} d\rho_t=	\int_{\mc{D}_\delta} \rho_t(v)dv\leq\|\rho_t\|_2|\mc{D}_\delta|^{\frac{1}{2}}\leq C(T)(A_\delta)^{\frac{1}{2}}\,,
	\end{equation}
	where $A_\delta$ denotes the area of the hyperspherical cap $\mc{D}_\delta$, which satisfies the formula
	\begin{equation}\label{cap2}
	A_\delta =\frac{1}{2}a_{d}I_{2\delta-\delta^2}\left (\frac{d-1}{2},\frac{1}{2} \right)\leq C \frac{\pi^{\frac{d}{2}}}{\Gamma(\frac{d}{2})}\frac{(d-1)^{\frac{1}{2}}}{d-2}\delta^{\frac{d-2}{2}}\,,
	\end{equation}
	where $a_d$ represents the area of a unit ball and $I_x(a,b)$ is the regularized incomplete beta function. Note that
	\begin{equation}
	A_\delta \to 0 \mbox{ as }\delta \to 0\,.
	\end{equation}
	This means that for $d$ sufficiently large it holds
	\begin{equation*}
	\int_{\mc{D}_\delta} (E(\rho_t)-v)^2 (v-v^*)^2 d\rho_t\leq 16\int_{\mc{D}_\delta} d\rho_t(v) \leq C(\lambda, \sigma, T, \|\rho_0\|_2)(A_{\delta})^{\frac{1}{2}}\leq 4C_T\delta^{\frac{d-2}{4}} \,.
	\end{equation*}
	Therefore we have
	\begin{align}
	\frac{d}{dt}V(\rho_t)&\leq  -\lambda V(\rho_t)\left(\delta-C^*\left(C_{\alpha,\EE}^{2\max \{1, \beta\}}\overline{\mc{V}}_T^{\min1/2\{1,\beta\}}+\varepsilon^{\beta}\right)\right)+4C_{\alpha,\EE}C_{\sigma,d}V(\rho_t)+\lambda C_T\delta^{\frac{d-2}{4}}
	\end{align}
	for all $t\in[0,T]$.
	Let us assume that
	\begin{equation}\label{asV}
	\delta-C^*\left(C_{\alpha,\EE}^{2\max \{1, \beta\}}\overline{\mc{V}}_T^{\min1/2\{1,\beta\}}+\varepsilon^{\beta}\right)\geq\theta>0,\quad\mbox{i.e. }0\leq C_{\alpha,\EE}^{2\max \{1, \beta\}}\overline{\mc{V}}_T^{\min1/2\{1,\beta\}}+\varepsilon^{\beta} \leq \frac{\delta-\theta}{C^\ast}\,.
	\end{equation}
	Then we have
	\begin{equation*}
	\frac{d}{dt}V(\rho_t)\leq  -(\lambda\theta- 4C_{\alpha,\EE}C_{\sigma,d})V(\rho_t)+\lambda C_T\delta^{\frac{d-2}{4}}\,,
	\end{equation*}
	which leads to
	\begin{equation*}
	{V(\rho_t)\leq V(\rho_0)e^{-(\lambda\theta- 4C_{\alpha,\EE}C_{\sigma,d})t}+\frac{\lambda C_T}{\lambda\theta- 4C_{\alpha,\EE}C_{\sigma,d}}\delta^{\frac{d-2}{4}}\mbox{ for all }t\in[0,T]\,,}
	\end{equation*}
	which is contractive as soon as $\lambda\theta> 4C_{\alpha,\EE}C_{\sigma,d}$.
	We are left to verify the assumptions that $\overline{\mc{V}}_T\leq \min\left\{T^{-1}\| \omega_\EE^\alpha\|_{L^1(\rho_0)}^2,T^{-1}\lambda^{-2}\| \omega_\EE^\alpha\|_{L^1(\rho_0)}^4,\frac{3}{8} \right\}$ and \eqref{asV}, which hold if we assume that
	\begin{eqnarray*}
		&&C_{\alpha,\EE}^{2\max \{1, \beta\}}\left(V(\rho_0)+\frac{\lambda C_T}{\lambda\theta- 4C_{\alpha,\EE}C_{\sigma,d}}\delta^{\frac{d-2}{4}} \right)^{\frac{1}{2}\min\{1,\beta\}} +\varepsilon^\beta<\frac{\delta-\theta}{C^\ast}, \\
		&&V(\rho_0)+\frac{\lambda C_T}{\lambda\theta- 4C_{\alpha,\EE}C_{\sigma,d}}\delta^{\frac{d-2}{4}}\leq \min\left\{T^{-1}\| \omega_\EE^\alpha\|_{L^1(\rho_0)}^2,T^{-1}\lambda^{-2}\| \omega_\EE^\alpha\|_{L^1(\rho_0)}^4,\frac{3}{8} \right\}\,.
	\end{eqnarray*}
	Hence we complete the proof.
\end{proof}

\begin{proof} ({\bf Theorem \ref{thm:mainresult}})	
	Proposition \ref{mainp} implies  that for any $\varepsilon_1>0$, there exists some $T^*$ large enough such that $$V(\rho_{T^*})\leq \varepsilon_0:=\frac{\lambda C_{T^*}}{\lambda\theta- 4C_{\alpha,\EE}C_{\sigma,d}}\delta^{\frac{d-2}{4}}+\varepsilon_1.$$ 
	Moreover $1 \geq |E(\rho_{T^*})| = \sqrt{1- 2 V(\rho_{T^*})} \geq \sqrt{1- 2 \varepsilon_0} $ and 
	$$
	\left |E(\rho_{T^*}) - \frac{E(\rho_{T^*})}{|E(\rho_{T^*})|} \right | \leq \frac{1- |E(\rho_{T^*})|}{|E(\rho_{T^*})|} \leq \frac{1-\sqrt{1- 2 \varepsilon_0}}{\sqrt{1- 2 \varepsilon_0} } \leq 2 \varepsilon_0, 
	$$
	as soon as $0 \leq \varepsilon_0 \leq \frac{1}{4}(\sqrt 5 -1)$, which is fulfilled as soon as $\delta, \varepsilon_1$ are chosen small enough.
	These estimates, triangle inequality and Proposition \ref{thmE} lead to the quantitative estimate
	\begin{align*}
	|E(\rho_{T^*})-v^*|&\leq  C(C_0,c_1,\beta)\left((1+C_{\alpha^*,\EE}^{\beta})\left(\frac{\lambda C_{T^*}}{\lambda\theta- 4C_{\alpha^*,\EE}C_{\sigma,d}}\delta^{\frac{d-2}{4}}+\varepsilon_1\right )^{\min\left \{1,\frac{\beta}{2}\right \}}+\varepsilon^\beta\right).
	\end{align*}
	Note once again here that $\varepsilon$, $\delta$, and $\varepsilon_1$ can be all chosen to be sufficiently small. 
\end{proof}
{
\subsection{Proof of the main result}\label{sec:mainproof}

Let us  finally address the proof of the main theorem of this paper.\\

\begin{proof} ({\bf Theorem \ref{mainresult00}})	
In order to show a concrete instance of the result, we develop the proof for the case where $\{V_{n}^i:=V_{\Delta t, n}^i: n=0,\dots,n_{T^*}; i=1\dots N\}$ are generated by the  iterative algorithm \eqref{Intro KViso num}. However, any other numerical scheme of order $m$ can be considered \cite{Platen}.
The SDE system \eqref{stochastic Kuramoto-Vicsek} is well-posed by \cite[Theorem 2.1]{fhps20-1} and it admits a pathwise strong solution $V_{t}^i$, $i=1,\dots,N$.
The iterative algorithm \eqref{Intro KViso num} is the discrete-time (projected Euler-Maruyama) approximation of the SDE system \eqref{stochastic Kuramoto-Vicsek} with order of approximation $m=1/2$
by classical results, e.g., see \cite[Theorem 2.2]{doi:10.1137/S0036142901389530}
\begin{equation}\label{strongconv}
\mathbb E \left [\sup_{n=0,\dots,n _{T^*}} |V_{\Delta t,n}- V_{t_n}|^2 \right ] \leq \bar C_1 (\Delta t)^{2 m},
\end{equation}
for $\bar C_1$ which depends linearly on $d$ and $N$, and possibly exponentially on $T^*$, $\lambda$, and $\sigma$ (see in particular the estimates before (2.11) in the proof of \cite[Theorem 2.2]{doi:10.1137/S0036142901389530}). 
Let us stress that the introduction of the post-projection $V^i_{n+1} \gets \tilde V^i_{n+1}/|\tilde V^i_{n+1}|$ to enforce the dynamics on the sphere may produce an additional error of at most order $\Delta t$ because
$$
\left |\tilde V^i_{n+1} - \tilde V^i_{n+1}/|\tilde V^i_{n+1}|\right |^2 =|\tilde V^i_{n+1}|^2+1-2|\tilde V^i_{n+1}|= (|\tilde V^i_{n+1}|-1)^2
$$
and, in view of
$$
\tilde V^i_{n+1} =V^i_{n} +\Delta t P(V_{n}^i)V_{n}^{\alpha, \EE} + \sigma |V_n^i - V_n^{\alpha, \EE}| P(V_n^i)\Delta B_n^i-\displaystyle\Delta t\frac{\sigma^2}{2}(V_n^i-V_n^{\alpha, \EE})^2(d-1)V_n^i,
$$
we obtain
\begin{eqnarray}
\mathbb E \left[(|\tilde V^i_{n+1}|-1)^2\right] &=&\mathbb E\left[ (|\tilde V^i_{n+1}|-| V^i_{n}|)^2\right] \nonumber \\
&\leq& \mathbb E \left[((1  + 2 (d-1)\sigma^2)\Delta t  + 2 \sigma |\Delta B_n^i |)^2\right]\nonumber \\
&\leq & (1  + 2 (d-1)\sigma^2)^2\Delta t^2 + 4 \sigma^2 d \Delta t + 2\sigma (1 + 2 (d-1)\sigma^2)\Delta t \sqrt{d \Delta t}\nonumber \\
&\leq & \bar C_1' \Delta t.
\end{eqnarray}
By \cite[Theorem2.2]{fhps20-1} we have also well-posedness of \eqref{monoparticle} with pathwise strong solution $\OV_{t}$. For $\OV_{0}^i$ drawn i.i.d. according to $\rho_0$, $i=1\dots, N$, an application of \cite[Theorem 3.1]{fhps20-1} yields
\begin{equation}\label{rateNagain}
\sup_{t \in [0, T]} \sup_{i=1,\dots,N}\mathbb E \left[ |V_t^i-\OV_t^i|^2\right] \leq   \bar C_{2} N^{-1},
\end{equation}
for any $T>0$ time horizon.  As clarified in \cite[Remark 3.2 and Lemma 3.1]{fhps20-1}, the constant $ \bar C_{2}$ depends at most linearly on $d$, and, as a worst case analysis, polynomially on $C_{\alpha^*,\EE}$, and exponentially on $T$. By law of large numbers, for $\rho_t= \operatorname{law}(\bar V_t)$ it holds
\begin{equation}\label{LLN}
\mathbb E  \left |\frac{1}{N} \sum_{i=1}^N \OV_{T^*}^i - E(\rho_{T^*}) \right|^2 \leq \bar C_2' N^{-1}.
\end{equation}
Under the assumptions of well-preparation,  Theorem \ref{thm:mainresult} yields
	\begin{equation}\label{locest2}
	\left |E(\rho_{T^*})-v^* \right |^2\leq \bar  C_3 \epsilon,
	\end{equation}
	for $\bar C_3$ that depends polynomially on $C_{\alpha^*,\EE}$. By combining the strong convergence \eqref{strongconv}, the mean-field limit \eqref{rateNagain}, the law of large numbers \eqref{LLN}, and the large time aymptotics \eqref{locest2} we conclude by multiple applications of Jensen inequality the final error estimate
	\begin{align}\label{mainresult}
&\mathbb E \left   [\left|\frac{1}{N} \sum_{i=1}^N V_{\Delta t, n_{T^*}}^i - v^* \right|^2 \right ] \nonumber \\
\leq& 8 \left  ( \mathbb E \left   [\left|\frac{1}{N} \sum_{i=1}^N( V_{\Delta t,n_{ T^*}}^i -  V_{T^*}^i) \right|^2 \right ]  +  \mathbb E \left   [\left|\frac{1}{N} \sum_{i=1}^N (V_{T^*}^i -  \OV_{T^*}^i) \right|^2 \right ]  \notag \right .\\
&+  \left . \mathbb E \left   [\left|\frac{1}{N} \sum_{i=1}^N \OV_{T^*}^i - E(\rho_{T^*})  \right|^2 \right ] + |E(\rho_{T^*}) - v^*|^2 \right ) \nonumber  \\
\leq& 8 \bar C_1(\Delta t)^{2m} + 8(\bar C_2+ \bar C_2') N^{-1} + 8 \bar C_3 \epsilon^2.
\end{align}
\end{proof}
}

\section{Auxiliary Results and Proofs}\label{sec:aux}

\subsection{Proofs of auxiliary lemmas}

\begin{proof}({\bf Lemma \ref{lemnorm}})
	
	From \eqref{eq:gen}  we get
	\begin{eqnarray*}
		\langle V^i_{n+1}, V^i_{n+1} \rangle &=& \langle V^i_{n}, V^i_{n} \rangle + \langle \Phi(\Delta t, V^i_n,V^i_{n+1},\xi^i_n), \Phi(\Delta t, V^i_n,V^i_{n+1},\xi^i_n) \rangle\\
		&& + 2\langle \Phi(\Delta t, V^i_n,V^i_{n+1},\xi^i_n), V^i_{n} \rangle. 
	\end{eqnarray*}
	Assuming $\langle V^i_{n+1}, V^i_{n+1}\rangle = \langle V^i_{n}, V^i_{n} \rangle$ implies
	\begin{eqnarray*}
		0&=&\langle \Phi(\Delta t, V^i_n,V^i_{n+1},\xi^i_n), \Phi(\Delta t, V^i_n,V^i_{n+1},\xi^i_n) \rangle + 2\langle \Phi(\Delta t, V^i_n,V^i_{n+1},\xi^i_n), V^i_{n} \rangle\\
		& = &
		\langle \Phi(\Delta t, V^i_n,V^i_{n+1},\xi^i_n), \Phi(\Delta t, V^i_n,V^i_{n+1},\xi^i_n) + 2 V^i_{n} \rangle\\
		& = & \langle \Phi(\Delta t, V^i_n,V^i_{n+1},\xi^i_n), V^i_{n+1}+V^i_n \rangle
	\end{eqnarray*}
	where we used the fact that $\Phi(\Delta t, V^i_n,V^i_{n+1},\xi^i_n)=V^i_{n+1}-V^i_n$. 
	\label{le:v1}
\end{proof}

\begin{proof}({\bf Lemma \ref{lemv}})
	Using  Jensen's inequality, one concludes that
	\begin{align}
	\int_{\BS^{d-1}} |v-v_{\alpha,\EE}(\rho_t)|^2d\rho_t &\leq \frac{1}{\|\omega_\alpha^{\EE}\|_{L^{1}(\rho_t)}}\int_{\BS^{d-1}} \int_{\BS^{d-1}} |v-u|^2e^{-\alpha \EE( u )} d\rho_t(v) d\rho_t(u)\,.
	\end{align}
	The expression on the right can be further estimated as follows
	\begin{align}
	\int_{\BS^{d-1}} |v-v_{\alpha,\EE}(\rho_t)|^2d\rho_t& \leq 4\frac{e^{-\alpha \underline{\EE}}}{\|\omega_\alpha^{\EE}\|_{L^{1}(\rho_t)}}V(\rho_t) \label{511}\\
	&\leq 4C_{\alpha,\EE}V(\rho_t)\label{511'}\,,
	\end{align}
	where$C_{\alpha,\EE}=e^{\alpha(\overline{\EE}-\underline{\EE})}$.
	Similarly one has
	\begin{align}
	\int_{\BS^{d-1}} |v-v_{\alpha,\EE}(\rho_t)|d\rho_t &\leq  \frac{1}{\|\omega_\alpha^{\EE}\|_{L^{1}(\rho_t)}}\int \int |v-u|e^{-\alpha \EE( u )} d\rho_t(v) d\rho_t(u) \leq2\frac{e^{-\alpha \underline{\EE}}}{\|\omega_\EE^\alpha\|_{L^{1}(\rho_t)}}V(\rho_t)^{\frac{1}{2}}\label{eqsi}\\
	&\leq  2C_{\alpha,\EE}V(\rho_t)^{\frac{1}{2}}\,.\notag
	\end{align}
	
	Next we notice that
	\begin{align}
	1-v_{\alpha,\EE}(\rho_t)^2=\frac{\int_{\BS^{d-1}} (v-v_{\alpha,\EE}(\rho_t)^2) \omega_\alpha^{\EE}(u) d \rho_t(u)}{\|\omega_\alpha^{\EE}\|_{L^{1}(\rho_t)}}\leq 4C_{\alpha,\EE}^2V(\rho_t)\,,
	\end{align}
	where we have used \eqref{511'} in the last inequality. This implies estimate $2)$.
	
	To obtain $3)$, we compute 
	\begin{align*}
	&|v_{\alpha,\EE}(\rho_t)-E(\rho_t)|^2=\int_{\BS^{d-1}}|v_{\alpha,\EE}(\rho_t)-v+v-E|^2d\rho_t(v)\\
	=&\int_{\BS^{d-1}}|v_{\alpha,\EE}(\rho_t)-v|^2d\rho_t+\int_{\BS^{d-1}}|v-E|^2d\rho_t+2\int_{\BS^{d-1}}\la v_{\alpha,\EE}(\rho_t)-v,v-E\ra d\rho_t\\
	\leq&4C_{\alpha,\EE}^2V(\rho_t)+2V(\rho_t)+2E^2-2=(4C_{\alpha,\EE}^2-2)V(\rho_t)\,,
	\end{align*}
	which completes the proof.
\end{proof}

\begin{proof}({\bf Lemma \ref{lemome}})
	The derivative of $\|\omega_\EE^\alpha \|_{L^1(\rho_t)}$ is given by
	\begin{align}
	\frac{d}{dt} \int_{\BS^{d-1}} \omega_\EE^\alpha(v) d\rho_t & = \int_{\BS^{d-1}} \frac{\sigma^2}{2}|v-v_{\alpha,\EE}(\rho_t)|^2 \Delta_{\BS^{d-1}} \omega_\EE^\alpha\notag\\
	&\quad - \lambda (\langle v_{\alpha,\EE}(\rho_t),v\rangle v-v_{\alpha,\EE}(\rho_t)) \cdot \nabla_{\BS^{d-1}} \omega_\EE^\alpha d\rho_t\notag\\
	&=\int_{\BS^{d-1}} \frac{\sigma^2}{2}|v-v_{\alpha,\EE}(\rho_t)|^2 \Delta_{\BS^{d-1}} \omega_\EE^\alpha+ \lambda P(v)v_{\alpha,\EE}(\rho_t)\cdot \nabla_{\BS^{d-1}}  \omega_\EE^\alpha d\rho_t\notag\\
	&=: \textbf{I} + \textbf{II}\,.
	\end{align}
	The gradient and the Laplacian of the weight function can be computed as
	\begin{align}
	\nabla_{\BS^{d-1}} \omega_\EE^\alpha(v) =\nabla \omega_\EE^\alpha\left (\frac{v}{|v|} \right )\bigg|_{|v|=1} =\frac{1}{|v|}\left(I-\frac{vv^T}{|v|^2}\right)\nabla \omega_\EE^\alpha \bigg|_{|v|=1} =-\alpha e^{-\alpha \EE }(I-vv^T) \nabla\EE\bigg|_{|v|=1}
	\end{align}
	and 
	\begin{align}
	\Delta_{\BS^{d-1}} \omega_\EE^\alpha(v)= \Delta \omega_\EE^\alpha\left (\frac{v}{|v|} \right)\bigg|_{|v|=1}=\frac{\Delta \omega_\EE^\alpha}{|v|}-(d-1)\frac{v}{|v|^3}\cdot \nabla \omega_\EE^\alpha-\frac{vv^T}{|v|^3}:\nabla^2\omega_\EE^\alpha\bigg|_{|v|=1}\,.
	\end{align}
	We further have
	\begin{align}
	\nabla \omega_\EE^\alpha &= -\alpha e^{-\alpha \EE} \nabla \EE \in \mathbb{R}^d\,; \\
	\nabla^2 \omega_\EE^\alpha &= -\alpha e^{-\alpha \EE}(-\alpha \nabla \EE \otimes \nabla \EE + \nabla^2 \EE) \in \mathbb{R}^{d\times d}\,;\\
	\Delta \omega_\EE^\alpha &= \alpha^2 e^{-\alpha \EE}|\nabla \EE|^2 - \alpha e^{-\alpha \EE} \Delta \EE \in \mathbb{R}\,.
	\end{align}	
	We estimate the term $\textbf{I}$ as follows
	\begin{align}\label{esI}
	\textbf{I} &= \frac{\sigma^2}{2}\int |v-v_{\alpha, \EE}|^2 \bigg(\Delta \omega_\EE^\alpha-(d-1) v \cdot \nabla \omega_\EE^\alpha  -v \otimes v:\nabla^2\omega_\EE^\alpha\bigg)d\rho_t(v) \notag\\
	&= \frac{\sigma^2}{2}\int |v-v_{\alpha, \EE}|^2 \bigg[ \alpha^2 |\nabla \EE|^2 - \alpha \Delta \EE + \alpha (d-1)v \cdot \nabla \EE\notag\\
	&\qquad \qquad + \alpha \bigg( v\otimes v : (-\alpha \nabla \EE \otimes \nabla \EE) + v\otimes v : \nabla^2 \EE\bigg)\bigg]e^{-\alpha \EE} d \rho_t(v) \notag\\
	&\geq \frac{\sigma^2}{2}\int |v-v_{\alpha, \EE}|^2 \bigg[ - \alpha \Delta \EE + \alpha (d-1)\nabla \EE \cdot v - \alpha^2 |\nabla \EE|^2 - \alpha|\nabla^2 \EE|\bigg] e^{-\alpha \EE} d\rho_t(v) \notag \\
	& \geq \frac{\sigma^2}{2}\int |v-v_{\alpha, \EE}|^2 e^{-\alpha \EE}\bigg[- \alpha c_2 - \alpha (d-1)c_1 - \alpha^2 c_1^2-\alpha c_2\bigg]d\rho_t(v)  \notag\\
	& \geq -2\sigma^2\alpha e^{-2\alpha \underline{\EE}}(2c_2+(d-1)c_1+\alpha c_1^2)\frac{V(\rho_t)}{\|\omega_{\alpha}^\EE\|_{L^1(\rho_t)}}\,,
	\end{align}
	where we have used that $|\nabla\EE|\leq c_1$;  $|\Delta \EE|, |\nabla^2\EE| \leq c_2$, estimate \eqref{511}  and the property
	\begin{equation}
	v\otimes v : \nabla \EE \otimes \nabla \EE = \sum_{i,j} v_i v_j \partial_i \EE \partial_j \EE  \leq (\sum_i \partial_i \EE)^2 \leq |\nabla \EE|^2\,.
	\end{equation}

	For the term $\textbf{II}$ we get
	\begin{align}
	\textbf{II} &=
	-\alpha \lambda \int_{\BS^{d-1}} e^{-\alpha\EE}P(v)v_{\alpha,\EE}(\rho_t)\cdot(\nabla\EE-vv^T\nabla\EE)d \rho_t=	-\alpha \lambda \int_{\BS^{d-1}} e^{-\alpha\EE}P(v)v_{\alpha,\EE}(\rho_t)\cdot \nabla\EE d \rho_t\notag\\
	&=\alpha \lambda \int_{\BS^{d-1}} e^{-\alpha\EE}(\langle v_{\alpha,\EE}(\rho_t),v\rangle v-v_{\alpha,\EE}(\rho_t)) \cdot  \nabla\EE  d\rho_t\notag\\
	&\geq -\alpha \lambda c_1e^{-\alpha \underline \EE}\int_{\BS^{d-1}} |\langle v_{\alpha,\EE}(\rho_t),v\rangle v-v_{\alpha,\EE}(\rho_t)| d\rho_t\,,
	\end{align}
	where in the second equality we have used the fact that $v\cdot P(v)v_{\alpha,\EE}(\rho_t)=0$. We observe that
	\begin{align}\label{esmu}
	\int_{\BS^{d-1}} |\langle v_{\alpha,\EE}(\rho_t),v\rangle v-v_{\alpha,\EE}(\rho_t)| d\rho_t&=\int_{\BS^{d-1}} |\langle v_{\alpha,\EE}(\rho_t)-v,v\rangle v +v-v_{\alpha,\EE}(\rho_t)| d\rho_t\notag\\
	&\leq \int_{\BS^{d-1}} |\langle v_{\alpha,\EE}(\rho_t)-v,v\rangle v | d\rho_t+\int_{\BS^{d-1}} |v- v_{\alpha,\EE}(\rho_t)| d\rho_t\notag\\
	&\leq 2\int_{\BS^{d-1}}|v-v_{\alpha,\EE}(\rho_t)|d\rho_t \leq  4 \frac{e^{-\alpha \underline{\EE}}}{\|\omega_\EE^\alpha\|_{L^{1}(\rho_t)}}V(\rho_t)^{\frac{1}{2}}\,,
	\end{align}
	where we have used \eqref{eqsi} in the last inequality. Thus we have
	\begin{equation}\label{esII}
	\textbf{II}\geq -2\alpha \lambda c_1e^{-\alpha \underline \EE}\int_{\BS^{d-1}}|v-v_{\alpha,\EE}(\rho_t)|d\rho_t(v)\geq-4\alpha \lambda c_1e^{-2\alpha \underline \EE}\frac{V(\rho_t)^{\frac{1}{2}}}{\|\omega_{\alpha}^\EE\|_{L^1(\rho_t)}}\,.
	\end{equation}
	
	Combining the inequalities \eqref{esI} and \eqref{esII} yields
	\begin{align}
	\frac{1}{2} \frac{d}{dt}\|\omega_\alpha^\EE \|_{L^1(\rho_t)}^2 & = \|\omega_\alpha^\EE  \|_{L^1(\rho_t)} \frac{d}{dt} \|\omega_\alpha^\EE \|_{L^1(\rho_t)} \notag \\ 
	& \geq -2\sigma^2\alpha e^{-2\alpha \underline{\EE}}(2c_2+(d-1)c_1+\alpha c_1^2)V(\rho_t)-4\alpha \lambda c_1e^{-2\alpha \underline \EE}V(\rho_t)^{\frac{1}{2}}\notag \\
	&=:-b_1(d,\sigma,\alpha,c_1,c_2,\underline{\EE})V(\rho_t)-b_2(\alpha,c_1,\underline{\EE})\lambda V(\rho_t)^{\frac{1}{2}}\,,
	\end{align}
	where $b_1,b_2\to 0$ as $\alpha\to\infty$.
\end{proof}
\subsection{Well-posedness and regularity result}

\begin{theorem}\label{thmregularity}
	For any given $T>0$, let $\rho_0\in L^2(\BS^{d-1})$. Then there exists a unique weak solution $\rho$ to equation \eqref{PDE}. Moreover it has the following regularity
	\begin{align}
	\rho\in L^\infty([0,T];L^2(\BS^{d-1}))\cap L^2([0,T];H^1(\BS^{d-1}))\mbox{ and } \partial_t\rho\in L^2([0,T];H(\BS^{d-1})')\,.
	\end{align}
\end{theorem}
\begin{proof}
	The proof is standard and based on Picard's iteration.  We sketch below the details. Let $\rho^0(x,t)\equiv \rho_0(x)$. For $n\geq 0$, let $\rho^{n+1}$ be the unique weak solution to following linear equation
	\begin{equation}
	\partial_t \rho_t^{n+1}= \lambda \nabla_{\BS^{d-1}} \cdot ((\langle v_{\alpha, \EE }(\rho_t^n), v \rangle v - v_{\alpha,\EE}(\rho_t^n) )\rho_t^{n+1})+\frac{\sigma^2}{2}\Delta_{\BS^{d-1}} (|v-v_{\alpha,\EE}(\rho_t^n) |^2\rho_t^{n+1}),\quad t>0\,,
	\end{equation}
	with the initial data $\rho^{n+1}(x,0)=\rho_0(x)$ for any given $\rho^n\in L^\infty([0,T];L^2(\BS^{d-1}))\cap L^2([0,T];H^1(\BS^{d-1}))$. For any given $T>0$ and $t\in[0,T]$, it is easy to compute that 
	\begin{align*}
	&\frac{1}{2}\frac{d}{dt}\|\rho_t^{n+1}\|_2^2+\frac{\sigma^2}{2}\int_{\BS^{d-1}}\nabla_{\BS^{d-1}}\rho_t^{n+1} \cdot \nabla_{\BS^{d-1}}  (|v-v_{\alpha,\EE}(\rho_t^n) |^2\rho_t^{n+1}) dv\\
	=&-\lambda  \int_{\BS^{d-1}}\nabla_{\BS^{d-1}}\rho_t^{n+1}\cdot (\langle v_{\alpha, \EE }(\rho_t^n), v \rangle v - v_{\alpha,\EE}(\rho_t^n) )\rho_t^{n+1}dv\leq \lambda  \int_{\BS^{d-1}}|\nabla_{\BS^{d-1}}\rho_t^{n+1}|\rho_t^{n+1}dv \,.
	\end{align*}
	This lead to
	\begin{align*}
	\frac{1}{2}\frac{d}{dt}\|\rho_t^{n+1}\|_2^2&\leq\lambda  \int_{\BS^{d-1}}|\nabla_{\BS^{d-1}}\rho_t^{n+1}|\rho_t^{n+1}dv-\frac{\sigma^2}{2}\int_{\BS^{d-1}}|\nabla_{\BS^{d-1}}\rho_t^{n+1}|^2 |v-v_{\alpha,\EE}(\rho_t^n) |^2dv\\
	&\quad -\sigma^2\int_{\BS^{d-1}} \nabla_{\BS^{d-1}}\rho_t^{n+1} \cdot (v-v_{\alpha,\EE}(\rho_t^n))\rho_t^{n+1}dv\\
	&\leq -\frac{\sigma^2}{2}\min\limits_{t\in[0,T]}\operatorname{ess}\inf\limits_{v\in \BS^{d-1}}|v-v_{\alpha,\EE}(\rho_t^n)|^2\| \nabla_{\BS^{d-1}}\rho_t^{n+1}\|_2^2\\
	&\quad+\varepsilon \| \nabla_{\BS^{d-1}}\rho_t^{n+1}\|_2^2+C(\varepsilon,\sigma,\lambda)  \|\rho_t^{n+1}\|_2^2 \\
	&\leq C(\varepsilon,\sigma,\lambda)  \|\rho_t^{n+1}\|_2^2\,,
	\end{align*}
	where we have used  H\"{o}lder's inequality in the second inequality. Applying Gronwall's inequality it yields that 
	\begin{equation}
	\|\rho_t^{n+1}\|_2^2+\int_0^T \| \nabla_{\BS^{d-1}}\rho_t^{n+1}\|_2^2dt\leq C(T,\sigma,\lambda,\|\rho_0\|_2)\,.
	\end{equation}
	We also get that for all $\psi\in H^1(\BS^{d-1})$
	\begin{align*}
	&\|\partial_t\rho_t^{n+1}\|_{H^1(\BS^{d-1})'}=\sup_{\|\psi\|_{H^1}\leq 1}|\la\partial_t\rho_t^{n+1},\psi\ra|\\
	\leq &\sup_{\|\psi\|_{H^1}\leq 1}\left|\la\nabla_{\BS^{d-1}}\psi,\lambda(\langle v_{\alpha, \EE }(\rho_t^n), v \rangle v - v_{\alpha,\EE}(\rho_t^n) )\rho_t^{n+1}+\frac{\sigma^2}{2}\nabla_{\BS^{d-1}}  (|v-v_{\alpha,\EE}(\rho_t^n) |^2\rho_t^{n+1})\ra\right|\\
	\leq &C(\lambda,\sigma)\|\rho_t^{n+1}\|_{H^1}\,.
	\end{align*}
	Thus we obtain $\partial_t\rho^{n+1}\in  L^2([0,T];H(\BS^{d-1})')$. Note that this also implies that $\rho^{n+1}\in \mc{C}([0,T];L^2(\BS^{d-1}))$ due to the fact that
	\begin{equation*}
	\max_{0\leq t\leq T}\| \rho^{n+1} \|_2\leq C(\|\rho^{n+1} \|_{L^2([0,T],H^1)} +\|\partial_t\rho^{n+1}\|_{ L^2([0,T];H(\BS^{d-1})')})\,,
	\end{equation*}
	where $C$ depends only $T$. Then by Aubin-Lions lemma, there exists a subsequence $\rho^{n_k}$ and a function $\rho\in L^2([0,T]\times \BS^{d-1})$ such that
	\begin{equation}
	\rho^{n_k}\to \rho \mbox{  in } L^2([0,T]\times \BS^{d-1}) \mbox{  as } k\to \infty\,.
	\end{equation}
	To finish the proof of existence we are left to pass the limit and verify $\rho$ is the solution, we omit the details here of this very standard concluding step (see, e.g., \cite[Theorem 2.4]{albi2017mean} for similar arguments).
	
	As for the uniqueness, it has been obtained in \cite[Section 2.2 and Section 2.3]{fhps20-1}  by using the uniqueness of the corresponding nonlinear SDE \eqref{selfprocess}.
\end{proof}

\section{Conclusions}

We presented the numerical implementation of a new consensus-based model for global optimization on the sphere, which is inspired by the kinetic Kolmogorov-Kuramoto-Vicsek equation.
The main result of this paper is about {the first and currently unique proof of the convergence of consensus-based optimization to global minimizers} provided conditions of well-preparation of the initial datum. We present several numerical experiments in low dimension and synthetic examples in order to illustrate the behavior of the method and we tested the algorithms in high dimension against state of the art methods in a couple of challenging problems in signal processing and machine learning, namely the phase retrieval problem and the robust subspace detection.
These experiments show that the algorithm proposed in the present paper scales well with the dimension and is very versatile (one just needs to modify the definition of the function $\EE$ and the rest  goes with the same code\footnote{{\it 
		https://github.com/PhilippeSu/KV-CBO}}!). The algorithm is able to perform essentially as good as {\it ad hoc} state of the art methods and in some instances it obtains quantifiably better results. 
The theoretical rate of convergence  is of order $N^{-1}$ in the particle number $N$ and it does not depend on the dimension. Multiplicative constants may depend at most linearly on the dimension $d$ and, as worst case scenario, exponentially in the parameter $\alpha$. The rate of convergence is exponential and explicitly computable from the parameters of the method, i.e., $\lambda \theta - 2 (d-1) e^{\alpha(\overline{\EE}-\underline{\EE})} \sigma^2$. The numerical experiments in high dimension ($d \approx 3000$) confirm that the method is in general not affected by curse of dimensionality. Moreover, the requirement of well-preparation of the initial datum (Definition \ref{def:wellprep}) is due to the proving technique we are using based on the monotone decay of the variance. In the case of symmetric cost functions $\EE(v)=\EE(-v)$, the well-preparation is by no means a severe restriction. We conjecture that with other proving  techniques the conditions of well-preparation can be removed, since in the numerical experiments the initialization by uniform distribution yields to global convergence consistently. In our view, this work represents a fundamental theoretical contribution to CBO methods on the sphere, on which to build variations of the algorithm with the aim of further improving its complexity and convergence towards the global minimum. A promising perspective in this direction is to consider the introduction of anisotropic noise in order to reduce dependence of the parameters from the dimension and to better explore the search space in case of very high dimensional problems \cite{carrillo2019consensus}. This and other algorithmic improvements are left to future research.\\

\textbf{Acknowledgment} Fornasier and Hui Huang acknowledge the support of the DFG Project "Identification of Energies from Observation of Evolutions" and the DFG SPP 1962 "Non-smooth and Complementarity-based Distributed Parameter Systems: Simulation and Hierarchical Optimization".     The present project and Philippe  S\"{u}nnen  are supported by the National Research Fund, Luxembourg (AFR PhD Project Idea ``Mathematical Analysis of Training Neural Networks'' 12434809). 
	Lorenzo Pareschi acknowledges the support of the John Von Neumann guest Professorship program of the Technical University of Munich during the preparation of this work. The authors acknowledge the support and the facilities of the LRZ Compute Cloud of the Leibniz Supercomputing Center of the Bavarian Academy of Sciences, on which the numerical experiments of this paper have been tested.

\bibliographystyle{plain}
\bibliography{bibfile}

\end{document}